%% file: main.tex
\documentclass[letterpaper, 10 pt, conference]{IEEEtran} 
\usepackage{etex}
\IEEEoverridecommandlockouts  

\usepackage{multicol}
\usepackage{tikz}

\input{preamble_packages.tex}

\usepackage{amsthm}
\newtheoremstyle{mystyle}%                % Name
  {}%                                     % Space above
  {}%                                     % Space below
  {\itshape}%                                     % Body font
  {}%                                     % Indent amount
  {\bfseries}%                            % Theorem head font
  {.}%                                    % Punctuation after theorem head
  { }%                                    % Space after theorem head, ' ', or \newline
  {\thmname{#1}\thmnumber{ #2}\thmnote{ (#3)}}% % Theorem head spec (can be left empty, meaning `normal')
\theoremstyle{mystyle}
\input{preamble_symbols.tex}

\input{shortcuts}

\usepackage{float}

\usepackage{hyperref}
\usepackage{graphicx}
\usepackage{epstopdf}
\usepackage{epsfig}

\newcommand{\myparagraph}[1]{{\bf #1.}}
\newcommand{\mysubparagraph}[1]{\emph{#1.}}
\usepackage{xr}

\newcommand{\journalVersion}[1]{}

\newcommand{\hor}{H}
\newcommand{\feigMin}{f_{\lambda}}
\newcommand{\flogDet}{f_{\det}}
\newcommand{\lambdaMax}{\lambda_{\max}}
\newcommand{\lambdaMin}{\lambda_{\min}}
\newcommand{\logdet}{\log\det}
\renewcommand{\prob}[1]{\text{Prob}(#1)}
\renewcommand{\E}[1]{\mathbb{E}\left[#1\right]}
\newcommand{\barMOmega}{\bar{\MOmega}}
\renewcommand{\P}{\vt}
\renewcommand{\V}{\vv}
\newcommand{\B}{\bias}
\renewcommand{\pos}[1]{\P_{#1}}
\newcommand{\vel}[1]{\V_{#1}}
\newcommand{\biasAcc}[1]{\vb_{#1}}
\newcommand{\Rot}[1]{\MR_{#1}}
\renewcommand{\acc}[1]{\va_{#1}}
\newcommand{\measAcc}[1]{\tilde\va_{#1}}

\newcommand{\noiseAcc}[1]{\noise_{#1}}
\renewcommand{\k}{k}
\renewcommand{\j}{j}
\newcommand{\ind}{i}
\newcommand{\indd}{h}
\newcommand{\Deltak}{\delta}
\newcommand{\Deltaij}{\delta_{\k\j}}
\newcommand{\DeltaijSq}{\hat\delta^2_{\k\j}}
\renewcommand{\c}{\k} % camera time
\renewcommand{\l}{l} % point id
\newcommand{\p}{\vp} % point coordinates
\newcommand{\cl}{_{\c\l}} % camera point
\newcommand{\ucl}{\vu_{\c\l}}

\newcommand{\measImu}[1]{\meas^\imu_{#1}}
\newcommand{\noiseImu}[1]{\noise^\imu_{#1}}
\newcommand{\MOmegaImu}[1]{\MOmega^\imu_{#1}}
\newcommand{\sigmaImu}{\sigma_\imu}
\newcommand{\cov}{\text{cov}}

\newcommand{\measU}[1]{\meas^\cam_{#1}}

\newcommand{\noiseu}[1]{\noise^\cam_{#1}}

\newcommand{\setSopt}{\setS^\star}
\newcommand{\setSgreedy}{\setS^\#}
\newcommand{\setScvx}{\setS^\circ}

\newcommand{\extra}[1]{{\color{black}#1}}

\newcommand{\Hz}{\text{Hz}}
\renewcommand{\sec}{\text{s}}

\newcommand{\vmubar}{\bar\vmu}
\newcommand{\vxxest}{\hat{\vxx}}

\renewcommand{\baselinestretch}{.985} 

\newcommand{\edited}[1]{{#1}\xspace} 
\newcommand{\editedTwo}[1]{{#1}\xspace} 

\title{\huge{Attention and Anticipation in Fast Visual-Inertial Navigation}}

\author{Luca Carlone and Sertac Karaman
\thanks{L.\,Carlone and S.\,Karaman are with the Laboratory for 
Information \& Decision Systems (LIDS), Massachusetts Institute of Technology, Cambridge, MA, USA, 
{\sf \{lcarlone,sertac\}@mit.edu}
}}

\begin{document}

\maketitle

%\LC{https://www.ucdavis.edu/news/visual-attention-drawn-meaning-not-what-stands-out}

%%%%%%%%%%%%%%%%%%%%%%%%%%%%%%%%%%%%
\input{abstract}

\begin{tikzpicture}[overlay, remember picture]
\path (current page.north east) ++(-4,-0.2) node[below left] {
This paper is conditionally accepted for publication in the IEEE Transaction on Robotics.
};
\end{tikzpicture}
\begin{tikzpicture}[overlay, remember picture]
\path (current page.north east) ++(-1.4,-0.6) node[below left] {
 Please cite the conference version as:
 L. Carlone and S. Karaman, 
``Attention and anticipation in fast visual-inertial navigation''
};
\end{tikzpicture}
\begin{tikzpicture}[overlay, remember picture]
\path (current page.north east) ++(-4.9,-1) node[below left] {
IEEE Intl. Conf. on Robotics and Automation (ICRA), pp. 3886-3893, 2017.
};
\end{tikzpicture}
\vspace{-5mm}

\section*{Supplementary Material}
\begin{itemize}
  \item Video: {\small \url{https://www.youtube.com/watch?v=uMLXNRiVuyU}}
\end{itemize}

%%%%%%%%%%%%%%%%%%%%%%%%%%%%%%%%%%%%
\input{intro}

%%%%%%%%%%%%%%%%%%%%%%%%%%%%%%%%%%%%
\input{relatedWork}

  \input{preliminaries}
%%%%%%%%%%%%%%%%%%%%%%%%%%%%%%%%%%%%
\input{attentionVIN}

  \input{featureSelection}
	\input{performanceMetrics}
	\input{forwardSimulationModel}
		\input{imuLinearModel}
		\input{visionLinearModel}
    \input{remarkLinear}
%%%%%%%%%%%%%%%%%%%%%%%%%%%%%%%%%%%%
\input{algorithmsAndGuarantees}

  \input{convexRelaxation}
  \input{greedy}
  \input{guarantees}

%%%%%%%%%%%%%%%%%%%%%%%%%%%%%%%%%%%%
\input{experiments}
%%%%%%%%%%%%%%%%%%%%%%%%%%%%%%%%%%%%
%\input{extensions}
%%%%%%%%%%%%%%%%%%%%%%%%%%%%%%%%%%%%
\input{conclusion}

\section*{Acknowledgments}
The authors gratefully acknowledge Christian Forster for the useful discussions 
and for providing the VIN simulator, and 
 % used for the Monte Carlo analysis. The authors are also grateful to
\edited{ the anonymous reviewers for providing invaluable comments to improve the quality 
of this manuscript.}
%%%%%%%%%%%%%%%%%%%%%%%%%%%%%%%%%%%%
\appendix
\input{appendix-linearImuModel}
\input{appendix-proofLongerFeatureTracks}
\input{appendix-proofUppenBoundLambdaMin}
\input{appendix-submodularityRatio}
%%%%%%%%%%%%%%%%%%%%%%%%%%%%%%%%%%%%

\bibliographystyle{IEEEtran}
\bibliography{refs,myRefs}

\end{document}

%% file: preamble_packages.tex
\usepackage{comment}
\usepackage{siunitx}
\usepackage{relsize}
\usepackage{ifthen}

\usepackage[caption=false]{subfig}

\usepackage[vlined,ruled,linesnumbered]{algorithm2e}
\usepackage{graphics} % for pdf, bitmapped graphics files
\usepackage{rotating}
\usepackage{color}
\usepackage{enumerate}
\usepackage[T1]{fontenc}
\usepackage{psfrag}
\usepackage{epsfig} % for postscript graphics files
\usepackage{booktabs}
\usepackage{graphicx,url}
\usepackage{multirow}
\usepackage{array}
\usepackage{latexsym}
\usepackage{amsfonts}
\usepackage{amsmath}
\usepackage{amssymb}
\usepackage{xstring}
\usepackage[noend]{algorithmic}
\usepackage{multirow}
\usepackage{xcolor}
\usepackage{prettyref}
\usepackage{flexisym}
\usepackage{bigdelim}
\usepackage{breqn} % load this last
\usepackage{listings}
\let\labelindent\relax
\usepackage{enumitem}
\usepackage{xspace}
\usepackage{bm}
\graphicspath{{./figures/}}
\usepackage{tikz}
\usetikzlibrary{matrix,calc}

\usepackage{mdwlist}

\let\itemize\undefined
\makecompactlist{itemize}{stditemize}

%% file: preamble_symbols.tex
\newrefformat{prob}{Problem\,\ref{#1}}
\newrefformat{def}{Definition\,\ref{#1}}
\newrefformat{sec}{Section\,\ref{#1}}
\newrefformat{sub}{Section\,\ref{#1}}
\newrefformat{prop}{Proposition\,\ref{#1}}
\newrefformat{app}{Appendix\,\ref{#1}}
\newrefformat{alg}{Algorithm\,\ref{#1}}
\newrefformat{cor}{Corollary\,\ref{#1}}
\newrefformat{thm}{Theorem\,\ref{#1}}
\newrefformat{lem}{Lemma\,\ref{#1}}
\newrefformat{fig}{Fig.\,\ref{#1}}
\newrefformat{tab}{Table\,\ref{#1}}

\newtheorem{theorem}{Theorem}

\newtheorem{corollary}[theorem]{Corollary}

\newtheorem{definition}[theorem]{Definition}
\newtheorem{proposition}[theorem]{Proposition}
\newtheorem{remark}[theorem]{Remark}

\newcommand{\bdmath}{\begin{dmath}}
\newcommand{\edmath}{\end{dmath}}
\newcommand{\beq}{\begin{equation}}
\newcommand{\eeq}{\end{equation}}
\newcommand{\bdm}{\begin{displaymath}}
\newcommand{\edm}{\end{displaymath}}
\newcommand{\bea}{\begin{eqnarray}}
\newcommand{\eea}{\end{eqnarray}}
\newcommand{\beal}{\beq \begin{array}{ll}}
\newcommand{\eeal}{\end{array} \eeq}
\newcommand{\beas}{\begin{eqnarray*}}
\newcommand{\eeas}{\end{eqnarray*}}
\newcommand{\ba}{\begin{array}}
\newcommand{\ea}{\end{array}}
\newcommand{\bit}{\begin{itemize}}
\newcommand{\eit}{\end{itemize}}
\newcommand{\ben}{\begin{enumerate}}
\newcommand{\een}{\end{enumerate}}

\newcommand{\setA}{\textsf{A}}
\newcommand{\setB}{\textsf{B}}

\newcommand{\setE}{\textsf{E}}
\newcommand{\setF}{\textsf{F}}

\newcommand{\setH}{\textsf{H}}

\newcommand{\setL}{\textsf{L}}

\newcommand{\setS}{\textsf{S}}

\newcommand{\setal}{~\emph{et~al.}\xspace}
\newcommand{\M}[1]{{\bm #1}} % Face for matrices
\renewcommand{\boldsymbol}[1]{{\bm #1}}
\newcommand{\hide}[1]{}

\newcommand{\grayText}[1]{{\color{gray} \text{#1} }}

\newcommand{\hiddenText}{{\color{gray} hidden text.}}
\newcommand{\hideWithText}[1]{\hiddenText}

\newcommand{\subject}{\text{ subject to }}

\newcommand{\E}{{\mathbb{E}}}

\newcommand{\prob}[1]{{\mathbb P}\left(#1\right)}
\newcommand{\tran}{^{\mathsf{T}}}

\newcommand{\e}{{\mathrm e}}
\newcommand{\inv}{^{-1}}

\newcommand{\zero}{{\mathbf 0}}
\newcommand{\eye}{{\mathbf I}}
\newcommand{\vect}[1]{\left[\begin{array}{c}  #1  \end{array}\right]}
\newcommand{\matTwo}[1]{\left[\begin{array}{cc}  #1  \end{array}\right]}

\newcommand{\Real}[1]{ { {\mathbb R}^{#1} } }

\newcommand{\MA}{\M{A}}
\newcommand{\MB}{\M{B}}
\newcommand{\MC}{\M{C}}

\newcommand{\ME}{\M{E}}

\newcommand{\MM}{\M{M}}
\newcommand{\MN}{\M{N}}
\newcommand{\MP}{\M{P}}
\newcommand{\MQ}{\M{Q}}

\newcommand{\MR}{\M{R}}

\newcommand{\MF}{\M{F}}

\newcommand{\MOmega}{\M{\Omega}}

\newcommand{\va}{\boldsymbol{a}} 
\newcommand{\vb}{\boldsymbol{b}}

\newcommand{\vp}{\boldsymbol{p}}

\newcommand{\vu}{\boldsymbol{u}}
\newcommand{\vv}{\boldsymbol{v}}
\newcommand{\vt}{\boldsymbol{t}}
\newcommand{\vxx}{\boldsymbol{x}} 

\newcommand{\vw}{\boldsymbol{w}}

\newcommand{\vnu}{\boldsymbol{\nu}}
\newcommand{\vmu}{\boldsymbol{\mu}}

\newcommand{\scenario}[1]{{\smaller \sf#1}\xspace}

\newcommand{\CVX}{\scenario{CVX}}

\newcommand{\MOSEK}{\scenario{MOSEK}}

\newcommand{\cvx}{{\sf cvx}\xspace}

\newcommand{\blue}[1]{{\color{blue}#1}}

\newcommand{\linkToPdf}[1]{\href{#1}{\blue{(pdf)}}}
\newcommand{\linkToPpt}[1]{\href{#1}{\blue{(ppt)}}}
\newcommand{\linkToCode}[1]{\href{#1}{\blue{(code)}}}
\newcommand{\linkToWeb}[1]{\href{#1}{\blue{(web)}}}
\newcommand{\linkToVideo}[1]{\href{#1}{\blue{(video)}}}
\newcommand{\award}[1]{\xspace} % {{\red{#1}}} % omit awards

\newcommand{\pos}{\boldsymbol{p}}
\newcommand{\meas}{\boldsymbol{z}} % vector

%% file: shortcuts.tex
\newcommand{\world}{\text{\tiny{W}}}
\newcommand{\imu}{\text{\tiny{IMU}}}

\newcommand{\prior}{\text{\tiny{PRIOR}}}
\newcommand{\cam}{{\scriptsize \text{cam}}}

\newcommand{\meters}{\rm{m}}

\renewcommand{\P}{\mathtt{t}}
\newcommand{\V}{\mathtt{v}}

\newcommand{\acc}{\mathbf{a}}

\newcommand{\bias}{\mathbf{b}}

\newcommand{\gravity}{\mathbf{g}}
\newcommand{\noise}{\boldsymbol\eta}

\newcommand{\half}{\frac{1}{2}}

\newcommand{\nrVisible}{n_\ell}

\newcommand{\myZ}{\zero}

%% file: abstract.tex
%!TEX root = main.tex

\begin{abstract} 
We study a \emph{Visual-Inertial Navigation} (VIN) problem in which a robot needs to estimate
its state using an on-board camera and an inertial sensor, without any prior knowledge of the 
external environment. 
We consider the case in which 
the robot can allocate limited resources to 
VIN, due to tight computational constraints. % (e.g., power constraints on embedded platforms). 
Therefore, we answer the following question: under limited resources, what are the 
most relevant visual cues to maximize the performance of visual-inertial navigation?
Our approach has four key ingredients. First, it is \emph{task-driven}, in that the 
selection of the visual cues is guided by a metric quantifying the VIN performance.
Second, it exploits the notion of \emph{anticipation}, since it uses a simplified model 
for forward-simulation of robot dynamics, predicting the utility of a set of visual cues 
over a future time horizon. 
% fast motion
% simplified model
% Predictive framework: looks at robot intentions
% another basic idea: simple model for forward simulation to decide what's relevant
Third, it is \emph{efficient and easy to implement}, since it leads to a greedy algorithm
for the selection of the most relevant visual cues.
Fourth, it provides \emph{formal performance guarantees}: we leverage submodularity to prove 
that the greedy selection cannot be far from the optimal (combinatorial) selection.
% win combinatorial complexity, guarantees, dramatic improvements in tests?
% task-driven: fast navigation
\editedTwo{Simulations and real experiments on agile drones show that our approach 
ensures state-of-the-art VIN performance while maintaining a lean processing time}. 
% leads 
% to dramatic improvements in the VIN performance. 
In the easy scenarios, our 
approach outperforms \editedTwo{appearance-based feature selection} in terms of localization errors. 
In the most challenging scenarios, 
%a clever choice of the visual features 
%can make the difference between successful state estimation and complete failure. 
% Under severe resource constraints, 
 it enables accurate visual-inertial navigation 
  while \editedTwo{appearance-based feature selection} fails to track robot's motion during aggressive maneuvers. 
% We also show that out approach leads to a sensible choice of features and discuss why it is a 
% necessary ingredient for fast navigation. % of robots.
\end{abstract}

%% file: intro.tex
%!TEX root = main.tex

\section{Introduction}
\label{sec:intro}

The human brain %visual processing chain 
can extract conceptual information from an image in 
a time lapse as short as $13$ ms~\cite{Potter14app-attention}. % and $80$ 
% Human brain can process information quickly
% http://www.livescience.com/42666-human-brain-sees-images-record-speed.html
% http://mollylab-1.mit.edu/lab/publications.html
One has proof of the human's capability to seamlessly process large amount of sensory data in everyday tasks,
including driving a car on a highway, or walking on a crowded street.
In the cognitive science literature, there is agreement on the fact that 
efficiency in processing the large amount of data we are confronted with is due to our ability to prioritize  
some aspects of the visual scene, while ignoring others~\cite{Carrasco11vr-attentionSurvey}.
One can imagine that sensory inputs compete to have access to the limited computational resources of our brain.
These resource constraints are dictated by the fixed amount of energy available to the brain as well as 
time constraints imposed by time-critical tasks.
\emph{Visual attention} is the cognitive process that allows humans to parse a large amount of visual data 
by selecting relevant information and filtering out irrelevant stimuli,
%This papers develops a computational framework for visual attention in robots.
%\emph{Attention} is the selective process that chooses the most relevant stimuli  
so to maximize performance\footnote{This definition oversimplifies the attention mechanisms in humans. 
While 
the role of attention is to optimally allocate resources to maximize performance, 
it is known that some involuntary attention mechanisms can actually hinder the correct execution of a task~\cite{Carrasco11vr-attentionSurvey}.} under limited resources. 
%Attention as a selective process that affects early visual processing
%``Attention is a selective process. Selection is necessary because there are severe limits on our capacity to process visual information. 
%These limits are likely imposed by the fixed amount of overall energy available to the brain and by the high-energy cost of the neuronal activity involved in cortical computation.''.
%decision making filters sensory information to adapt them to available computational resources
%(limited cognitive and brain resources) while getting optimal results
%These results are consistent with the idea that selective attention results in greater resource allocation to the 
% attended location, at the cost of available resources at the unattended location.

\mysubparagraph{Robots vs. humans} The astonishing progress in robotics and computer vision over the last three decades 
might induce us to ask: how far is robot perception from human performance? 
Let us approach this question by looking at the 
state of the art in visual processing for different tasks. 
Without any claim to be exhaustive, we %purposely 
consider few representative papers (sampled 
over the last \edited{3} years) and we only look at timing performance.
A state-of-the-art approach for object detection~\cite{Redmon16cvpr} detects objects in a scene
 in 22ms on a Titan X GPU. 
A high-performance approach for stereo reconstruction~\cite{Pillai16icra} builds a triangular mesh of a 3D scene in  
10-100ms on a single CPU (at resolution $800\times 600$).
% The computatoin of
% per image, which excludes the popular SIFT (∼ 300ms) [19],
% D. G. Lowe, “Distinctive image features from scale-invariant keypoints,”
% International Journal of Computer Vision, vol. 60, no. 2, pp. 91–110,
% 2004
%A visual inertial odometry algorithm~\cite{Forster15rss} performs motion estimation in 10-50ms (CPU, multiple cores).
A state-of-the-art vision-based SLAM approach~\cite{Mur-Artal15tro} requires around 400ms for local mapping and motion tracking 
and more than 1s for global map refinement (CPU, multiple cores).
% Object detection~\cite{Redmon16cvpr}: 22ms (on GPU)
% Stereo reconstruction~\cite{Pillai15icra}: 10-100 (on single-CPU, resolution 800x600) 
% Visual inertial odometry~\cite{Forster15rss}: 10-40ms (on CPU, multiple cores)
The reader may notice that for each task, in isolation, modern algorithms require more time than 
what a human needs to parse an entire scene. Arguably, while a merit of the robotics 
and computer vision communities has been to push performance in each task, we 
are quite far from a computational model in which all these tasks
 (pose estimation, geometry reconstruction, scene understanding) are concurrently %seamlessly 
 executed in the blink of an eye. 

\mysubparagraph{Efficiency via general-purpose computing} One might argue that catching up with human efficiency is only a matter of time: according to Moore's law, 
the available computational power grows at exponential rate, hence we 
% exponential growth with transistor count doubling every two years.
only need to wait for more powerful computers. An analogous argument would suggest that using GPU rather than CPU
would boost performance in some of the tasks mentioned above. 
By comparison with human performance, we realize that this argument is not completely accurate.
While it is true that we can keep increasing the computational resources to meet given 
time constraints (i.e., enable faster processing of sensory data), the increase in computation implies an 
increase in energy consumption; for instance, a Titan X GPU has a nominal power consumption of 250W~\cite{TitanXwebsite} 
while a Core i7 CPU has a power consumption as low as 11W~\cite{processorsWebsite}.
On the other hand, human processing constantly deals with limited time and energy constraints, and 
is parsimonious in allocating only the resources that are necessary to accomplish its goals.

\mysubparagraph{Efficiency via specialized computing} Another potential alternative to 
enable high-rate low-power perception 
and bridge the gap between  human and robot perception 
is to design specialized hardware for machine perception. As extensively discussed 
in our previous work~\cite{Zhang17rss-vioChip}, algorithms and hardware co-design allows minimizing resource utilization by exploiting a tight-integration of algorithms and specialized hardware, 
and leveraging opportunities (e.g., pipelining, low-cost arithmetic) provided by ASICs 
(Application-Specific Integrated Circuits) and FPGAs 
(Field-Programmable Gate Arrays). While we have shown that using specialized hardware for VIN 
leads to a reduction of the power consumption of 1-2 orders of magnitude (with comparable performance), 
three main observations motivate the present work.
First, the development of specialized hardware for perception is an expensive and time-consuming process 
and the resulting hardware is difficult to upgrade. % or extend.   
Second, rather than designing optimized hardware that can meet given performance requirements, 
it may be desirable to develop a framework that can systematically trade-off performance for 
computation, hence more flexibly adjusting to the available, possibly time-varying, computational resources and performance requirements. 
Third, extensive biological evidence suggests that efficient perception requires both 
specialized circuitry (e.g., visual perception in humans is carried out by highly specialized areas of the 
brain~\cite{Cavanagh11vr-visualCognition}) and a mechanism to prioritize stimuli (i.e., visual attention~\cite{Carrasco11vr-attentionSurvey}).

\myparagraph{Contribution}
\edited{In this paper, we investigate how to speed-up computation (or, equivalently, reduce the computational effort)
 in visual navigation by prioritizing sensor data in a task-dependent fashion.
% Our goal is to design an approach that can 
% optimally allocate the limited on-board resources of a robot to maximize localization performance.
% In particular, 
% we propose an approach for visual-inertial navigation under limited computational resources.
In particular, we focus on a motion estimation task, \emph{Visual-Inertial Navigation} (VIN), 
and consider the case in which,}
due to constraints on the on-board computation, a robot can only use a small number of visual features in the 
environment to support motion estimation. We then design a visual attention mechanism that selects a suitable 
set of visual features to maximize localization accuracy; our general framework is presented in~\prettyref{sec:attention}. 

Our approach is task-driven: 
\edited{we consider a motion estimation task (VIN), and our approach} selects features that maximize a task-dependent performance metric, that 
we present in~\prettyref{sec:metrics}.
Contrarily to the literature on visual feature selection, we believe that the utility of a feature is not an intrinsic 
property of the feature itself (e.g., appearance), but it rather stems from the intertwining 
of the environment and the observer state. Our approach seamlessly captures both visual saliency and the 
task-dependent utility of a set of features.

%; this is consistent with recent frameworks for attention and saliency in 
%human navigation~\cite{Caduff08cp-attention}. 
%  - main message: saliency is not an intrinsic property of a feature but stems from the intertwining 
% of feature observed and environment.
% - Attention is a psychological construct that describes detection, selection, discrimination of stimuli, as well as allocation of 
% limited cognitive resources to competing attentional demands (Scholl 2001).
%Accordingly, our approach 
% our attention mechanism 
Our attention mechanism is predictive in nature:
 when deciding which feature is more useful, our approach
performs fast forward-simulations of the state of the robot leading to a feature selection that is aware 
of the dynamics and the ``intentions'' of the robot.
%to select the more informative features depending on robot's intentions. 
The forward simulation is based on a simplified model 
which we present in~\prettyref{sec:estimationModel}. %A general approach for feature selection is presented in~\prettyref{sec:featureSelection}.

Since the \editedTwo{optimal allocation of the resources is} a hard combinatorial problem, 
in~\prettyref{sec:algorithmsAndGuarantees} 
we present a greedy algorithm for attention allocation. 
%The algorithm selects a subset of visual features according to the available computational resources.
In the same section, 
%we also recall another technique, based on convex relaxation.
%\prettyref{sec:algorithmsAndGuarantees}  
we leverage recent results on submodularity 
to provide formal performance guarantees for the greedy algorithm.
 \prettyref{sec:algorithmsAndGuarantees}  also reviews related techniques
%The same section also recalls related techniques 
based on convex relaxations.

\prettyref{sec:experiments} provides an experimental evaluation of the proposed approach. The results confirm that our 
approach can boost performance in standard VIN pipelines and enables accurate navigation %on platforms 
under agile motions and strict resource constraints.
The proposed approach largely outperforms appearance-based feature selection methods, and 
drastically reduces the computational time required by the VIN back-end.
%with strict 
%resource constraints.

This paper extends the preliminary results presented in~\cite{Carlone17icra-vioAttention}. % in multiple directions. 
In particular, the discussion on convex relaxations for features selection (\prettyref{sec:convexRelax} and \prettyref{sec:convexRelaxGuarantees}), 
the performance guarantees of \prettyref{prop:nonVanishingRatio}, 
the simulation results of \prettyref{sec:easySim}, and 
the experimental evaluation on the 11 \emph{EuRoC} datasets~\cite{Burri16ijrr-eurocDataset} 
are novel and have not been previously published.

%% file: relatedWork.tex
%!TEX root = main.tex

\section{Related Work}
\label{sec:relatedWork}

\extra{This work intersects several lines of research across  fields.}
%We review them in this section, 
% pointing the reader to the sources 
%%%%%%%%%%%%%%%%%%%%%%%%%%%%%%%%%%%%%%%%%%%%%%%%%%%%%%%%%%%%%%%%%%%%%%%%%%%%%%%%%%%%%%%%%%%%%%%%%%%%%%%%%%%%%%%%%%

\myparagraph{Attention and Saliency in  Neuroscience and Psychology}  % Cognitive
 Attention is a central topic in human and animal vision research with more than 2500 papers published since 
 the 1980s~\cite{Carrasco11vr-attentionSurvey}. While a complete coverage is outside the scope of this work, we 
review few basic concepts, using the surveys of Carrasco~\cite{Carrasco11vr-attentionSurvey}, 
Borji and Itti~\cite{Borji13pami-surveyAttention}, Scholl~\cite{Scholl01cog-attentionSurvey},
and the work of Caduff and Timpf~\cite{Caduff08cp-attention} as main references.
Scholl~\cite{Scholl01cog-attentionSurvey} defines attention as the discrimination of sensory stimuli, and the allocation of 
 limited resources to competing attentional demands. Carrasco~\cite{Carrasco11vr-attentionSurvey} identifies 
three types of attention: \emph{spatial}, \emph{feature-based}, and \emph{object-based}. Spatial attention prioritizes 
different locations of the scene by moving the eyes towards a specific location (\emph{overt} attention)
or by focusing on relevant locations within the field of view (\emph{covert} attention).
Feature-based attention prioritizes the detection of a specific feature (color, motion direction, orientation) 
independently on its location. Object-based attention prioritizes specific objects.
In this work, we are mainly interested in covert spatial attention: which locations in the field of view are 
the most informative for navigation?  Covert attention in humans is a combination of voluntary and involuntary 
mechanisms that guide the processing of visual stimuli at given locations in the scene~\cite{Carrasco11vr-attentionSurvey}.
Empirical evidence shows that attention is task-dependent in both primates and humans~\cite{Caduff08cp-attention,Moran85science-attention}.
Borji and Itti~\cite{Borji13pami-surveyAttention} explicitly capture this aspect by distinguishing bottom-up and top-down attention models;
in the former the attention is captured by visual cues (stimulus-driven), while in the latter the attention is guided by the 
goal of the observer. 
Caduff and Timpf~\cite{Caduff08cp-attention} study landmark saliency in human navigation and conclude that 
saliency stems from the intertwining of intrinsic property of a landmark (e.g., appearance) 
and the state of the observer  (e.g., prior knowledge, observation pose). 
%The paper also reviews 
%related work on landmark saliency in the 
%Geographic Information Science literature.
Another important aspect, that traces back to the \emph{guided search theory} of Wolfe~\cite{Wolfe94pbr-attention} and Spekreijse~\cite{Spekreijse00vr-attention}, 
is the distinction between pre-attentive and attentive visual processes. 
Pre-attentive processes handle \emph{all} incoming sensory data in parallel; then, attentive processes only 
work on a filtered-out-version of the data, which the brain deems more relevant. 
% All visual stimuli across the visual field are processed in parallel, and the most salient regions are attended.
%
%  Pre-attentive VS attentive visual processes:   Wiki: Pre-attentive processing is the subconscious accumulation of information from the environment.[1][2] All available information is pre-attentively processed.[2] Then, the brain filters and processes what is important.
%
% Model of Wolfe (from Caduff):
% One of the most influential theories of visual search is the guided search theory (Wolfe 1994). It suggests a two-stage model of visual processing. In the pre-attentive stage, feature maps are computed in parallel in several feature dimensions (e.g., red, blue, green, and yellow features for color; steep, shallow, left, and right maps for orientation). In the second stage, top-down factors modulate the bottom- up values, and the weighted feature maps are combined additively to form an activation map that eventually guides visual attention in a sequential manner.
General computational models for attention are reviewed in~\cite{Borji13pami-surveyAttention}, including Bayesian models, 
graph-theoretic, and information-theoretic formulations. 
\myparagraph{Feature Selection in Robotics and Computer Vision}
The idea of enhancing performance in visual SLAM 
and visual odometry via active feature selection is not novel. %, with most contributions being very recent.
Sim and Dudek~\cite{Sim99iccv-featureSelection} and Peretroukhin\setal~\cite{Peretroukhin15iros-predictiveVIO} 
use training data to learn a model of the quality of the visual features. 
Each feature is mapped from  a hand-crafted predictor space to a scalar weight that quantifies 
its usefulness for pose estimation; 
in~\cite{Peretroukhin15iros-predictiveVIO}  the weights are then used to rescale the measurement covariance 
of each observation.
% Nickerson et al. [30] detect landmarks in
% hand-coded maps,
 Ouerhani\setal~\cite{Ouerhani05ecmr-attention} construct a topological map using attentional landmarks.
Newman and Ho~\cite{Newman05icra-salientFeatures} consider a robot equipped with camera and laser range finder 
and perform feature selection using an appearance-based notion of visual saliency.
Sala\setal~\cite{Sala06tro-landmarkSelection} use a co-visibility criterion to select good landmarks that are visible from multiple viewpoints.
Siagian and Itti~\cite{Siagian07iros-attention} investigate a \edited{bio-inspired} attention model 
within Monte Carlo localization.
Frintrop and Jensfelt~\cite{Frintrop08tro-attentionLandmarks} use an %bio-inspired  visual 
attention framework for 
 landmark selection and active gaze control; feature selection is based on the VOCUS 
 %Frintrop07acs,
 model~\cite{Frintrop08tro-attentionLandmarks}, which includes 
  a bottom-up attentional system (which computes saliency from the feature appearance), and can incorporate
  a top-down mechanism (which considers task performance). Active gaze control, instead, is obtained as the 
  combination of three behaviors: landmark redetection, landmark tracking, and exploration of new areas.
Hochdorfer and Schlegel~\cite{Hochdorfer09icra-landmarksSelection} propose a landmark rating and selection mechanism 
based on area coverage to enable life-long mapping.
Strasdat\setal~\cite{Strasdat09icra-landmarksSelection} propose a reinforcement learning approach for landmark selection.
Chli and Davison~\cite{Chli09ras} and Handa\setal~\cite{Handa10cvpr} 
use available priors to inform feature matching, hence reducing the computational cost.
Jang\setal~\cite{Jang13icufn-featureSelection} propose an approach for landmark classification to improve accuracy in visual odometry; 
 each feature class is used separately to estimate rotational and translational components of the ego-motion.
Shi\setal~\cite{Shi13mva-featureSelection} propose a feature selection technique %that selects features 
to improve robustness of data association in SLAM.
\edited{The notion of visual attention and saliency has been also investigated in the context of scene 
understanding. For instance, Oliva and Torralba~\cite{Oliva01ijcv} propose the notion of \emph{Space Envelop} to obtain a coarse 
description of a scene which abstracts away unnecessary details, while 
Torralba\setal~\cite{Torralba06pr} propose an attention mechanism for natural search tasks, which 
combines bottom-up saliency and top-down aspects to identify the image region that is likely to 
attract the attention of human observer searching for a given object.
}
Very recent work in  computer vision  use attention to reduce 
the computational burden in neural networks.
Mnih\setal~\cite{Mnih14arxiv-attention} reduce the processing of object detection and tracking with
a recurrent neural network by introducing the notion of 
\emph{glimpse}, which provides higher resolution in areas of interest within the image.
Xu\setal~\cite{Xu16arxiv-attention} use visual attention to improve image content description.
% describe
% the content of images
% \LC{add https://arxiv.org/pdf/1502.03044.pdf}
% ``This paper introduced a novel visual attention model that is formulated as a single recurrent neural network which takes a glimpse window as its input and uses the internal state of the network to select the next location to focus on as well as to generate control signals in a dynamic environment.''
% ``One important property of human perception is that one does not tend to process a whole scene in its entirety at once. Instead humans focus attention selectively on parts of the visual space to acquire information when and where it is needed, and combine information from different fixations over time to build up an internal representation of the scene''
% ``In line with its fundamental role, the guidance of human eye movements has been extensively studied in neuroscience and cognitive science literature. While low-level scene properties and bottom up processes (e.g. in the form of saliency; [11]) play an important role, the locations on which humans fixate have also been shown to be strongly task specific''
Cvi\v{s}i\'c and Petrovi\'c~\cite{Cvisic15ecmr-featureSelection} speed up computation in stereo odometry by feature selection;
the selection procedure is based on bucketing (which uniformly distributes the features across the image), and 
appearance-based ranking. 

%their repeatability and compatibility with other features?
Our approach is loosely related to techniques for graph sparsification in which features are pruned a-posteriori
from the SLAM factor graph to reduce computation; we refer the reader to the survey~\cite{Cadena16tro-SLAMsurvey}
for a review of these techniques.
% Bisacco: dedicated HW to track larger number of features at high frame rate (FPGA implementation)
% Pasquale15
% gaze selection?
% Enabling Depth-driven Visual Attention on the iCub Humanoid Robot: Instructions for Use and New Perspectives
% Giulia Pasquale, Tanis Mar, Carlo Ciliberto, Lorenzo Rosasco, Lorenzo Natale
% http://arxiv.org/abs/1509.06939

The contributions that are most relevant to our proposal are the one of
Davison~\cite{Davison05iccv}, 
Lerner\setal~\cite{Lerner07tro}, 
Mu\setal~\cite{Mu15rss}, 
Wu\setal~\cite{Wu16iser-vins},
and Zhang and Vela~\cite{Zhang15cvpr}.
The pioneering work of Davison~\cite{Davison05iccv} is one of the first papers 
to use information theoretic constructs to reason about visual features, and 
shares many of the motivations discussed in the present paper.
Contrarily to the present paper, Davison~\cite{Davison05iccv} considers a model-based tracking problem 
in which the state of a moving camera has to be estimated from observations of known features.
%Moreover, in this paper we propose the use of fast linear models to anticipate the utility of selecting a subset of features.
In hindsight, we also provide a theoretical justification for the use of a greedy algorithm (similar to the one used in~\cite{Davison05iccv})
  which we prove able to compute near-optimal solutions. 
Lerner\setal~\cite{Lerner07tro} study landmark selection in a localization problem
 with known landmarks; the robot has to choose a subset of landmarks to 
 observe so to minimize the localization uncertainty.
 % the 3D position of external landmarks is known a priori and uses a trace criterion to select the optimal set of landmarks to minimize pose covariance. 
The optimal subset selection is formulated as a mixed-integer program and 
relaxed to an SDP. %The solution is the rounded with a randomized procedure. 
While the problem we solve in this 
paper is different (visual inertial odometry vs. localization with known map), an interesting aspect of~\cite{Lerner07tro}
 is the use of a \emph{requirement matrix} that weights the state covariance and encodes task-dependent uncertainty constraints.
 %;\extra{we will borrow a similar idea in~\prettyref{sec:conclusion}.}
Mu\setal~\cite{Mu15rss} propose a two-stage approach to select a subset of landmarks to minimize the 
probability of collision and a subset of measurements to accurately localize those landmarks.
Our approach shares the philosophy of task-driven measurement selection, but has three key differences.
First, we use a simplified linear model for forward dynamics simulation: this is in the spirit of 
RANSAC, in that a simplified algebraic model is used to quickly filter out less relevant data.
Second, we consider different performance metrics, going beyond the determinant criterion used in~\cite{Mu15rss} 
and related work on graph sparsification.
%~\cite{Ila10tro}. 
Third, we perform feature selection in a single stage
and  leverage submodularity to provide formal performance guarantees. % on our feature selection mechanism.
Wu\setal~\cite{Wu16iser-vins} consider a multi camera system and split the feature selection process into a cascade of 
two resource-allocation problems: (i) how to allocate resources among the cameras, and (ii) how to select features in each camera, 
according to the allocated resources. The former problem is \editedTwo{solved by taking simplifying assumptions} on the distribution of the 
features, the latter is based on the heuristic feature selection scheme of~\cite{Kottas14icra-vins}.
Our paper attempts to formalize feature selection by leveraging the notion of submodularity.
\edited{Zhang and Vela~\cite{Zhang15cvpr} perform feature selection using an observability score and provide sub-optimality guarantees
using submodularity. Our proposal is similar in spirit to~\cite{Zhang15cvpr} with few important differences. 
First, our approach is based on \emph{anticipation}: the feature selector is aware of the intention (future motion) 
of the robot and selects the features accordingly. Second, we operate in a fixed-lag smoothing setup and investigate other performance metrics.
%do not leverage the presence of anchor points. 
% (the experiments in \cite{Zhang15cvpr} use \emph{anchor} points for motion estimation).}
 % some of the features in~\cite{Zhang15cvpr} can be fixed 
%the anchors, w in~\cite{Zhang15cvpr})}.
Third, 
%we consider a metric quantifying the motion estimation error and we purposely disregard the map reconstruction quality.
% Fourth, 
from the theoretical standpoint, we provide multiplicative suboptimality bounds 
%which are stronger than the 
% additive bound of~\cite{Zhang15cvpr}, 
and we prove conditions under which those bounds are non-vanishing.}
\myparagraph{Sensor Scheduling and Submodularity}
Feature selection is deeply related to the problem of sensor scheduling in control theory.
The most common setup for sensor scheduling is the case in which $N$ sensors monitor a phenomenon of interest and one has to choose 
$\kappa$ out of the $N$ 
available sensors to maximize some information-collection metric; this setup is also known as 
 \emph{sensor selection} or \emph{sensor placement}. The literature on sensor selection includes 
 approaches based on convex relaxation~\cite{Joshi09tsp-sensorSelection}, Bayesian optimal design~\cite{Giraud95iros-sensorSelection}, and
submodular optimization~\cite{Krause08jmlr-submodularity}. The problem is shown to be NP-hard in~\cite{Bian06ipsn-sensorSelection}.
Shamaiah\setal~\cite{Shamaiah10cdc-sensorScheduling} leverage submodularity and provide performance guarantees 
when optimizing the log-determinant of the estimation error covariance.
A setup which is closer to the one in this paper is the case in which the sensed phenomenon is dynamic;
in such case the sensor scheduling can be formulated in terms of the optimal selection of 
$\kappa$ out of $N$ possible measurements to be used in the update phase of a Kalman filter (KF). 
Vitus\setal~\cite{Vitus08automatica-sensorScheduling} use a tree-search approach for 
sensor scheduling.
Zhang\setal~\cite{Zhang15cdc-sensorSelection} proves that sensor scheduling within Kalman filtering is NP-hard 
and shows that the trace of the steady state prior and posterior KF covariances are not submodular, 
despite the fact that greedy 
algorithms are observed to work well in practice.
Jawaid and Smith~\cite{Jawaid15automatica-sensorScheduling} provide counterexamples showing that 
in general the maximum eigenvalue and the trace of the covariance are not submodular.
Tzoumas\setal~\cite{Tzoumas16acc-sensorScheduling} generalize the derivation of~\cite{Jawaid15automatica-sensorScheduling} 
to prove submodularity of the logdet over a fixed time horizon, under certain assumptions 
on the observation matrix.
Summers\setal~\cite{Summers16tcns-sensorScheduling} show that several metrics based on the 
controllability and observability Gramians are submodular.
\myparagraph{Visual-Inertial Navigation}
As the combined use of the visual and vestibular system is key to human navigation, 
recent advances in visual-inertial navigation 
on mobile robots are enabling unprecedented performance in pose estimation in GPS-denied environments using
commodity hardware.
% In this paper we argue that a crucial component is still missing: anticipation. 
% While human perception is  tightly coupled with
The literature on visual-inertial navigation is vast, with many contributions proposed over the last two years, 
including approaches based on filtering~\cite{Mourikis07icra,Kottas14icra-vins,Davison07pami,Blosch15iros,Jones11ijrr,Hesch14ijrr},
%Leutenegger13rss, 
fixed-lag smoothing~\cite{Mourikis08wvlmp,Sibley10jfr,DongSi11icra,Leutenegger15ijrr}, 
and full smoothing~\cite{Bryson09icra,Indelman13ras,Shen14thesis,keivan14iser,PatronPerez15,Lupton12tro,Forster16tro}.
%Forster15rss-imuPreintegration,
We refer the reader to~\cite{Forster16tro} for a comprehensive review.

%% file: preliminaries.tex
%!TEX root = main.tex

% \subsection*{Notation}
% \label{sec:preliminaries}

% \vspace{0.1cm}
\myparagraph{Notation}
We use lowercase and uppercase bold letters to denote vectors (e.g. $\vv$) and matrices (e.g. $\MM$), respectively.
Sets are denoted by  sans script fonts (e.g. $\setA$). %, $\setB$.  
Non-bold face letters are used for scalars and indices (e.g. $j$) and function names (e.g. $f(\cdot)$). 
The symbol $|\setA|$ denotes the cardinality of $\setA$.
The identity matrix of size $n$ is denoted with $\eye_n$. An $m \times n$ zero matrix is denoted by $\zero_{m\times n}$. 
 %The transpose of a vector or a matrix is represented by $\tran$ ($\MX\tran$).
 \editedTwo{%For a matrix $\MM$, 
 $\MM \succeq 0$ indicates that the matrix $\MM$ is positive semidefinite.}
 %; similarly, for matrices $\MM$ and $\MN$, the notation $\MM \succeq \MN$ means that $\MM - \MN$ is positive semidefinite.}
The symbol $\|\cdot\|$ denotes the Euclidean norm for vectors and the spectral norm for matrices. 

%% file: attentionVIN.tex
%!TEX root = main.tex

\section{Attention in Visual-Inertial Navigation}
\label{sec:attention}

%\LC{Considerations on notation: $\Deltaij$ is inconsistent, $k$ and $j$ might not be the best choice of time stamps}

We design an attention mechanism that selects $\kappa$ relevant visual features 
(e.g., Harris corners) from the 
current frame in order to maximize the performance  of visual-inertial motion estimation.
The $\kappa$ features have to be selected out of $N$ available features present in the camera image; 
the approach can deal with both monocular and stereo cameras (a stereo camera is treated as a rigid pair 
of monocular cameras).
% In this paper we consider a visual-inertial motion estimation task, hence maximizing the task performance 
% is equivalent to minimizing the estimation error in VIN.

%% file: featureSelection.tex
%!TEX root = main.tex

% \subsection{Task-driven Feature Selection}
% \label{sec:featureSelection}

We call $\setF$ the set of all available features (with $|\setF| = N$). 
If we denote with $f(\cdot)$ our task-dependent performance metric 
(we formalize a suitable  metric for VIN in \prettyref{sec:metrics}), 
we can state our feature selection problem as follows: 
\bea
\label{eq:featureSelection}
\max_{\setS \subset \setF} & f(\setS) \quad  \subject & |\setS| \leq \kappa
\eea
The problem looks for a subset of features $\setS$, containing no more than $\kappa$ features, 
which optimizes the task performance $f(\cdot)$. This is a standard feature selection problem and 
has been used across multiple fields, including machine learning~\cite{Das11icml-submodularity}, 
robotics~\cite{Mu15rss}, and sensor networks~\cite{Joshi09tsp-sensorSelection}.
Problem~\eqref{eq:featureSelection} is NP-hard~\cite{Bian06ipsn-sensorSelection} in general. 
In the rest of this paper we are interested in designing a suitable performance metric $f(\setS)$ 
for our VIN task, and provide fast approximation algorithms to solve~\eqref{eq:featureSelection}. 
% \bea
% \min_{\setS \subset \setF} & f(\setS) \quad  \subject & |\setS| \leq \kappa
% \eea

We would like to design a performance metric $f(\cdot)$ that  captures task-dependent requirements: in our 
case the metric has to quantify the uncertainty in the VIN motion estimation. 
Moreover, the metric should capture aspects already deemed relevant in 
related work. First, the metric has to reward the selection of the most distinctive features (in terms of appearance) 
since these are more likely to be re-observed in consecutive frames.
Second, the metric has to reward features that remain within the field of view for a longer time.
Therefore, \emph{anticipation} is a key aspect: the metric has to be aware 
that under certain motion some of the features are more likely to remain in the field of view of the camera.
Third, the metric has to reward features that are more informative to reduce uncertainty. 
In the following section we propose two performance metrics that 
seamlessly capture all these aspects. % without explicitly enforcing them.

%% file: performanceMetrics.tex
%!TEX root = main.tex

\subsection{Task-dependent Performance Metrics for VIN}
\label{sec:metrics}

%In this paper we consider a visual inertial motion estimation task.
% , hence our 
% performance metric quantifies the accuracy of the motion estimation.
Here we propose two metrics that quantify the accumulation of estimation errors over 
the horizon $\hor$, under the selection of a set of visual features $\setS$. 
Assume that $k$ is the time instant at which the features need to be  selected.
Let us call $\vxxest_k$ the (to-be-computed) state estimate of the robot at time $k$: we will be more precise 
about the variables included in $\vxxest_k$ in~\prettyref{sec:imuLinearModel}; for now the reader can 
think that $\vxxest_k$ contains the estimate for the pose and velocity of the robot at time $k$, as well as 
the IMU biases. We denote with $\vxxest_{k:k+\hor} \doteq [\vxxest_{k} \; \vxxest_{k+1} \; \ldots \; \vxxest_{k+\hor}]$ the future 
state estimates within the horizon $\hor$.
Moreover, we call $\MP_{\k:\k+\hor}$ the covariance matrix of our estimate $\vxxest_{k:k+\hor}$, 
and $\MOmega_{\k:\k+\hor} \doteq \MP_{\k:\k+\hor}\inv$ the corresponding information matrix.
Two natural metrics to capture the accuracy of $\vxxest_{k:k+\hor}$ are described in the following.

%%%%%%%%%%%%%%%%%%%%%%%%%%%%%%%%%%%%%%%%%%%%%%%%%%%%%%%%%%%%%%%%%%%%%%%%%%%%%%%%%%%%%%%%%%%%%%%%%%%%%%%%%%%%%%%%%%%%%%%%%%%%%
\myparagraph{Worst-case Estimation Error}
The worst-case error variance is quantified by the largest eigenvalue $\lambdaMax(\MP_{\k:\k+\hor})$ of the covariance matrix 
$\MP_{\k:\k+\hor}$, see e.g.,~\cite{Joshi09tsp-sensorSelection}.
 Calling $\lambdaMin(\MOmega_{\k:\k+\hor})$ the smallest eigenvalue of the information matrix $\MOmega_{\k:\k+\hor}$, 
if follows that $\lambdaMax(\MP_{\k:\k+\hor})  = 1/\lambdaMin(\MOmega_{\k:\k+\hor})$, hence minimizing the worst-case error is the 
same as maximizing $\lambdaMin(\MOmega_{\k:\k+\hor})$. Note that the information matrix $\MOmega_{\k:\k+\hor}$ is function of the 
selected set of measurements $\setS$, hence we write $\lambdaMin(\MOmega_{\k:\k+\hor}(\setS))$.
%; for the sake of simplicity, in the following we simply write $\lambdaMin(\MOmega_{\k:\k+\hor})$

Therefore our first metric (to be maximized) is:
\beq
\label{eq:eigMin}
\feigMin(\setS) = \lambdaMin(\MOmega_{\k:\k+\hor}(\setS)) = \lambdaMin\left(\barMOmega_{\k:\k+\hor} + \sum_{\l \in \setS} \Delta_\l \right)
\eeq
where on the right-hand-side, we exploited the additive structure of the information matrix, where $\barMOmega_{\k:\k+\hor}$ is the 
information matrix of the estimate when no features are selected (intuitively, this is the inverse of the covariance resulting from the IMU integration), 
while $\Delta_\l$ is the information matrix associated with the selection of the $\l$-th feature. We will give an explicit expression to $\barMOmega_{\k:\k+\hor}$ 
and $\Delta_\l$ in~\prettyref{sec:estimationModel}.

%%%%%%%%%%%%%%%%%%%%%%%%%%%%%%%%%%%%%%%%%%%%%%%%%%%%%%%%%%%%%%%%%%%%%%%%%%%%%%%%%%%%%%%%%%%%%%%%%%%%%%%%%%%%%%%%%%%%%%%%%%%%%
\myparagraph{Volume and Mean Radius of the Confidence Ellipsoid}
The $\varepsilon$-confidence ellipsoid is the ellipsoid that contains the estimation error with 
probability $\varepsilon$. Both the volume and the mean radius of the $\varepsilon$-confidence ellipsoid  
are proportional to the determinant of the covariance matrix. In particular, the volume $V$ 
and the mean radius $\bar{R}$ of an $n$-dimensional ellipsoid associated with the covariance 
$\MP_{\k:\k+\hor}$
 can be written as~\cite{Joshi09tsp-sensorSelection}:
\bea
\label{eq:volumeAndRadius}
V = \frac{(\alpha \pi)^{n/2}}{\Gamma(\frac{n}{2}+1)} \det(\MP_{\k:\k+\hor}^\half) \;, \quad
\bar{R} = \sqrt{\alpha} \det(\MP_{\k:\k+\hor})^{\frac{1}{2n}}
\eea
where $\alpha$ is the quantile of the $\chi^2$ distribution with $n$ degrees of freedom and upper tail 
probability of $\varepsilon$, $\Gamma(\cdot)$ is the Gamma function, and $\det(\cdot)$ is the determinant of a square matrix.

From~\eqref{eq:volumeAndRadius} we note that to minimize the volume and the mean radius of the 
confidence ellipsoid we can equivalently minimize the determinant of the covariance.
Moreover, since
\bea
\logdet(\MP_{\k:\k+\hor}) = \logdet(\MOmega_{\k:\k+\hor}\inv) = - \logdet(\MOmega_{\k:\k+\hor}) \nonumber
\eea
then minimizing the size of the confidence ellipsoid is the same as maximizing
the log-determinant of the information matrix, leading to our second performance metric:
\beq
\label{eq:logDet}
\flogDet(\setS) = \logdet(\MOmega_{\k:\k+\hor}(\setS)) = \logdet\left(\barMOmega_{\k:\k+\hor} + \sum_{\l \in \setS} \Delta_\l \right)
\eeq
where we again noted that the information matrix is function of the selected features and can 
be written in additive form.

%%%%%%%%%%%%%%%%%%%%%%%%%%%%%%%%%%%%%%%%%%%%%%%%%%%%%%%%%%%%%%%%%%%%%%%%%%%%%%%%%%%%%%%%%%%%%%%%%%%%%%%%%%%%%%%%%%%%%%%%%%%%%
\myparagraph{Probabilistic Feature Tracks}
The performance metrics described so far already capture some important aspects: they are task-dependent in that they 
both quantify the motion estimation performance; moreover, they are predictive, in the sense that they look at the result 
of selecting a set of features over a short (future) horizon. As we will see in~\prettyref{sec:visionLinearModel}, 
the model also captures the fact that longer feature tracks are more informative,
%: intuitively, longer feature tracks 
%will have ``larger'' $\Delta_i$ in~\eqref{eq:eigMin} 
%and~\eqref{eq:logDet} \extra{(a formal proof is given later)};
therefore it implicitly rewards the selection of features 
that are co-visible across multiple frames. 

The only aspect that is not yet modeled is the fact that, even when a feature is in the field of 
view of the camera, there is some chance that it will not be correctly tracked
%re-detected
 and the corresponding feature track will be lost. For instance, 
if the appearance of a feature is not distinctive enough, the feature track may %break prematurely.
 be shorter than expected.

To model the probability that a feature track is lost, 
we introduce $N$ Bernoulli random variables $b_1, \ldots, b_N$. 
Each variable $b_\l$ represents the outcome of the tracking of feature $\l$: 
if $b_\l=1$, then the feature is successfully tracked, otherwise, the feature 
track is lost. For each feature we assume $p_\l = \prob{b_\l = 1}$ to be given. \edited{In practice 
one can correlate the appearance of each feature to $p_\l$, such that 
more distinctive features have higher probability of being tracked, or can learn the probabilities from data.} 
% if they are in the field of view\footnote{\edited{We remark that }}. 
Using the binary variables $\vb \doteq \{b_1,\ldots,b_N\}$, we write the information matrix at the 
end of the horizon as:
\beq
\label{eq:OmegaWithb}
\textstyle \MOmega_{\k:\k+\hor}(\setS,\vb) = \barMOmega_{\k:\k+\hor} + \sum_{\l \in \setS} b_\l \Delta_\l
\eeq
which has a clear interpretation: if the $\l$-th feature  is correctly tracked, then $b_\l = 1$ 
and the corresponding information matrix $\Delta_\l$ is added to  $\barMOmega_{\k:\k+\hor}$; on the other hand, if the feature tracks is lost, then $b_\l = 0$ and
the corresponding information content simply disappears from the sum in~\eqref{eq:OmegaWithb}.

Since $\vb$ is a random vector, our information matrix is now a stochastic 
quantity $\MOmega_{\k:\k+\hor}(\setS,\vb)$, hence we have to redefine our 
performance metrics to include the expectation over $\vb$:
\beq
\label{eq:expectation}
f(\setS) = \E { f(\MOmega_{\k:\k+\hor}(\setS,\vb)) }
\eeq
where the function
%symbol $f$ to denote either $\feigMin$ or $\flogDet$.
$f(\cdot)$ denotes either  $\feigMin(\cdot)$ or  $\flogDet(\cdot)$.

Computing the expectation~\eqref{eq:expectation}  leads
to a sum with a combinatorial number of terms, which is hard to 
even evaluate. 
To avoid the combinatorial 
explosion, we use Jensen's inequality, \edited{which produces an upper bound on the expectation of a 
\emph{concave} function as follows}:
\beq
\label{eq:JensensInequality}
\E { f(\MOmega_{\k:\k+\hor}(\setS,\vb)) } \leq  f(\E {\MOmega_{\k:\k+\hor}(\setS,\vb)} )
\eeq
\edited{Since both $\feigMin(\cdot)$ and $\flogDet(\cdot)$ are concave functions, 
the Jensen's inequality produces an upper bound for our expected cost}. In the 
rest of this paper we maximize this bound, rather that the original 
cost. 
The advantage of doing so is that the 
right-hand-side of~\eqref{eq:JensensInequality} can be efficiently 
computed as:
\bea
\label{eq:JensensInequalityRHS}
\textstyle f(\E {\MOmega_{\k:\k+\hor}(\setS,\vb)} = f( \barMOmega_{\k:\k+\hor} + \sum_{\l \in \setS} p_\l \Delta_\l )
\eea 
where we used the definition~\eqref{eq:OmegaWithb}, the fact
that the expectation is a linear operator, and that $\E{b_\l} = p_\l$.
Therefore, our performance metrics can be written explicitly as:
\bea
\label{eq:probabilisticMetrics}
%\textstyle 
\feigMin(\setS) = \lambdaMin\left(\barMOmega_{\k:\k+\hor} + \sum_{\l \in \setS} p_\l \Delta_\l \right)
\\
%\textstyle 
\flogDet(\setS) = \logdet\left(\barMOmega_{\k:\k+\hor} + \sum_{\l \in \setS} p_\l \Delta_\l \right) \nonumber
\eea
 which coincide with the deterministic counterparts~\eqref{eq:eigMin},~\eqref{eq:logDet}
 when $p_\l = 1, \forall \l$. Interestingly, in~\eqref{eq:probabilisticMetrics} 
 the probability that a feature is not tracked 
 simply discounts the corresponding information content. Therefore, the approach 
 considers features that are more likely to get lost as less informative, which 
 is a desired behavior. 
\edited{We remark that the probabilities $p_\l$ only capture the distinctiveness of the visual features, while the geometric aspects 
(e.g., co-visibility) are captured in the matrices $\Delta_\l$, as described in~\prettyref{sec:visionLinearModel}.}
 While the derivation so far is quite general and provides a feature selection mechanism for 
 any feature-based SLAM system, in the following we focus on visual-inertial navigation and 
 we provide explicit expressions for the matrices $\barMOmega_{\k:\k+\hor}$ and $\Delta_\l$ appearing in eq.~\eqref{eq:probabilisticMetrics}.
% We denote with $p_i$ the probability that a feature is re-detected (or tracked) over multiple frames
% with some probability feature is lost:
% Therefore we would like to 
% optimize 
% $\E_{f_1,f_2,...} [f(x)]$ in which $f_1 = 1$ if feature is tracked, 
% or $0$ if it is not. Expectation has a combinatorial number of terms, therefore, 
% we use Jensen inequality and rather maximize $f(\E(x))$, which amounts to ``reducing'' 
% the information content of the measurements based on the corresponding probability of tracking.
% The probability of tracking is proportional to feature quality?

%% file: forwardSimulationModel.tex
%!TEX root = main.tex

\subsection{Forward-simulation Model}
\label{sec:estimationModel}

%\LC{Considerations on notation: $\Deltaij$ is inconsistent, $k$ and $j$ might not be the best choice of time stamps}

The feature selection model proposed in~\prettyref{sec:attention} and the metrics in~\prettyref{sec:metrics} require to predict the evolution of the information matrix over the horizon $\hor$.
% The evolution of the information matrix depends on the future measurements, which we 
% assume to be given by on-board IMU and camera. 
In the following we show how to forward-simulate 
the IMU and the camera; we note that we do not 
require to simulate actual IMU measurements, but only need to predict the corresponding 
information matrix, which depends on the IMU noise statistics.
%use knowledge about the IMU noise statistics.

The forward-simulation model depends on the future motion of the robot 
(the IMU and vision models are function of the future poses of the robot); therefore, anticipation is a 
key element of our approach: the feature selection mechanism is aware of the immediate-future 
intentions of the robot and selects features accordingly. As we will see in the experiments, 
this enables a more clever selection of features during sharp turns and aggressive maneuvers.
In practice, the future poses along the horizon can be computed from the current control actions; 
 for instance, if the controller is planning over a receding horizon, 
 one can get the future poses by integrating the dynamics of the vehicle. 
 In this sense, our attention mechanism involves a tight integration of control and perception.

The algorithms for feature selection that we present in~\prettyref{sec:algorithmsAndGuarantees} 
are generic and work for any positive definite $\barMOmega_{\k:\k+\hor}$ 
and any positive semidefinite $\Delta_\l$. Therefore, the non-interested reader can safely skip this section, 
which provides explicit expressions for $\barMOmega_{\k:\k+\hor}$ and $\Delta_\l$ in the 
visual-inertial setup. %, thus loosing some of the geometric insights later

Before delving into the details of the IMU and vision model we remark a key design goal of our 
%design of the 
forward-simulation model: efficiency. The goal of an attention mechanism is to reduce the 
cognitive load later on in the processing pipeline; therefore, by design, it should not be computational demanding, 
as that would defeat its purpose. % of reducing the processing load. 
For this reasons, in this section we present 
a simplified VIN model which is designed to be efficient to compute, while capturing all the aspects of interest 
of a full visual-inertial estimation pipeline, e.g.,~\cite{Forster16tro}. 

%% file: imuLinearModel.tex
%!TEX root = main.tex

\subsubsection{IMU Model \edited{and Priors}}
\label{sec:imuLinearModel}

Our simplified IMU model is based on a single assumption: the accumulation of the rotation error due to gyroscope 
integration over the time horizon is negligible. In other words, the relative rotation estimates predicted by the 
gyroscope are accurate. This assumption is realistic, even for inexpensive IMUs: 
%in the experimental 
%section we show that rotation errors from gyroscope  
the drift in rotation integration is typically small  
and  negligible over the time horizon considered in our attention system (in our tests we 
consider a time horizon of $3\sec$).
% when integrating over tens of seconds, 
%which is way beyond what we need here (a typical horizon in our attention system is $\hor = 3$s). 

Assuming that the rotations are accurately known allows restricting the state to the robot position, linear velocity, 
and the accelerometer bias. Therefore, in the rest of this paper, the (unknown) state of the robot at time $\k$ is 
  $\vxx_\k \doteq [\pos{\k} \; \vel{\k} \; \biasAcc{\k} ]$, where $\pos{\k} \in \Real{3}$ is the 3D position of the 
  robot, $\vel{\k} \in \Real{3}$ is its velocity, and $\biasAcc{\k}$ is the (time-varying) accelerometer bias.
  We also use the symbol $\Rot{\k}$ to denote the attitude of the robot at time $\k$: this is assumed to be known from 
  gyroscope integration over the horizon $\hor$, hence it is not part of the state.

As in most VIN pipelines, we want to estimate the state of the robot at each frame\footnote{
The derivation is identical for the case in which 
 we associate a state to each keyframe, rather than each frame, as done in related work~\cite{Forster16tro}.}.
Therefore, the goal of this subsection, 
  similarly to~\cite{Forster16tro}, is to reformulate a set of IMU measurements between two consecutive frames $\k$ and $\j$ as 
  a single measurement that constrains $\vxx_\k$ and $\vxx_\j$. Differently from~\cite{Forster16tro}, we show how to 
  get a 
  \emph{linear} measurement model. 
  %To do so, we start by recalling the measurement model of an IMU.

The on-board accelerometer measures 
 the acceleration $\acc{\k}$ of the sensor with respect to an inertial frame, and is affected by  
  additive white noise $\noiseAcc{\k}$ and a slowly varying sensor bias $\biasAcc{\k}$. Therefore, the 
  measurement $\measAcc{\k} \in \Real{3}$ acquired by the accelorometer at time $\k$ is modeled as~\cite{Forster16tro}:
\bea
\measAcc{\k} &=& \Rot{\k}\tran \left( \acc{\k} - \gravity \right) + \biasAcc{\k} + \noiseAcc{\k}, \label{eq:accModel}
\eea
where $\gravity$ is the gravity vector, expressed in the inertial frame. 
To keep notation simple, we omit the reference frames in our notation, which follow closely the convention used in~\cite{Forster16tro}:
 position $\pos{\k}$ and velocity $\vel{\k}$ are expressed in the \edited{(inertial) world} frame,\footnote{\edited{As usual when adopting MEMS inertial sensors, we assume that a local reference frame on Earth (our ``world'' frame) can be approximated as inertial, since 
 the effects of the Earth rotation are negligible with respect to the measurement noise. For a more comprehensive discussion on  reference frames for inertial navigation we refer the reader to~\cite[Chapter 2.2]{Farrell08book}.}} while the bias $\biasAcc{\k}$ is expressed in the sensor frame.

Given position $\pos{\k}$ and velocity $\vel{\k}$ at time $\k$, we can forward-integrate and obtain 
 $\pos{\j}$ and $\vel{\j}$ at time $j > \k$:
\bea
\label{eq:integrationKeyframes}
 \vel{j} &=& \textstyle  \vel{\k} + \sum_{\ind=\k}^{\j-1} \acc{\ind}  \Deltak \nonumber
 \\
 && \grayText{(from~\eqref{eq:accModel} we know 
 $\acc{\ind} = \gravity  + \Rot{\ind}(\measAcc{\ind} - \biasAcc{\ind} - \noiseAcc{\ind})$,}
 \nonumber \\
 && \grayText{and assuming constant bias between frames, $\biasAcc{\ind} = \biasAcc{\k}$)}
 \nonumber \\
&=&
 \textstyle \vel{\k} + \gravity\Deltaij + \sum_{\ind=\k}^{\j-1} \edited{\Rot{\ind}} \left( \measAcc{\ind}\!-\! \biasAcc{\k} \!-\! \noiseAcc{\ind} \right) \Deltak  \label{eq:velEq}
\eea
%%%%%%%
\bea
% POSITION
\pos{j} &=& \textstyle \pos{\k} + 
\sum_{\ind=\k}^{\j-1} \left( \vel{\ind} \Deltak + \frac{1}{2}  
\acc{\ind} \Deltak ^2 \right) \nonumber
\\
&& \grayText{(substituting $\acc{\ind} = \gravity  + \Rot{\ind}(\measAcc{\ind} - \biasAcc{\k} - \noiseAcc{\ind})$)}
\nonumber \\
 &=&
\textstyle \pos{\k} + \sum_{\ind=\k}^{\j-1} 
( \vel{\ind} \Deltak + \frac{1}{2}\gravity\Deltak^2 + 
%\nonumber \\  && 
\frac{1}{2} \Rot{\k} \left( \measAcc{\ind}\!-\! \biasAcc{\k} \!-\! \noiseAcc{\ind} \right) \Deltak^2 ) 
\nonumber 
\\
&& \grayText{(substituting $\vel{j}$ from~\eqref{eq:velEq} with $\j = \ind$)}
\nonumber \\  
&=& \textstyle \pos{\k} + \frac{1}{2}\gravity\DeltaijSq + 
\sum_{\ind=\k}^{\j-1} 
\frac{1}{2} \Rot{\ind} \left( \measAcc{\ind}\!-\! \biasAcc{\k} \!-\! \noiseAcc{\ind} \right) \Deltak^2
\nonumber \\  
&+&\!\! \textstyle \sum_{\ind=\k}^{\j-1} 
( \vel{\k} \!+\! \gravity\Deltak_{\k\ind} \!+\! \sum_{\indd=\k}^{\ind-1} 
\Rot{\indd} \left( \measAcc{\indd}\!-\! \biasAcc{\k} \!-\! \noiseAcc{\indd} \right) \Deltak ) 
\label{eq:posEq}
\eea
where $\Deltak$ is the sampling time of the IMU,  
$\Deltaij \doteq  \sum_{\ind=\k}^{\j-1} \Deltak$, 
and $\DeltaijSq \doteq  \sum_{\ind=\k}^{\j-1} \Deltak^2$;
as in~\cite{Forster16tro}, we assumed that the IMU bias remains constant between two frames.
The evolution of the bias across frames can be modeled as a random walk:
\bea
\label{eq:biasEvolution}
\biasAcc{\j} &=& \textstyle \biasAcc{\k} - \noise^\B_{\k\j}
\eea
where $\noise^\B_{\k\j}$ is a zero-mean random vector.

Noting that the state appears linearly in~\eqref{eq:velEq}-\eqref{eq:biasEvolution}, 
it is easy to rewrite the three expressions together in matrix form:
\bea
\label{eq:imuModel}
\measImu{\k\j} = \MA_{{\k\j}} \vxx_{\k:\k+\hor} + \noiseImu{\k\j}
\eea
where $\measImu{\k\j}  \in \Real{9}$ is a suitable vector,\footnote{\edited{The expression of $\measImu{\k\j}$ 
is inconsequential for the 
following derivation, but the interested reader can find details and derivations
% the derivation and the explicit form of all involved vectors and 
 in Appendix~\ref{sec:linearImuModelAppendix}.}} and
 \edited{$\noiseImu{{\k\j}} \in \Real{9}$} is zero-mean random noise.
 We remark that while $\measImu{\k\j}$ is function of the future IMU measurements, 
 this  vector is not actually used in our approach (what matters is $\MA_{{\k\j}}$ 
 and the information matrix of $\noiseImu{\k\j}$), hence we do not need to simulate future measurements.
An explicit expression for the matrix $\MA_{{\k\j}} \in \Real{9 \times 9(\hor+1)}$, 
the vector $\measImu{\k\j}$, and the covariance of $\noiseImu{\k\j}$ 
is given in Appendix~\ref{sec:linearImuModelAppendix}.
The matrix $\MA_{{\k\j}}$ is a sparse block matrix with
$9 \times 9$ blocks, which is all zeros, except the blocks corresponding to 
the state at times $\k$ and  $\j$. 
%
% \bea
% \label{eq:definitionsV}
% \deltaVmeas\ij = \gravity\Deltaij + \sum_{\k=i}^{j-1} \R(\k) \tilde\acc(\disctime) \Deltak 
% \nonumber \\
% \noiseAccd\ij \doteq [\noiseAccd(i)\tran \; \noiseAccd(i+1)\tran \; \ldots \; \noiseAccd(j-1)\tran]\tran
% \nonumber \\
% J^v\ij = [\R(1)\tran \Delta_1 \;\; \R(2)\tran \Delta_2 \;\; \ldots \;\; \R(j-1)\tran \Delta_{j-1}]\tran
% \nonumber \\
% \Theta^v\ij = (\sum_{\k=i}^{j-1} \R(\k) \Deltak)
% \eea

\edited{
From  linear estimation theory, we know that, using the IMU measurements~\eqref{eq:imuModel} 
for all consecutive frames $\k,\j$ in the horizon $\hor$,
 the information matrix of the optimal estimate of the state $\vxx_{\k:\k+\hor}$ given the IMU data is:
\bea
\label{eq:barMOmegaImuNoPriors}
\barMOmega_{\k:\k+\hor}^\imu = \sum_{\k\j \in \setH} (\MA_{{\k\j}}\tran  \MOmegaImu{\k\j}   \MA_{{\k\j}})
\eea
where $\setH$ is the set of consecutive frames within the time horizon $\hor$, and 
$\MOmegaImu{\k\j} \in \Real{9\times 9}$ is the information matrix of the noise vector $\noiseImu{\k\j}$ introduced in eq.~\eqref{eq:imuModel}.

While the IMU measurements constrain the states in the future horizon $\hor$, the predicted information matrix at time $\k+\hor$ is also influenced by the initial information matrix at time $\k$.
This information matrix, referred to as $\barMOmega_{\k}^\prior \in \Real{9\times 9}$, is computed and maintained by 
the VIN estimator,\footnote{
\edited{In our implementation, the VIN estimation is a fixed-lag smoother based on iSAM2 (described in 
Section~\ref{sec:experiments}) and the information matrix at time $\k$ is obtained by marginalizing states in the smoother other than $\vxx_{\k}$}.} and can be understood as a prior on the state at time $\k$. The presence of this information matrix results in an additional term in the expression of $\barMOmega_{\k:\k+\hor}$, which is the information matrix of the state estimate before any vision measurement is 
selected, as per eq.~\eqref{eq:OmegaWithb}.
In particular, $\barMOmega_{\k:\k+\hor}$ can be written as:
\bea
\label{eq:barMOmegaImu}
\barMOmega_{\k:\k+\hor} = \barMOmega_{\k:\k+\hor}^\imu + 
\barMOmega_{\k:\k+\hor}^\prior
\eea
where $\barMOmega_{\k:\k+\hor}^\prior$ is a matrix which is zero everywhere except the top-left $9\times 9$ block which is equal 
to $\barMOmega_{\k}^\prior$, to reflect our prior on the state at time $\k$.
The matrix $\barMOmega_{\k:\k+\hor}$ represents the information matrix of an optimal estimator of the state $\vxx_{\k:\k+\hor}$ before selecting visual measurements and follows the notation we have already introduced in~\eqref{eq:OmegaWithb}.
The presence of priors and measurements proceeds in full analogy with a standard fixed-lag smoother, while here we have the advantage of working with a linear model. 
}

% Then the expression of $\V(j)$ becomes:
% %
% \bea
% \label{eq:linearMeasurementsV}
%  \deltaVmeas\ij &=& \V(j) - \V(i) + \Theta^v\ij \bias^a(i)  + J^v\ij \noiseAccd\ij
% \eea
% %
% Similarly, we can develop the expression of $\P(j)$ as follows:
% %
% \bea
% \label{eq:definitionsP}
% \deltaPmeas\ij = \sum_{\k=i}^{j-1} \Deltak \deltaVmeas\ik + \half ( \gravity + \R(\k) \tilde\acc(\disctime) ) \Deltak^2 
% \nonumber \\
% J^p\ij = [\R(1)\tran \Delta_1 \;\; \R(2)\tran \Delta_2 \;\; \ldots \;\; \R(j-1)\tran \Delta_{j-1}]\tran
% \nonumber \\
% J^p\ij = \sum_{\k=i}^{j-1}  J^v\ik \Deltak + \half[\R(1)\tran \Delta_1^2 \;\; \ldots \;\; \R(j-1)\tran \Delta_{j-1}^2]\tran
% \nonumber \\
% \Theta^v\ij = \sum_{\k=i}^{j-1}  \Theta^v\ik \Deltak + \R(\k) \Deltak^2
% \eea
% %
% Then the expression of $\P(j)$ becomes:
% %
% \bea
% \label{eq:linearMeasurementsP}
%  \deltaPmeas\ij &=& \P(j) - \P(i) - \V(i) \Deltaij + \Theta^p\ij \bias^a(i)  + J^p\ij \noiseAccd\ij
% \eea
% %

%% file: visionLinearModel.tex
%!TEX root = main.tex

\subsubsection{Vision Model}
\label{sec:visionLinearModel}

Also for the vision measurements, we are interested in designing a 
linear measurement model, which simplifies the actual (nonlinear) 
perspective projection model. 
To do so, we have to express a pixel measurement as a linear function of the 
unknown state that we want to estimate.

%For this purpose we use an \emph{algebraic error model}, rather
A (calibrated) pixel measurement of an external 3D point (or landmark) $\l$
identifies the 3D bearing of the landmark in the camera frame.
%  identifies a 3D direction, expressed 
% in the camera frame.
Mathematically, 
if we call $\ucl$ the unit vector corresponding to the (calibrated) pixel 
observation of $\l$ from the robot pose 
at time $\c$, $\ucl$ satisfies the following relation:
%our camera projection model is as follows:
%
\bea
\label{eq:visionModel0}
\ucl  \times \left( ( \MR^\world_{\cam,\c})\tran (\p_\l - \vt^\world_{\cam,\c} ) \right) = \zero_3
\eea 
where $\times$ is the cross product between two vectors,
$\p_\l$ is the 3D position of landmark $\l$ (in the world frame), 
 $\MR^\world_{\cam,\c}$ and $\vt^\world_{\cam,\c}$ are the rotation and translation describing the camera pose at time $\c$ 
 (w.r.t. the world frame).
In words, the model~\eqref{eq:visionModel0} requires the observed point (transformed to the camera frame) 
to be collinear to the measured direction $\ucl$, since the cross product measures the deviation 
from collinearity~\cite{Tron15acc-rigidity}. 
%In the following, we use~\eqref{eq:visionModel0} as our measurement 
%model, observing that it is linear in the unknown 

Now we note that
 for two vectors $\vv_1$ and $\vv_2$, the cross product $\vv_1 \times \vv_2 = [ \vv_1 ]_\times \vv_2$, 
 where $[ \vv_1 ]_\times$ is the skew symmetric matrix built from $\vv_1$. Moreover, we 
 note that the camera pose w.r.t. the world frame, $(\MR^\world_{\cam,\c},\vt^\world_{\cam,\c})$, 
 can be written as the composition of the IMU pose w.r.t. the world frame, $(\Rot{\c},\pos{\c})$, 
 and the relative pose of the camera w.r.t. the IMU, $(\MR^\imu_{\cam},\vt^\imu_{\cam})$ (known from 
 calibration).
 Using these two considerations, we rewrite~\eqref{eq:visionModel0} equivalently as:
\bea
\label{eq:visionModel1}
[ \ucl ]_\times \left( ( \Rot{\c} \MR^\imu_{\cam} )\tran (\p_\l - ( \pos{\c} + \Rot{\c} \vt^\imu_{\cam} ) ) \right) = \zero_3
\eea 
In presence of measurement noise,~\eqref{eq:visionModel1} becomes:
\bea
\label{eq:visionModel2}
[ \ucl ]_\times \left( ( \Rot{\c} \MR^\imu_{\cam} )\tran (\p_\l - ( \pos{\c} + \Rot{\c} \vt^\imu_{\cam} ) ) \right) = \noiseu{\c\l}
\eea 
where $\noiseu{\c\l}$ is a zero-mean random noise with known covariance.
Under the assumptions that rotations are known from gyroscope integration, 
the unknowns in model~\eqref{eq:visionModel2} are the robot position $\pos{\c}$ (which is part of our 
state vector $\vxx_{\k:\k+\hor}$) and 
the position of the observed 3D landmark $\p_\l$.
\edited{Rearranging the terms we obtain:
\bea
% \label{eq:visionModel222}
[ \ucl ]_\times  ( \MR^\imu_{\cam} )\tran \vt^\imu_{\cam} = 
[ \ucl ]_\times ( \Rot{\c} \MR^\imu_{\cam} )\tran  ( \pos{\c} -  \p_\l) \nonumber
+
\noiseu{\c\l}
\eea 
The only unknowns in the previous equation 
%in~\eqref{eq:visionModel222} 
are $\pos{\c}$ and $\p_\l$, hence the model is linear in the state
 and can be written in matrix form:}
\bea
\label{eq:visionModel_single}
\measU{\c\l} = \MF\cl \vxx_{\k:\k+\hor} + \ME\cl \p_\l + \noiseu{\c\l}
\eea
\edited{with $\measU{\cl} \doteq [ \ucl ]_\times  ( \MR^\imu_{\cam} )\tran \vt^\imu_{\cam}$, and for suitable matrices 
$\MF\cl \in \Real{3\times 9(\hor+1)}$ and $\ME\cl \in \Real{3\times 3}$.}  
In order to be triangulated, a point has to be observed across multiple frames. 
Stacking the linear system~\eqref{eq:visionModel_single} 
for each observation pose from which $\l$ is visible, we get a single linear system:
\bea
\label{eq:visionModel_point}
\measU{\l} = \MF_\l \vxx_{\k:\k+\hor} + \ME_\l \p_\l + \noiseu{\l}
\eea
\edited{where $\measU{\l} \in \Real{3\nrVisible}$, $\MF_\l \in \Real{3\nrVisible \times 9(\hor+1)}$, and $\ME_\l \in \Real{3\nrVisible \times 3}$ are obtained by stacking (row-wise) 
 $\measU{\c\l}$, $\MF\cl$, and $\ME\cl$, respectively, for the $\nrVisible$ frames in which $\l$ is visible.}
 % for all frames $\k:\k+\hor$.}
 As for the IMU model, the expression of $\measU{\l}$ is inconsequential for our derivation, 
 as it does not influence the future state  covariance.
 On the other hand $\MF_\l$ and $\ME_\l$ depend on the future measurements $\ucl$: 
 for this reason, computing these matrices requires simulating pixel projections 
 of $\p_\l$ for each frame in the horizon. When using a stereo camera, we have an 
 estimate of $\p_\l$ hence we can easily project it to the future frames.
  In a monocular setup, we can guess the depth of new features from the existing features 
  in the VIN back-end. 
\edited{We remark that when simulating pixel projections at future frames in the horizon, 
we also perform a visibility check (i.e., we use the camera projection model to assert whether a landmark projects within the image frame), hence restricting the set of future visual measurements according to the 
expected motion of the vehicle and the camera field of view. This aspect is crucial in making our feature selection approach ``predictive'' in nature. 
As discussed in Proposition~\ref{prop:longerFeatureTracks} below, the resulting model seamlessly captures the intuitive fact that landmarks that remain visible and 
can be tracked across multiple frames are more informative.
}
  % assume an expected depth for the scene to get an estimate 
  % of $\p_\l$ from the pixel measurement in the current frame.

 Now we note that we cannot directly use the linear model~\eqref{eq:visionModel_point} 
 to estimate our state vector $\vxx_{\k:\k+\hor}$, since it contains the unknown 
position of landmark $\l$. One way to circumvent this problem is to include the 3D 
point in the state vector. This is undesirable for two reasons; first, including the landmarks 
as part of the state would largely increase the 
dimension of the state space (and hence of the matrices in~\eqref{eq:probabilisticMetrics}). Second, it may create undesirable behaviors of our performance metrics;  
for instance, the metrics might induce to select features that minimize the uncertainty of a far 3D point 
rather than focusing on the variables we are actually interested in (i.e., the state of the robot).

To avoid this undesirable effects, we analytically eliminate the 3D point from the estimation 
using the Schur complement trick~\cite{Carlone14icra-smartFactors}.
We first write the information matrix of the joint state $[\vxx_{\k:\k+\hor} \; \p_\l]$ 
from the linear measurements~\eqref{eq:visionModel_point}:
\bea
\label{eq:infoMatrixPoint}
\MOmega^{(\l)}_{\k:\k+\hor} = 
\matTwo{
\MF_\l\tran \MF_\l   & \MF_\l\tran 	\ME_\l \\
\ME_\l\tran \MF_\l    &  \ME_\l\tran \ME_\l
} \edited{\in \Real{(9\hor+12) \times (9\hor+12)}}
\eea
\edited{where, for simplicity, we assumed the pixel measurement noise to be the identity matrix}.
Using the Schur complement trick we marginalize out the landmark $\l$ 
and obtain the information matrix of our state $\vxx_{\k:\k+\hor}$ given the measurements~\eqref{eq:visionModel_point}:
\bea
\label{eq:infoMatrixVision}
\Delta_\l = 
\MF_\l\tran \MF_\l - \MF_\l\tran \ME_\l (\ME_\l\tran \ME_\l)\inv  \ME_\l\tran \MF_\l  \;\; \edited{\in \Real{9(\hor+1) \times 9(\hor+1)}}
\eea
Eq.~\eqref{eq:infoMatrixVision} is the (additive) contribution to the information matrix 
of our state estimate due to the measurements of a single landmark 
$\l$. This is 
the matrix that we already called $\Delta_\l$ in~\eqref{eq:OmegaWithb}.
The matrix $\Delta_\l$ % $\MF_\l\tran \MF_\l - \MF_\l\tran \ME_\l (\ME_\l\tran \ME_\l)\inv  \ME_\l\tran \MF_\l$ 
is sparse, and its sparsity pattern is dictated by the co-visibility of landmark $\l$ 
across different frames~\cite{Carlone14bmvc-miningSfM}.
 It is worth noticing that $(\ME_\l\tran \ME_\l)\inv$ is the covariance of the estimate of the landmark 
 position~\cite{Carlone14bmvc-miningSfM}, and it is invertible as long as the landmark $\l$ can be triangulated.
\edited{In our implementation, we only consider landmarks that can be triangulated (for which $(\ME_\l\tran \ME_\l)\inv$ invertible) 
as candidates for feature selection.}

 \edited{The following proposition formally proves that longer feature tracks lead to ``larger'' $\Delta_\l$ 
 hence the formulations~\eqref{eq:probabilisticMetrics} will encourage the selection of features with long tracks.

 \begin{proposition}[Long Feature Tracks]
\label{prop:longerFeatureTracks}
Consider two landmarks $\l_1$ and $\l_2$ having identical predicted future pixel measurements till time $k_1$
(i.e., $\measU{\tau\l_1} = \measU{\tau\l_2}$)
for $\tau = k,\ldots,k_1$, but such that $l_1$ is only tracked till $k_1$, while $l_2$ is tracked 
till a later frame $k_2$ (with $k < k_1 < k_2 \leq \hor$). Then $\Delta_{\l_1} \preceq \Delta_{\l_2}$.
\end{proposition}

The proof of \prettyref{prop:longerFeatureTracks} is given in Appendix~\ref{app:proof:prop:longerFeatureTracks}.
\prettyref{prop:longerFeatureTracks} ensures that features with long tracks carry a ``larger'' information content $\Delta_\l$  
hence contributing more to the objectives~\eqref{eq:probabilisticMetrics} (recall that for both choices of performance metrics 
it holds $f\left(\barMOmega_{\k:\k+\hor} + \Delta_{\l_1} \right) \leq f\left(\barMOmega_{\k:\k+\hor} + \Delta_{\l_2} \right)$ whenever $\Delta_{\l_1} \preceq \Delta_{\l_2}$). Note, however, that \prettyref{prop:longerFeatureTracks} ensures that
 long feature tracks ``dominate'' short tracks only when their future measurements are identical. 
 Therefore, in general, a heuristic selection based on feature track length may provide a suboptimal solution for~\eqref{eq:featureSelection}
(intuitively, such choice would not account for the \emph{geometry} of the features). 
%the result in \prettyref{prop:longerFeatureTracks} 
}

 %In practice, when computing $\ME_\l\tran \ME_\l$ we check that it is full rank, otherwise 
 %we set $\MOmega_{\k:\k+\hor}\at{\l}$ to zero (points that cannot be triangulated do not carry 
 % information about the robot position).
 %Therefore,~\eqref{eq:infoMatrixVision} 
%  also provides a simple way to filter out 
%  measurements that do not produce valid point estimates. 
% We close this section with a remark about the importance of designing 
% simplified linear models for our attention system.
% \LC{add remark on point covariance, invertibility, desirable feature of detecting early features that can
% be triangulated}
%has a clear sparsity structure: the nonzero blocks in the matrix corresponds to the 
%positions of the camera at the timestamp at which the landmark $\l$ is observed.
% Note that 
% , obtaining:
% %
% \bea
% \label{eq:visionModel_point}
% \tildemeasU{\l} = \MQ_\l \vxx_{\k:\k+\hor} + \MQ_\l \noiseu{\l}
% \eea
% where $\MQ_\l \doteq (\eye - \ME_\l ( \ME_\l\tran  \ME_\l)\inv  \ME_\l\tran) $, 
% $\tildemeasU{\l} = \MQ_\l\measU{\l}$. The matrix $\MQ_\l$ is an orthogonal projector 
% onto the null space of the matrix $\ME_\l$, hence $\MQ_\l \ME_\l = \zero$ and 
% $\MQ_\l \MQ_\l = \MQ_\l$ (idempotent matrix). 

%% file: remarkLinear.tex
%!TEX root = main.tex

\begin{remark}[Linear measurement models]
 Sections~\ref{sec:imuLinearModel} and~\ref{sec:visionLinearModel} 
 provide linear measurement models for inertial and visual measurements. 
 \edited{In particular, we assumed that rotations are known over a short time horizon and this allowed 
 us to obtain measurement models that are linear with respect to the unknown robot state.}
Within our framework, one might directly use linearized models of the 
nonlinear inertial and perspective models commonly used in VIN~\cite{Forster16tro}.
Our choice to design linear models has three motivations. First, we 
operate over a smaller state space (which does not include 
rotations and gyroscope biases), hence making matrix manipulations faster.
Second, we avoid the actual computational cost of linearizing the 
nonlinear models. Third, thanks to the simplicity of the models,
we enable a geometric understanding of our feature selection mechanisms~(\prettyref{sec:guarantees}).
\end{remark}

%% file: algorithmsAndGuarantees.tex
%!TEX root = main.tex

\section{Attention Allocation: Algorithms and Performance Guarantees}
\label{sec:algorithmsAndGuarantees}

In this section we discuss computational approaches to find a 
set of features % $\setSapp$ 
that approximately solves the feature selection 
problem~\eqref{eq:featureSelection}. It is known that finding the 
optimal subset $\setSopt$ which solves~\eqref{eq:featureSelection} exactly is NP-hard~\cite{Bian06ipsn-sensorSelection}, hence we cannot hope to find 
efficient algorithms to compute $\setSopt$ in real-world problems.\footnote{
	In typical real-world problems, the set of available visual feature is larger than 200, and 
	we are asked to select $10-100$ features, depending on on-board resources.
	In those instances, the cost of a brute force search is prohibitive. 
}
The solution we adopt in this paper is to design \emph{approximation algorithms}, 
which are computationally efficient and provide performance guarantees (roughly speaking, 
produce a set % $\setSapp$  
which is not far from the optimal subset $\setSopt$).
We remark that we are designing a \emph{covert attention mechanism}: our algorithms 
only select a set of features that have to be retained and used for state estimation, 
while we do not attempt to actively control the motion of the camera.
% Also, one of the design requirements, is that the input data for our attention 
% algorithms need to be easy to compute: 
% pre-attentive algorithms: feature extraction
% attention system: need to be fast
% MILP
% convex relaxation
% %\subsection{Optimal Attention Allocation}
% Optimal Attention Allocation: combinatorial

In the following we present two classes of algorithms.
The former is based on a 
convex relaxation of the original combinatorial problem~\eqref{eq:featureSelection}. 
\edited{We do not claim the convex relaxations as original contributions, since they have been proposed multiple times in other contexts, 
see, e.g.,~\cite{Joshi09tsp-sensorSelection}. We briefly review the convex relaxations and the corresponding performance guarantees 
in~\prettyref{sec:convexRelax} and~\prettyref{sec:convexRelaxGuarantees}.
The second class of approximation algorithms includes  greedy selection methods, and are discussed in greater details  in~\prettyref{sec:greedy}. 
We provide performance guarantees for the greedy algorithms in~\prettyref{sec:guarantees}.}

%% file: convexRelaxation.tex
%!TEX root = main.tex

\subsection{Convex Relaxations}
\label{sec:convexRelax}

This section presents a convex-relaxation approach to compute an approximate solution for problem~\eqref{eq:featureSelection}. 

Using~\eqref{eq:probabilisticMetrics}, we rewrite 
 problem~\eqref{eq:featureSelection} explicitly as:
\bea
\label{eq:featureSelectionExplicit}
\max_{\setS \subset \setF} & f\left( \barMOmega_{\k:\k+\hor} + \sum_{\l \in \setS} p_\l \Delta_\l \right) 
% \nonumber \\
\;
\subject  |\setS| \leq \kappa  
\eea
where $f(\cdot)$ denotes either  $\feigMin(\cdot)$ or  $\flogDet(\cdot)$
(for the moment there is not need to distinguish the two metrics).

Introducing binary variables $s_\l$, for $\l=1,\ldots,N$, we rewrite~\eqref{eq:featureSelectionExplicit} 
equivalently as: 
\bea
\label{eq:featureSelectionBinary}
\max_{s_1, \ldots, s_N} & f\left( \barMOmega_{\k:\k+\hor} + \sum_{\l \in \setS} s_\l p_\l \Delta_\l \right) 
\\
%\quad
\subject & \sum_{\l=1}^N s_l  \leq \kappa  \;\;, \;\;  s_l \in \{0,1\} \; \forall \; l \in \{1,\ldots,N\}
\nonumber
\eea
Problem~\eqref{eq:featureSelectionBinary} is a binary optimization problem.
While problem~\eqref{eq:featureSelectionBinary} would return the optimal subset $\setSopt$, it is still NP-hard to solve, 
due to the constraint that $s_\l$ have to be binary.

Problem~\eqref{eq:featureSelectionBinary} admits an simple 
convex relaxation:
\bea
\label{eq:convexRelaxation}
f^\star_\cvx = \max_{s_1, \ldots, s_N} & f\left( \barMOmega_{\k:\k+\hor} + \sum_{\l \in \setS} s_\l p_\l \Delta_\l \right) 
\\
%\quad
\subject & \sum_{\l=1}^N s_l  \leq \kappa  \;\;, \;\;  s_l \in [0,1] \; \forall \; l \in \{1,\ldots,N\}
\nonumber
\eea
where the binary constraint $s_l \in \{0,1\}$ is replaced by the convex constraint 
$s_l \in [0,1]$. Convexity of problem~\eqref{eq:convexRelaxation} follows from the fact that 
 we maximize a concave cost under linear inequality constraints.\footnote{Both the 
smallest eigenvalue and the log-determinant of a positive definite matrix are concave functions~\cite{Boyd04book} 
of the matrix entries.} % which are convex in the unknowns $s_\l$.

This convex relaxation  has been proposed multiple times in other contexts 
(see, e.g.,~\cite{Joshi09tsp-sensorSelection}). The solution ${s_1^\star, \ldots, s_N^\star}$ of~\eqref{eq:convexRelaxation} is not binary in 
general and a rounding procedure is needed to distinguish the features 
that have to be discarded ($s_\l=0$) from the features that have to be selected ($s_\l=1$).
A common rounding procedure is to simply select the $\kappa$ features with the largest $s_\l$, 
while randomized rounding procedures have also been considered~\cite{Lerner07tro}.
We use the former, and we call $\setScvx$ the set including the indices of the 
$\kappa$ features with the largest $s_\l^\star$, where ${s_1^\star, \ldots, s_N^\star}$ is the 
optimal solution of~\eqref{eq:convexRelaxation}.

% \myparagraph{Performance guarantees} 
\subsection{Performance Guarantees for the Convex Relaxations}
\label{sec:convexRelaxGuarantees}

The convex relaxation~\eqref{eq:convexRelaxation} has been observed to work well in practice, although there is no 
clear (a-priori) performance guarantee on the quality of the set $\setScvx$. 

Let us call $f^\star_\cvx$ the optimal objective of the relaxed problem~\eqref{eq:convexRelaxation}, 
$f(\setScvx)$ the objective attained by the rounded solution, and $f(\setSopt)$ the 
optimal solution of the original NP-hard problem~\eqref{eq:featureSelectionBinary}.
Then, one can easily obtain a-posteriori 
performance bounds by observing that:
\bea
\label{eq:convexInequalityChain}
f(\setScvx) \leq f(\setSopt) \leq f^\star_\cvx 
\eea
where the first inequality follows from optimality of $\setSopt$ 
(any subset of $\kappa$ features has cost at most $f(\setSopt)$), 
while the latter from the fact that~\eqref{eq:convexRelaxation} is a 
relaxation of the original problem.
% (having a larger feasible set can attain a larger objective).    

The chain of inequality~\eqref{eq:convexInequalityChain} suggests a simple (a-posteriori) 
performance bound for the quality of the set produced by the convex relaxation~\eqref{eq:convexRelaxation}:
\bea
\label{eq:guaranteesConvex}
f(\setSopt) - f(\setScvx) \leq f^\star_\cvx - f(\setScvx) 
\eea
i.e., the suboptimality gap $f(\setSopt) - f(\setScvx)$ of the subset $\setScvx$ 
is bounded by the difference $f^\star_\cvx - f(\setScvx)$, which can be computed (a posteriori) 
after solving~\eqref{eq:convexRelaxation}.

%% file: greedy.tex
%!TEX root = main.tex

\subsection{Greedy Algorithms and Lazy Evaluation}
\label{sec:greedy}

This section presents a second approach to approximately solve problem~\eqref{eq:featureSelection}.
Contrarily to the convex relaxation of \prettyref{sec:convexRelax}, here 
we consider a greedy algorithm that selects $\kappa$ features that (approximately) 
maximize the cost $f(\cdot)$. 

The algorithm 
starts with an empty set $\setSgreedy$ and 
performs 
$\kappa$ iterations. At each iteration, it adds the feature that, 
if added to $\setSgreedy$, induces the largest increase in the cost function.
The pseudocode of the algorithm is given in~\prettyref{alg:greedy}.

\begin{algorithm}[h]
\SetAlgoLined
\textbf{Input:} \ $\barMOmega_{\k:\k+\hor}$, $\Delta_\l$, for $\l = 1,\ldots,N$, and $\kappa$ \; 
\textbf{Output:} \ feature subset $\setSgreedy$ \; 
%\% Initialization \\
$\setSgreedy = \emptyset$ \; \label{line:initSet}
 \For{$i = 1,\ldots,\kappa$}{ \label{line:forLoop1}
 	\% Compute upper bound for $f(\setSgreedy \!\cup \!\edited{\{\l\}})$, $\!\l=1,\ldots,N\!\!$ \\
 	$[U_1, \ldots, U_N] = \text{upperBounds}(\barMOmega_{\k:\k+\hor}, \Delta_1, \ldots,\Delta_N)\!$ \;
 	\label{line:upperBound}
 	\% Sort features using upper bound \\
 	$\setF^\downarrow = \text{sort}(U_1,\ldots,U_N)$ \; \label{line:sort}
 	\% Initialize best feature \\
 	$f_\max = -1$ \, ; \, $\l_\max = -1$ \;
 	\For{$\l \in \emph{\setF}^\downarrow$}{ \label{line:forLoop2}
 	 	\If{$U_l < f_\max$}{   \label{line:break1}
  			\textbf{break} \;   \label{line:break2}
		}						\label{line:break3}
 		\If{$f(\setSgreedy \cup \l) > f_\max$}{ \label{line:costCheck1}
  			$f_\max = f(\setSgreedy \cup \l)$ \, ; \,  $\l_\max = \l$ \; \label{line:costCheck2}
		} \label{line:costCheck3}
 	}
  $\setSgreedy = \setSgreedy \cup \l_\max$ \; \label{line:addFeature}
 }
 \caption{Greedy algorithm with lazy evaluation\label{alg:greedy}}
\end{algorithm}  

In line~\ref{line:initSet} the algorithm starts with an empty set.
The ``for'' loop in line~\ref{line:forLoop1} iterates $\kappa$ times: 
at each time the best feature is added to the subset $\setSgreedy$ (line~\ref{line:addFeature}).
The role of the ``for'' loop in line~\ref{line:forLoop2} is to compute the feature that induces 
the maximum increase in the cost (lines~\ref{line:costCheck1}-\ref{line:costCheck3}).
The remaining lines provide a lazy evaluation mechanism. 
For each feature $\l$ we compute an \editedTwo{upper bound on the cost $f(\setSgreedy \cup \{\l\})$} (line~\ref{line:upperBound}).
The features are sorted (in descending order)
 according to this upper bound (line~\ref{line:sort}).
 The advantage of this is that by comparing the current best feature with 
 this upper bound (line~\ref{line:break1}) we can avoid checking features that 
 are guaranteed to attain a smaller cost.

 Clearly, the lazy evaluation is advantageous if the upper bound is faster to compute 
 than the actual cost. The following propositions provide two useful (and computationally 
 cheap) upper bounds for our cost functions.

\begin{proposition}[Upper bounds for $\logdet$: Hadamard's inequality, Thm 7.8.1~\cite{Horn85book}]
\label{prop:boundLogDet}
For a positive definite matrix $\MM \in \Real{n \times n}$ with diagonal elements $\MM_{ii}$, 
it holds: 
\bea
\det(\MM) \leq \textstyle \prod_{i=1}^n \MM_{ii} 
\Leftrightarrow
\logdet(\MM) \leq \textstyle \sum_{i=1}^m \log \MM_{ii}
\eea
\end{proposition}

\begin{proposition}[Eigenvalue Perturbation Bound~\cite{Ipsen09siam-eigenvaluePerturbation}]
\label{prop:boundsIpsen}
Given Hermitian matrices $\MM, \Delta \in \Real{n \times n}$, and 
denoting with $\lambda_i(\MM)$ the $i$-th eigenvalue of $\MM$, the following inequalities hold:
\bea
|\lambda_i(\MM+\Delta) - \lambda_i(\MM)| &\leq& \|\Delta\|  \label{eq:weyl}\\
\min_j |\lambda_i(\MM) - \lambda_j(\MM+\Delta)| &\leq& \|\Delta \vv_i \|  \label{eq:ipsen}
\eea
where $\vv_i$ is the eigenvector of $\MM$ associated to $\lambda_i(\MM)$.
\end{proposition}
Eq.~\eqref{eq:weyl} is a restatement of the classical Weyl inequality, while~\eqref{eq:ipsen} is a tighter bound from Ipsen
 and Nadler~\cite{Ipsen09siam-eigenvaluePerturbation}.
 To clarify how the bounds in~\prettyref{prop:boundsIpsen} provide us with an upper bound for $\lambdaMin$, we prove the following result.

\begin{corollary}[Upper bounds for $\lambdaMin$]
\label{prop:boundEigMin}
Given two symmetric and positive semidefinite matrices $\MM, \Delta \in \Real{n \times n}$ the following inequality holds:
\bea
\label{eq:boundEigMin}
\lambdaMin(\MM+\Delta) \leq \lambdaMin(\MM) + \|\Delta \vv_\min \| %\leq \lambdaMin(\MM) +  \|\Delta\|
\eea
where $\vv_\min$ is the eigenvector of $\MM$ associated to the smallest eigenvalue $\lambdaMin(\MM)$.
\end{corollary}

\edited{The proof of Corollary~\ref{prop:boundEigMin} is given in Appendix~\ref{app:proof:prop:boundEigMin}.}
While~\prettyref{alg:greedy} highlights the simplicity of the greedy algorithm, it is unclear 
whether this algorithm produces good subsets of features. We tackle this question in the next section. 

%% file: guarantees.tex
%!TEX root = main.tex
\subsection{Performance Guarantees for the Greedy Algorithm}
\label{sec:guarantees}

This section shows that the greedy algorithm (\prettyref{alg:greedy})
 admits provable sub-optimality bounds. These bounds
guarantee that the greedy selection cannot perform much worse than the
optimal strategy.
The section tackles separately the two metrics presented in~\prettyref{sec:metrics}, 
since the corresponding performance guarantees are fundamentally different.

Our results are based on the recent literature on submodularity and 
submodular maximization.
Before delving in the guarantees for each metric, we provide 
few preliminary definitions, which can be safely skipped by the expert reader. 

\begin{definition}[Normalized and Monotone Set Function~\cite{nemhauser78mp-submodularity}]
\label{def:normalizedMonotone}
A set function $f: 2^\setF \rightarrow \Real{}$ is said to be \emph{normalized} if $f(\emptyset) = 0$;
$f(\setS)$ is said to be \emph{monotone} (non-decreasing) if for any subsets $\setA \subseteq \setB \subseteq \setF$, 
it holds $f(\setA) \leq f(\setB)$.
\end{definition}

\begin{definition}[Submodularity~\cite{nemhauser78mp-submodularity}]
A set function $f: 2^\setF \rightarrow \Real{}$ is \emph{submodular} if, 
for any subsets $\setA \subseteq \setB \subseteq \setF$, and for any element $e \in \setF \setminus \setB$,
 it holds that:
\bea
f(\setA \cup \{e\}) - f(\setA) \geq f(\setB \cup \{e\}) - f(\setB)
\eea
\end{definition}
Submodularity formalizes the notion of diminishing returns: 
adding a measurement to a small set of measurement is more advantageous than adding 
it to a large set. 
 Our interest towards submodularity is motivated by the following result.
\begin{proposition}[Near-optimal submodular maximization~\cite{nemhauser78mp-submodularity}]
\label{prop:submodularityGuarantees}
Given a normalized, monotone, submodular set function $f: 2^\setF \rightarrow \Real{}$,
and calling $\setSopt$ the optimal solution of the maximization problem~\eqref{eq:featureSelection}, 
then the set $\setSgreedy$, computed by the greedy~\prettyref{alg:greedy}, is such that:
\bea
f(\setSgreedy)  \geq (1-1/\e) f(\setSopt) \approx 0.63 f(\setSopt)
\eea
\end{proposition}
This bound ensures us that the worst-case performance of the greedy algorithm 
cannot be far from the optimum. In the following we tailor this result to our 
feature selection problem.

%%%%%%%%%%%%%%%%%%%%%%%%%%%%%%%%%%%%%%%%%%%%%%%%%%%%%%%%%%%%%%%%%
\subsubsection{Sub-optimality guarantees for $\logdet$}

%Related work in sensor scheduling proves submodularity and monotonicity of 
%the $\logdet$. A formal result is given in the following.
It is possible to show that $\logdet$ is submodular with respect to the 
set of measurements used for estimation. This result and the 
corresponding performance guarantees are formalized as follows.
%, together with the corresponding performance guarantees.

\begin{proposition}[Submodularity of $\logdet$~\cite{Shamaiah10cdc-sensorScheduling}]
\label{prop:submodularityLogDet}
The set function $\flogDet(\setS)$ defined in~\eqref{eq:logDet}
is monotone and submodular. Moreover, the greedy algorithm applied to~\eqref{eq:featureSelection} using $\flogDet(\setS)$ as objective enjoys
the following performance guarantees:
\bea
\label{eq:submodularityLogDet}
\flogDet(\setS) \geq (1-1/\e) \flogDet(\setSopt) + \frac{\flogDet(\emptyset)}{\e}
\eea
\end{proposition}

The result is proven in~\cite{Shamaiah10cdc-sensorScheduling} %is given for scalar measurements, 
and has been later rectified to account for the need of normalized functions in~\cite{Jawaid15automatica-sensorScheduling}.
The extra term $\frac{\flogDet(\emptyset)}{e}$ in~\eqref{eq:submodularityLogDet} indeed follows from the application 
of~\prettyref{prop:submodularityGuarantees} 
to the normalized function $\flogDet(\setS) - \flogDet(\emptyset)$.

%%%%%%%%%%%%%%%%%%%%%%%%%%%%%%%%%%%%%%%%%%%%%%%%%%%%%%%%%%%%%%%%%
\subsubsection{Sub-optimality guarantees for $\lambdaMin$}
% Our interest towards the maximization of $\lambdaMin$ 
% While related work already provides proof of submodularity of $\flogDet(\setS)$, the maximization 
% of the $\logdet$ may have some drawbacks. Maximizing $\flogDet(\setS)$ essentially means minimizing 
% the volume of the uncertainty ellipsoid of the estimation error; however, the volume of the 
% ellipsoid can be small also when the ellipsoid if ``flat'': intuitively, the performance metric 
% may reward a large uncertainty reduction in some dimensions, while letting the uncertainty 
% grow in others. 
% %Our interest towards $\lambdaMin$ is motivated by the fact that in some 
% %applications one would rather prefer to minimize the worst case error.
% %This behavior is somehow undesirable in VIN, since we would like to keep 
% %uncertainty small along every direction (say, along every Cartesian axis).
% The previous considerations motivate us to investigate 
%  performance guarantees for our second metric, $\feigMin(\setS)$, which minimizes the worst-case error. 
 Currently, no result is readily available to bound the suboptimality gap of the greedy algorithm applied to the maximization of the smallest
  eigenvalue of the information matrix (or equivalently minimizing the largest eigenvalue of the covariance).
  Indeed, related work provides counterexamples, showing that this metric is not submodular in general, 
  while the greedy algorithm is observed to perform well in practice~\cite{Jawaid15automatica-sensorScheduling}.
  In this section we provide a first result showing that, despite the fact that $\feigMin(\setS)$ fails 
  to be submodular, it is not far from a submodular function. This notion is made more precise in the following.

\begin{definition}[Submodularity ratio~\cite{Das11icml-submodularity,Das12nips-submodularity}]
The \emph{submodularity ratio} of a non-negative set function $f(\cdot)$ with respect to a set $\setS$ and an integer $\kappa \geq 1$ 
is defined as:
\bea
\label{eq:subRatio}
\gamma_{\setS} \doteq \min_{ \substack{ \setL \subseteq \setS, \\ \setE: |\setE|\leq\kappa, \setE \; \cap \setL = \emptyset} }
 \frac{ \sum_{e \in \setE} \left( f(\setL \cup \{e\}) - f(\setL) \right) }{ f(\setL \cup \setE) - f(\setL) }
\eea
\end{definition}
It is possible to show that if $\gamma_{\setS} \geq 1$, then the function $f(\cdot)$ is submodular.
However, in this context we are interested in the submodularity ratio, since it enables a 
less restrictive sub-optimality bound, as described in the following proposition.

\begin{proposition}[Approximate submodular maximization~\cite{Das11icml-submodularity,Das12nips-submodularity}]
\label{prop:approxSub}
Let $f(\cdot)$ be a non-negative monotone set function and let 
$\setSopt$ be the optimal solution of the maximization problem~\eqref{eq:featureSelection}, 
then the set $\setSgreedy$, computed by the greedy~\prettyref{alg:greedy} is such that:
\bea
f(\setSgreedy) \geq (1- \e^{-\gamma_{\setSgreedy}}) f(\setSopt)
\eea
where $\gamma_{\setSgreedy}$ is the submodularity ratio of $f(\cdot)$ with respect to $\setSgreedy$ 
and $\kappa = |\setSgreedy|$.
\end{proposition}

\prettyref{prop:approxSub} provides a multiplicative suboptimality bound
whenever $\gamma_{\setSgreedy} > 0$. In the following we show that this is 
generally the case when maximizing the smallest eigenvalue.

\begin{proposition}[Non-vanishing Submodularity ratio of $\lambdaMin$]
\label{prop:nonVanishingRatio}
Call $\setSgreedy$ the set returned by the greedy algorithm maximizing $\lambdaMin$. 
For any set $\setL \subseteq \setSgreedy$  
call $\vmubar$ the eigenvector corresponding to the smallest eigenvalue of the 
matrix $\barMOmega_{\k:\k+\hor} + \sum_{\l \in \setL} \Delta_\l$.
Moreover call $\vmubar_0, \vmubar_2, \ldots, \vmubar_\hor \in \Real{3}$, the subvectors of 
$\vmubar$ corresponding the robot positions at time $\k,\ldots,\k+\hor$.
Then the submodularity ratio of $\lambdaMin$ is bounded away from zero if 
  $\vmubar_i \neq \vmubar_j$, for some $i, j$.
\end{proposition}

\edited{The proof of~\prettyref{prop:nonVanishingRatio} is given in Appendix~\ref{app:proof:prop:nonVanishingRatio}.}
In words,~\prettyref{prop:nonVanishingRatio} states that 
the submodularity ratio does not vanish as long as the directions of largest uncertainty 
change along the \edited{future} horizon. The following corollary is a straightforward consequence of~\prettyref{prop:nonVanishingRatio}.
%This is an initial result to explain the excellent performance of 
%for the different positions in the horizon. 
%While this is admittedly a loose result, it provides a partial explanation 

\begin{corollary}[Approximate submodularity of $\lambdaMin$]
\label{cor:approxSumLambdaMin}
The set function $\feigMin(\setS)$ defined in~\eqref{eq:probabilisticMetrics}
is monotone. Moreover, under the assumptions of~\prettyref{prop:nonVanishingRatio}, 
the greedy algorithm applied to~\eqref{eq:featureSelection} \editedTwo{using $\feigMin(\setS)$} 
as objective enjoys the guarantees of~\prettyref{prop:approxSub} 
for a nonzero $\gamma_{\setSgreedy}$.
\end{corollary}

\begin{proof}
Monotonicity follows from the Weyl inequality~\cite{Ipsen09siam-eigenvaluePerturbation}.
% (adding a positive semidefinite matrix cannot decrease the smallest eigenvalue of a matrix). 
The guarantees of the greedy algorithm follow from~\prettyref{prop:approxSub} 
and~\prettyref{prop:nonVanishingRatio}.
\end{proof}

\prettyref{cor:approxSumLambdaMin} guarantees that the approximation bound of~\prettyref{prop:approxSub} 
does not vanish, hence the greedy algorithm always approximate the optimal solution up to a constant-factor.
% This is in contrast with~\cite{Zhang15cvpr}, in which the additive bound can easily produce vacuous guarantees. 
Empirical evidence, shown in~\prettyref{sec:experiments}, confirms that the greedy algorithm applied to 
the maximization of $\feigMin(\setS)$ has excellent performance, producing near-optimal results in all test instances. 

\begin{remark}[Geometric Intuition Behind Greedy Selection with $\lambdaMin$]
Our linear model enables a deeper understanding of the geometry behind the 
greedy selection. The greedy selection rewards features $\l$ with large 
objective $\lambdaMin(\barMOmega_{\k:\k+\hor} + \Delta_\l)$ or, equivalently, 
large marginal gain $\lambdaMin(\barMOmega_{\k:\k+\hor} + \Delta_\l) - \lambdaMin(\barMOmega_{\k:\k+\hor})$.
The following chain of relations provides a geometric understanding of which features induce 
%are the features with 
a large marginal gain:
%% LONG VERSION
\beal
\lambdaMin(\barMOmega_{\k:\k+\hor} + \Delta_\l) - \lambdaMin(\barMOmega_{\k:\k+\hor}) \nonumber \\ 
\grayText{(from Rayleigh quotient)} \nonumber \\
= \min_{\|\vnu\|=1}  \vnu\tran (\barMOmega_{\k:\k+\hor} + \Delta_\l) \vnu -
\min_{\|\vmu\|=1} \vmu\tran (\barMOmega_{\k:\k+\hor}) \vmu \nonumber \\
\grayText{(calling $\bar\mu$ the minimizer of the second summand)} \nonumber \\
=  \min_{\|\vnu\|=1} \vnu\tran  (\barMOmega_{\k:\k+\hor} + \Delta_\l) \vnu -
\vmubar\tran (\barMOmega_{\k:\k+\hor}) \vmubar  \nonumber \\
\grayText{(\edited{substituting a suboptimal solution $\bar\mu$ in the first summand})} \nonumber \\
\leq  \vmubar\tran (\barMOmega_{\k:\k+\hor} + \Delta_\l) \vmubar -
\vmubar\tran (\barMOmega_{\k:\k+\hor}) \vmubar \nonumber \\
\grayText{(simplifying and substituting the expression of $\Delta_\l$)} \nonumber \\
= \vmubar\tran \Delta_\l \vmubar = 
\vmubar\tran  \MF_\l\tran (\eye - \ME_\l (\ME_\l\tran \ME_\l)\inv  \ME_\l\tran) \MF_\l \vmubar
\nonumber \\
\grayText{(defining the idempotent matrix $\MQ \doteq (\eye - \ME_\l (\ME_\l\tran \ME_\l)\inv  \ME_\l\tran)$)} \nonumber \\
= \vmubar\tran  \MF_\l\tran \MQ \MF_\l \vmubar  = 
\vmubar\tran  \MF_\l\tran \MQ \MQ \MF_\l \vmubar  =
 \| \MQ \MF_\l \vmubar \|^2 
\nonumber \\
\grayText{(using the triangle inequality and substituting $\MF_\l$)} \nonumber \\
\leq  \| \MQ \|^2 \| \MF_\l \vmubar \|^2 = \| \MF_\l \vmubar \|^2 = 
\sum_{\k=0}^\hor \| [\vu_{\k\l}]_\times (\MR^\world_{\cam,\c})\tran \vmubar_\k \|^2
\eeal
where $\vmubar_\k$ is the subvector of $\vmubar$ at the entries corresponding to the 
robot position \edited{coordinates} at time $\k$. Intuitively, the inequalities reveal that the marginal gain 
is small when $\| [\vu_{\k\l}]_\times (\MR^\world_{\cam,\c})\tran \vmubar_\k \|$ is small, i.e., when we pick landmark 
observations where the measured bearing $\vu_{\k\l}$ is nearly parallel to 
the directions of large uncertainty $\vmubar_\k$, transformed in the 
camera frame by the rotation $(\MR^\world_{\cam,\c})\tran$. 
For instance, if we have
%For instance, if we have 
large uncertainty in the forward direction, it is not convenient to use features in front of the robot 
(i.e., with bearing parallel to the direction of largest uncertainty);
accordingly, the greedy approach would select features in the
periphery of the image, which intuitively provide a better way to reduce uncertainty.
 %which confirms the intuition that features in the
%periphery of the image may be more informative to reduce the uncertainty in this case.
%and it may be more convenient to select features in the
%periphery of the image. %, rather in the front. % of the robot.

% tends to be 
% large when $\| [\vu_{\k\l}]_\times (\MR^\world_{\cam,\c})\tran \vmubar_\k \|$ is large, i.e., when we pick landmark 
% observations where the measured bearing $\vu_{\k\l}$ is as orthogonal as possible to 
% the directions of large uncertainty $\vmubar_\k$, transformed in the 
% camera frame by the rotation $(\MR^\world_{\cam,\c})\tran$. For instance, if we have 
% large uncertainty in the forward direction, it is more convenient to use features 
% in the periphery of the image, rather in the front. % of the robot.
% with lazy evaluation prioritizes 
% features with large upper bounds. According to~\prettyref{prop:boundsIpsen} 
% the upper bound for the $\l$-th feature is $\|\Delta_\l \mu \|$ where $\mu$ 
% is the eigenvector corresponding to the smallest eigenvalue of the information matrix, 
% or equivalently the eigenvector corresponding to the largest eigenvalue of the covariance. 
% Intuitively, $\mu$ describes the direction in the state space in which the uncertainty
% of our estimate is large. Recalling the structure of $\Delta_\l$ from~\eqref{eq:infoMatrixVision}, the 
% following inequalities follow:
% \beq
% \|\Delta_l \mu \|^2 = \| \MF_\l\tran (\eye - \ME_\l (\ME_\l\tran \ME_\l)\inv  \ME_\l\tran) \MF_\l \mu \|^2
% \eeq

\end{remark}

%% file: experiments.tex
%!TEX root = main.tex

\newcommand{\greedy}{\scenario{greedy}}
\newcommand{\rounded}{\scenario{rounded}}
\newcommand{\relaxed}{\scenario{relaxed}}
\newcommand{\nofeatures}{\scenario{no features}}
\newcommand{\random}{\scenario{random}}
\newcommand{\eigen}{\scenario{eigen}}

\newcommand{\minEig}{\scenario{minEig}}
\newcommand{\logDet}{\scenario{logDet}}
\newcommand{\quality}{\scenario{quality}}
\newcommand{\noselection}{\scenario{no}-\scenario{selection}}
\newcommand{\relRotErr}{\epsilon^r_R}
\newcommand{\relTranErr}{\epsilon^r_T}
\newcommand{\absRotErr}{\epsilon^a_R}
\newcommand{\absTranErr}{\epsilon^a_T}

\section{Experiments}
\label{sec:experiments}

This section provides three sets of experimental results. 
The first set of tests, in~\prettyref{sec:easySim}, 
shows that the greedy algorithm attains near-optimal solutions 
in solving problem~\eqref{eq:featureSelection}, 
%and it's more stable than 
% convex relaxation techniques, 
while being faster than \edited{general purpose solvers for the convex relaxations discussed in~\prettyref{sec:convexRelax}}.
The second set of tests, in~\prettyref{sec:monteCarlo}, evaluates 
our c++ pipeline in realistic simulations, showing that our feature 
selection techniques boost VIN performance; 
the same section also shows the 
advantage of using our lazy evaluation.
The third set of tests, in~\prettyref{sec:realTests}, 
evaluates our approach on real data collected by an agile micro aerial vehicle. % moving in indoor environments.

%%%%%%%%%%%%%%%%%%%%%%%%%%%%%%%%%%%%%%%%%%%%%%%%%%%%%%%%%%%%%%%%%%%%%%%%%%%%%%%%%%%%%%%%%%
\subsection{Assessment of the Greedy Algorithms for Feature Selection}
\label{sec:easySim}

This section answers the following question: \emph{how good is the greedy~\prettyref{alg:greedy} 
to (approximately) solve the combinatorial optimization problem~\eqref{eq:featureSelection}?} 
In particular, we show that the greedy algorithm finds a 
near-optimal solution of~\eqref{eq:featureSelection}, for 
both choices of the cost function~\eqref{eq:probabilisticMetrics}; 
we also show that the convex relaxation approach of~\prettyref{sec:convexRelax} 
finds near-optimal solutions, while being more computationally expensive.

%% =================================================
\mysubparagraph{Testing setup} To generate random instances 
of problem~\eqref{eq:featureSelection}, we consider a robot moving along a straight line at  
a constant speed of $2\meters/\sec$. The robot is equipped with an IMU 
with sampling period $\Deltak = 0.01\sec$; we choose the accelerometer noise 
density equal to $0.02 \meters/(\sec^2\sqrt{\Hz})$, and the
accelerometer bias continuous-time noise density to be 
 $0.03 \meters/(\sec^3\sqrt{\Hz})$.
 We also simulate an on-board monocular camera, which measures 
 3D points randomly scattered in the environment, at a (key)frame rate of $0.5\sec$.
The robot has to select a set of $\kappa$ features
out of $N$ available visual measurements. 
We assume that at the time of feature selection, the 
position covariance of the robot is $10^{-2} \cdot \eye_3$, while its velocity 
and accelerometer bias covariances are $10^{-2}\cdot\eye_3$ and $10^{-4} \cdot\eye_3$, respectively. 
Using this information, we build the matrix $\barMOmega_{\k:\k+\hor}$, using  
a prediction horizon of $2.5\sec$. Moreover, from the available feature measurements,
we build the matrices $\Delta_\l$; in these tests we assume $p_\l = 1$, i.e., 
we disregard appearance during feature selection.

%% =================================================
\mysubparagraph{Techniques and evaluation metrics} We compare two approaches to solve~\eqref{eq:featureSelection}: 
the greedy algorithm of~\prettyref{alg:greedy} and the convex relaxation approach~\eqref{eq:convexRelaxation}.
We implemented the convex relaxation using \CVX/\MOSEK as parser/solver for~\eqref{eq:convexRelaxation}, and 
then we computed the rounded solution as described in~\prettyref{sec:convexRelax}. 
For the evaluation in this section, we implemented both the greedy algorithm and the convex relaxation
in Matlab.
We evaluate these approaches for each 
choice of the objective functions $\feigMin$ and $\flogDet$ defined in~\eqref{eq:probabilisticMetrics}. 
Ideally, for each technique, we should compare the objective attained by the techniques, 
versus the optimal objective. Unfortunately, the optimal objective is hard to compute 
and a brute-force approach is prohibitively slow, even for relatively small problem 
instances.\footnote{Even in a small instance in which we are required to select 50
out of 100 available visual measurements, a brute-force approach would need 
to evaluate around $10^{29}$ possible sets.} 
Luckily, the convex relaxation~\eqref{eq:convexRelaxation} also produces an upper bound 
on the optimal cost of~\eqref{eq:featureSelection} (\emph{c.f.} eq.~\eqref{eq:convexInequalityChain}), 
hence we can use this upper bound to understand how far are the greedy and the rounded solution of~\eqref{eq:convexRelaxation} 
from optimality.

%% =================================================
\mysubparagraph{Results} We consider problems of increasing sizes in which we are 
given $N$ features and we have to select half of them ($\kappa = N/2$) to maximize 
the objective in~\eqref{eq:featureSelection}. For each $N$, we compute statistics over $50$ Monte Carlo.

\prettyref{fig:matlabSim}(a) shows the smallest eigenvalue objective $\feigMin$ 
attained by 
 the different techniques for increasing number of features $N$. 
Besides the greedy, the rounded convex relaxation (label: \rounded), and 
the relaxed objective (label: \relaxed), we show 
%the objective attained when no features are selected 
%(label: \nofeatures) and 
the objective attained by picking a random subset of $\kappa$ features (label: \random).  
% and compares the objective attained by the greedy solution, 
% against the rounded solution resulting from the convex relaxation~\eqref{eq:convexRelaxation} 
% for increasing problem sizes $N$. %The costs are averaged over the 50 runs.
We are solving a maximization problem hence the larger the objective the better.
\prettyref{fig:matlabSim}(a) shows that in all tested instances,  \greedy and \rounded match the upper bound \relaxed 
(the three lines are practically indistinguishable), hence 
they both produce optimal solutions (\emph{c.f.} eq.~\eqref{eq:guaranteesConvex}). 
The resulting solution is far better than \random.
This result is somehow surprising, since 
the smallest eigenvalue is not submodular in general, and the greedy algorithm enjoys 
 weaker performance guarantees (\prettyref{cor:approxSumLambdaMin}). However, this observation is in agreement with  
  related work in other fields, e.g.,~\cite{Zhang15cdc-sensorSelection}.
% For larger problem sizes, however, \CVX 
% experiences numerical problems, often failing at returning a feasible solution.
% Moreover, both the quality of the upper bound and the rounded solution from the 
% convex relaxation degrades, with the upper bound falling below the objective 
% produces by the greedy algorithm (we impute this fact to numerical reasons).
While both \greedy and \rounded return good solutions, solving the convex problem~\eqref{eq:convexRelaxation}  
is usually more expensive than computing the greedy solution:  
%While we will provide timing results on a c++ implementation in the following sections, 
%it is worth noting that 
the CPU time of our greedy algorithm in Matlab (without lazy evaluation)
is around $0.4\sec$ (for $N = 50$), while \CVX requires around $0.8\sec$.

\input{figMatlabSim}

Analogous considerations hold for the objective $\flogDet$. 
\prettyref{fig:matlabSim}(b) shows the log-determinant attained by the different techniques, 
for increasing number of available features $N$; also in this case the algorithms have to
select $\kappa = N/2$ features. As in the previous tests, 
\greedy and \rounded attain the optimal solution in all test instances, matching 
 the upper bound \relaxed, and performing remarkably better than a random choice.
% In the case of the smallest eigenvalue, the 
% performance of the compared techniques is indistinguishable. 
Regarding the CPU time, our Matlab implementation of the greedy algorithm 
to optimize \edited{$\flogDet$} takes  
 around $0.1\sec$ (for $N = 50$), while \CVX requires more than 1min to 
solve~\eqref{eq:convexRelaxation}.\footnote{\CVX uses a successive approximation method 
to maximize the log-det objective, which is known to be fairly slow.}
\edited{We remark that while \CVX is a state-of-the-art general-purpose solver 
for convex programming, our analysis does not rule out the possibility of designing 
fast specialized solvers, e.g.,~\cite{Rosen16wafr-sesync}; such an attempt, however, is outside the scope of this paper.
} 
Since the greedy algorithms are as accurate as the convex relaxation techniques, 
while being faster \editedTwo{than general-purpose convex solvers}, in the following we focus on the former.

%%%%%%%%%%%%%%%%%%%%%%%%%%%%%%%%%%%%%%%%%%%%%%%%%%%%%%%%%%%%%%%%%%%%%%%%%%%%%%%%%%%%%%%%%%
\subsection{Importance of Feature Selection in VIN}
\label{sec:monteCarlo}

This section answers the following question: \emph{does the feature selection 
resulting by solving~\eqref{eq:featureSelection} lead to performance improvements in VIN?} 
In the following we show that the proposed feature selection approach boosts 
VIN performance in realistic Monte Carlo simulations. %, while beings fast to compute.

%% =================================================
\mysubparagraph{Testing setup} 
We adopt the benchmarking problem of~\cite{Forster16tro}
and pictured in~\prettyref{fig:monteCarlo}(a) as testing setup. 
We simulate a robot that follows a circular trajectory 
 with a sinusoidal vertical motion. The total length of the 
trajectory is $120\meters$. The on-board camera 
 has a focal length of $315$ pixels and runs at a rate of $2.5\Hz$ (simulating keyframes).
Simulated acceleration and gyroscope measurements are obtained as in~\cite{Forster16tro}. 

%% =================================================
\mysubparagraph{Implementation and evaluation metrics}
In this section we focus on the greedy algorithms and we use those to 
select a subset of visual features. We implemented the greedy algorithms and the 
construction of the
matrices required in the functions~\eqref{eq:probabilisticMetrics} in c++, 
using $\eigen$ for the computation of the log-determinant and the smallest eigenvalue. 
For numerical reasons, rather than computing the determinant and taking the logarithm, 
we directly compute the log-determinant from the Cholesky decomposition of the 
 matrix. For the computation of the smallest eigenvalue we use \eigen's svd 
function. 

Our feature selection approach is used as an add-on to a visual-inertial pipeline 
similar to the one described in~\cite{Forster16tro}. Our VIN pipeline estimates the 
navigation state (robot pose, velocity, and IMU biases) using the structureless visual 
model and the pre-integrated IMU model described in~\cite{Forster16tro}.
The entire implementation is based on the GTSAM optimization library~\cite{Dellaert12tr}.
Our implementation differs from~\cite{Forster16tro} in three important ways.
First, in this paper we use the iSAM2 
algorithm within a fixed-lag smoothing approach; we marginalize out states outside 
a smoothing horizon of $6\sec$, which helps bounding latency and memory requirements.
% of
%our implementation.\footnote{While iSAM2 has constant complexity in odometry problems, 
%its memory requirements still grow over time.}
Second, we do not adopt SVO as visual front-end; in this simulations we do not need a 
front-end as we simulate landmark observations, while in the following section we describe 
a simple real-world front-end. Finally, rather than feeding to the VIN estimator all available measurements, 
we use the feature selection algorithms described in this paper to select a small 
set of informative visual observations.

In this section we evaluate two main aspects of our approach.
First, we show that a clever selection of the features does actually impact VIN 
accuracy. Second, we show that the lazy evaluation approach discussed in~\prettyref{sec:greedy}
 speeds up the computation of the greedy solution. 
 We use two metrics for accuracy: the \emph{absolute translation error}, which is the 
 Euclidean distance between the estimated position and the actual position, and 
 the \emph{relative translation error}, which computes the Euclidean norm of the 
 difference between the 
 estimated translation between time $k$ and time $k+1$ and the actual translation.
 Indeed the relative translation error 
 quantifies how quickly the estimate drifts from the ground truth.  
 Since absolute positions are not observable in visual-inertial odometry, 
 the relative error is a more reliable performance metric. 
When useful, we report absolute and relative rotation errors (defined in analogy with 
the translation ones). 

\input{figMonteCarloResults}

%% =================================================
\mysubparagraph{Results} We simulate 50 Monte Carlo runs; in each run we add random 
noise to the acceleration, gyroscope, and camera measurements. 
To make the simulation realistic, the statistics about measurement noise are 
identical to the ones used in the real tests of~\prettyref{sec:realTests}.
% , with the only difference that we do not
In each run, the robot performs VIN and, at each camera frames, it selects $\kappa = 20$
visual features out of all the features in the field of view. 
We compare three feature selection strategies. 
The greedy selection resulting from~\prettyref{alg:greedy}
with the eigenvalue objective $\feigMin$ (label: \minEig), 
the greedy  selection with the log-determinant cost $\flogDet$ (label: \logDet), 
and a random selection which randomly draws $\kappa$ of the available features (label: \random).

\prettyref{fig:monteCarlo}(c)-(d) show the absolute translation  and absolute rotation errors, 
averaged over 50 Monte Carlo runs. From the figure is clear that a clever selection of the 
features, resulting from \logDet or \minEig, deeply impacts performance in VIN. 
Our techniques largely improve estimation errors, when compared against the \random selection; both approaches 
result in similar performance. From the figure we note that the absolute errors have some oscillations:
this is a consequence of the fact that the trajectory is circular;  in general, this 
stresses the fact that absolute metrics may be poor indicators of performance in visual-inertial odometry. 
In this case, the relative error metrics confirm the results of~\prettyref{fig:monteCarlo}(c)-(d): 
the average translation and rotation errors are given in~\prettyref{tab:relativeErrorsSim}; 
in parenthesis we report the 
error reduction percentage with respect to the random baseline. % is given in parenthesis.

\input{tableRelativeErrorsSim}

 \prettyref{fig:monteCarlo}(b) reports the CPU time required for feature selection.
 The figure considers both cost functions (\logDet and \minEig) and compares timing when using 
 our lazy evaluation, as described in \prettyref{alg:greedy}, against a naive implementation of the greedy algorithm that 
 always tests the marginal gain of every feature (i.e., for which the stopping condition in line~\ref{line:break1} 
 of \prettyref{alg:greedy}
 is disabled). 
 The naive greedy (without lazy evaluation) always results in $\kappa N$ objective evaluations. 
 When using lazy evaluation, the number of objective evaluation depends on the tightness 
 of the upper bounds used in~\prettyref{alg:greedy}.
 From   \prettyref{fig:monteCarlo}(b), we see that the advantage of using the lazy evaluation is marginal 
 for the log-determinant cost; this is not surprising, since the Hadamard's inequality of
 \prettyref{prop:boundLogDet} usually gives a fairly loose bound. On the other hard, the 
 advantage of using the lazy evaluation is significant for the \minEig, resulting 
 in a reduction of the computational time of $20\%$.
 The average CPU time required by \prettyref{alg:greedy} (with lazy evaluation) to select $\kappa = 20$
 features is $0.069\sec$ for \logDet and $0.195\sec$ for \minEig.
 While these timing may be already acceptable for applications, there are large 
 margins to speed up computation: we postpone these considerations to~\prettyref{sec:conclusion}.

%%%%%%%%%%%%%%%%%%%%%%%%%%%%%%%%%%%%%%%%%%%%%%%%%%%%%%%%%%%%%%%%%%%%%%%%%%%%%%%%%%%%%%%%%%
\subsection{Real Tests: Agile Navigation on Micro Aerial Vehicles}
\label{sec:realTests}
In this section we show that our feature selection approach enhances VIN performance 
in real-world navigation problems with micro-aerial vehicles (MAVs).

%% =================================================
\mysubparagraph{Testing setup} 
We use the \emph{EuRoC} benchmark~\cite{Burri16ijrr-eurocDataset}
for our evaluation. 
The EuRoC datasets are collected with an AscTech Firefly hex-rotor helicopter equipped with a VI
 (stereo) visual-inertial sensor. The camera records stereo images at resolution $752\times 480$ 
 and framerate $20\Hz$; IMU data is collected at $200\Hz$.
We refer to~\cite{Burri16ijrr-eurocDataset} for a technical description of the datasets.
In this context we only remark that the datasets %provides stereo images and IMU data, 
 contain test instances at increasing levels of complexity, collected in a machine 
hall environment and in a smaller Vicon room.
In our tests, 
the measurement variances, as well as the intrinsic and extrinsic calibration parameters 
 match exactly the one specified in the dataset.  
 The most relevant parameters used in our tests are given in~\prettyref{tab:parameters}; 
 in the front-end we used openCV's \emph{goodFeaturesToTrack} for feature detection 
 and the Lucas-Kanade method for feature tracking;
 as input to the detector we specify a minimum quality level for the features and a desired 
 number of features to extract ($N$).
 From these $N$ features our selector has to retain $\kappa = 10$ features that will be used by the back-end.
 In this sense, feature detection and tracking at the front-end are pre-attentive mechanisms: they 
 work on a large set of features, which are later filtered out by our feature selector. 
 The feature selector uses a predictive horizon of 3\sec; 
 in practice, the future pose estimates along the horizon can be computed from the control inputs,
 by integrating the dynamics of the vehicle (\prettyref{sec:estimationModel}).
 Since the control inputs are not available in the EuRoC dataset, we compute the 
 future poses by attaching ground truth motion increments to the current pose estimate. 
 The only assumption in doing so is that the control loop and the estimation quality are 
 good enough to track a desired set of future poses; this is the case in VIN in which the short-term drift is small.
 %Results are averaged over 5 runs. 

\input{tableParameters}
%% =================================================
\mysubparagraph{Techniques} 
We compare four VIN approaches. The first two VIN approaches use the \minEig and the \logDet selectors 
proposed in this paper.
The third approach uses a selector that picks the $\kappa$ features with highest quality 
(i.e., highest score in \emph{goodFeaturesToTrack}). This selector is commonly used in VIN and 
only accounts for the appearance of the visual features; we denote it with the 
label ``\quality'', following openCV's terminology. 
The fourth technique is a VIN approach using 200 features (selected as the ones with largest score in \emph{goodFeaturesToTrack}) and is used to have a reference 
performance for the case in which the VIN system has less stringent computational constraints (label: \noselection).

In order to compute the tracking probabilities $p_\l$, we modified openCV's \emph{goodFeaturesToTrack} in order to 
have access to the features' scores. Then, we mapped the scores to probabilities in $[0,1]$, such that 
% normalized the scores in $[0,1]$ and
%  we set the probabilities to be equal to the normalized scores; as a consequence, 
 more distinguishable features 
 have higher tracking probabilities $p_\l$.
 % In most datasets we used $0.2$ as minimum quality level in \emph{goodFeaturesToTrack}. 
 % For the datasets labeled as ``difficult'', the presence of 
 % dark images and and motion blur makes feature detection particularly challenging; for these 
 % two datasets we set the minimum quality to $0.01$ in order to have a sufficient number of features.
 %\input{tableRealResults}

 \input{figAccuracyNoSuccProb}

 \input{figSnapshots}

%% =================================================
\mysubparagraph{\edited{Results: Accuracy}}
\prettyref{fig:accuracyNoSuccProb} shows the performance of the compared techniques on all the 11 EuRoC datasets.
The EuRoC benchmark includes datasets  of different levels of complexity, with the difficult datasets 
being challenging for standard VIN pipeline due to the fast motion of the MAV.
In this section we show that we can obtain accurate position estimation with as few as $\kappa=10$ features; 
this \editedTwo{budget} is enforced for each frame; for instance, if we are tracking $r$ features from the 
previous frame, then in the current frame we can only retain $\kappa-r$ features. 

\prettyref{fig:accuracyNoSuccProb}(a) compares the VIN performance using the relative 
translation errors as metric. The figure confirms that the difficult datasets tend to have larger translation 
errors. Moreover, it shows that the proposed techniques, \minEig and \logDet, lead to smallest errors 
compared to the baseline \quality. Clearly, the technique \noselection, which uses 20x more features, leads to the 
smaller errors. To better appreciate the advantage of  \minEig and \logDet with respect to \quality, 
\prettyref{fig:accuracyNoSuccProb}(b) shows the relative improvement, i.e., the relative translation error reduction, of the two techniques with respect to \quality. The figure shows  
% We also report the absolute translation errors for comparison, 
% while we already observed that absolute errors are not a good metric in odometry problems.
% Besides the errors, for the \minEig and \logDet, we also report the error reduction (in percent) 
% with respect to the baseline \quality. 
%The table shows 
that the proposed feature selectors result in much smaller drift across 
 all but one datasets. 
 %In the easy datasets (e.g., V1_01_easy) the advantage is negligible, 
%while in all the others the gain in accuracy is substantial. 
The average error reduction is larger than $20\%$ and overcomes $40\%$ in the datasets MH_02_easy, MH_05_difficult, and V2_01_easy.
In particular, in the dataset MH_05_difficult the estimate resulting from the \quality-based feature selection 
diverged after a sharp turn, while our techniques were able to ensure accurate pose estimation.
% (the gain is smaller for \logDet in MH_03_medium, while is significantly larger in MH_05_difficult).
% In the last dataset, V1_02_medium, 
% in the more difficult datasets MH_05_difficult and V1_02_medium the use of our feature 
% selection mechanism makes the difference between accurate pose estimate and complete divergence: 
% the \quality estimator (limited to 50 features) is not able to keep track in these instances 
% and leads to trajectory estimates that are far from the ground truth.
% In a single dataset (MH_04_difficult) \minEig resulted in larger errors, while \logDet consistently
%  outperformed the baseline. 
% \bit
% \item table: stats (ave/max error, ave/max timing, feature sel timing) for each dataset
% \item accuracy: figure: accuracy of different techniques for 20-40-60 features max
% \item timing: figure: feature selection timing, compared to feature detection and optimization time, for increasing features
% \eit
The dataset V1_03_difficult is the only one in which the proposed techniques have slightly worse performance. 
We noticed that in datasets with severe motion blur 
the advantage of the proposed techniques may vary,
% the \quality-based selection is still competitive, 
and this is 
due to the fact that we are using a simplistic model for the tracking probabilities $p_\l$.
For completeness, \prettyref{fig:accuracyNoSuccProb}(c) reports the absolute translation error as a percentage of the 
trajectory traveled; this is another common metric for VIN. We notice that \noselection has excellent performance, 
while using 200 features (average error accumulation is $0.17\%$ of the trajectory length). 
\edited{The approach \noselection exhibits similar or smaller estimation errors 
with respect to related techniques benchmarked on the EuRoC datasets in~\editedTwo{\cite{Sun18ral-UAVsystem}}.}\footnote{\edited{The interested reader can compare Fig.~2(a)
 in~\editedTwo{\cite{Sun18ral-UAVsystem}}
with \prettyref{fig:accuracyNoSuccProb}(c) in the present paper, noting that we report 
errors as percentage of the distance traveled and 10 out of 11 of the EuRoC datasets are less than 100m long~\cite{Burri16ijrr-eurocDataset}; furthermore, 
we notice that~\editedTwo{\cite{Sun18ral-UAVsystem}} reported systematic failures on 2 of the datasets, while we are able to successfully complete all the datasets.}}
Moreover, the proposed techniques, \logDet and \minEig,
are able to ensure an average error accumulation of $0.42\%$ and $0.46\%$, respectively, while using only 10 features!

To get a better intuition behind the large performance boost induced by the proposed techniques, 
we report few snapshots produced by our pipeline in~\prettyref{fig:snapshots}. 
Each sub-figure shows, for the current frame, the tracked features (green squares with the optical flow 
vector), the available features (red crosses), and the features selected (yellow circles) by 
(a) \quality, (b) \logDet, and (c) \minEig. The frames are captured during a sharp left turn from the
 MH_03_medium dataset. The \quality selector simply picks the most distinguishable features, resulting in 
  many features selected on the right-hand side of the image; these features are of scarce utility: they 
 will soon disappear from the field of view due to the motion of the MAV. 
 On the other hand, \logDet and \minEig are predictive and they leverage the knowledge of the immediate 
 motion of the platform; therefore they tend to discard features 
that fall outside the field of view and select features on the left-hand side of the image.
 % In general, the proposed feature selector are able to produce longer feature tracks and 
 % are better at tackling fast motion. 

%%%%%%%%%%%%%%%%%%%%%%%%%%%%%%%%%%%%%%%%%%%%%%%%%%%%%%%%%%%%%%%%%%%%%%%%%%%%%%%%%%%%%%%%%%
% \subsection{Timing and Trade-offs}
% \label{sec:discussion}

\edited{
%% =================================================
\mysubparagraph{Results: Timing and Trade-offs}
\prettyref{fig:CPUtime} reports the average CPU time required by the VIN back-end for all techniques and datasets. 
% For the datasets MH_04_difficult and MH_05_difficult the CPU times are larger: this 
% is due to the fact that in these datasets we use a lower quality level for feature detection (0.01), 
% resulting in larger number 
% of features fed to the selector. %; this fact does not significantly impact $\logDet$.
For the proposed techniques, the back-end time includes both the CPU time spent on feature selection 
and the time spent for estimation in VIN (factor graph optimization with iSAM2).
The figure shows that \logDet is able to reduce the back-end time by $30-40\%$ in most scenarios, with respect to \noselection;
 in particular, the average time for \noselection is $57$ms while the average time for \logDet is $35$ms.
%, which does not perform feature selection.
The CPU time of \quality is even smaller, at the cost of degraded performance (\prettyref{fig:accuracyNoSuccProb}). 
\edited{Consistent} with the Monte Carlo analysis, in our current implementation
 \logDet is faster  than \minEig, which implies a back-end time larger than $100$ms in our tests.

\input{figCPUTtimeNoSuccProb}

\input{figLogDetDetails-timing}

\input{figLogDetDetails-tradeoff}

\prettyref{fig:logDetDetails-timing} provides a more detailed breakdown of the back-end time
for increasing number of selected features and for each of the EuRoC datasets. 
The figure focuses on the \logDet approach, which currently has been already shown to be remarkably faster than the \minEig approach. 
The blue portion of each bar reports the time spent on feature selection, while the red portion corresponds 
to the time spent on factor graph optimization in iSAM2; the sum of these two times corresponds to the overall back-end 
time reported in \prettyref{fig:CPUtime}. 
The figure compares the timing results for the proposed approach versus \noselection for increasing number of selected features 
(in all cases, \noselection retains all the tracked features). 
The figure provides few useful insights.
First, the computational advantage in using the proposed attention mechanism is more noticeable when the number of selected features is small. In particular, when selecting 5-10 features we 
typically have a large computational saving by using the proposed approach. However, when the number of features approaches the overall number of 
available measurements (100 in our tests), the computational advantage can be negligible.
%\noselection, which avoids the feature selection altogether ensures faster estimation. 
This is due to the 
fact that our current implementation of the feature selection is relatively slow and indeed its computational cost 
is often comparable with the cost of running iSAM2. In \edited{Section~\ref{sec:conclusion}} we discuss extensions that 
have the potential of making the selection time negligible.
The second insight is that the computational advantage of the proposed approach is more evident in easy datasets, while in some of the difficult datasets 
(e.g., V12_med, V13_hard, V22_med, V23_hard) the computational advantage becomes marginal or inexistent.
% there is no computational advantage in using the proposed approach. 
This mismatch results from the fact that in the easy datasets the front-end is able to track many features (typically all the features that we set as upper bound) which implies that the \noselection approach requires more time to perform estimation. In the more difficult scenarios, instead, the front-end typically tracks less features hence resulting in faster estimation (and degraded performance). 
Note that while in the difficult scenarios the computational advantage of 
our approach may be limited, the accuracy boost resulting from our approach still suggests its use over \quality.

\prettyref{fig:logDetDetails-tradeoff} provides the computation/performance trade-off for increasing number of features and for each 
of the EuRoC datasets, using the \logDet selector.
The back-end time, shown as a dashed red line, corresponds to the sum of the feature selection time and the iSAM2 estimation time. As expected the 
time increases with the number of selected features.
The figure also reports the absolute translation error as a percentage of the 
trajectory traveled, shown as a solid blue  line. The error typically decreases when increasing the number of selected features. 
Results are averaged over 5 runs performed on each dataset.
The trade-off plots in 
\prettyref{fig:logDetDetails-tradeoff} can be used to decide the number of features to use to attain a desired level of accuracy or 
given an upper bound on the time that can be spent performing estimation at the back-end.
}

%% file: figMatlabSim.tex
%!TEX root = main.tex

\begin{figure}[h]
\hspace{-0.5cm}
\begin{minipage}{\textwidth}
\begin{tabular}{cc}%
\begin{minipage}{4cm}%
\centering%
\includegraphics[width=1.14\columnwidth]{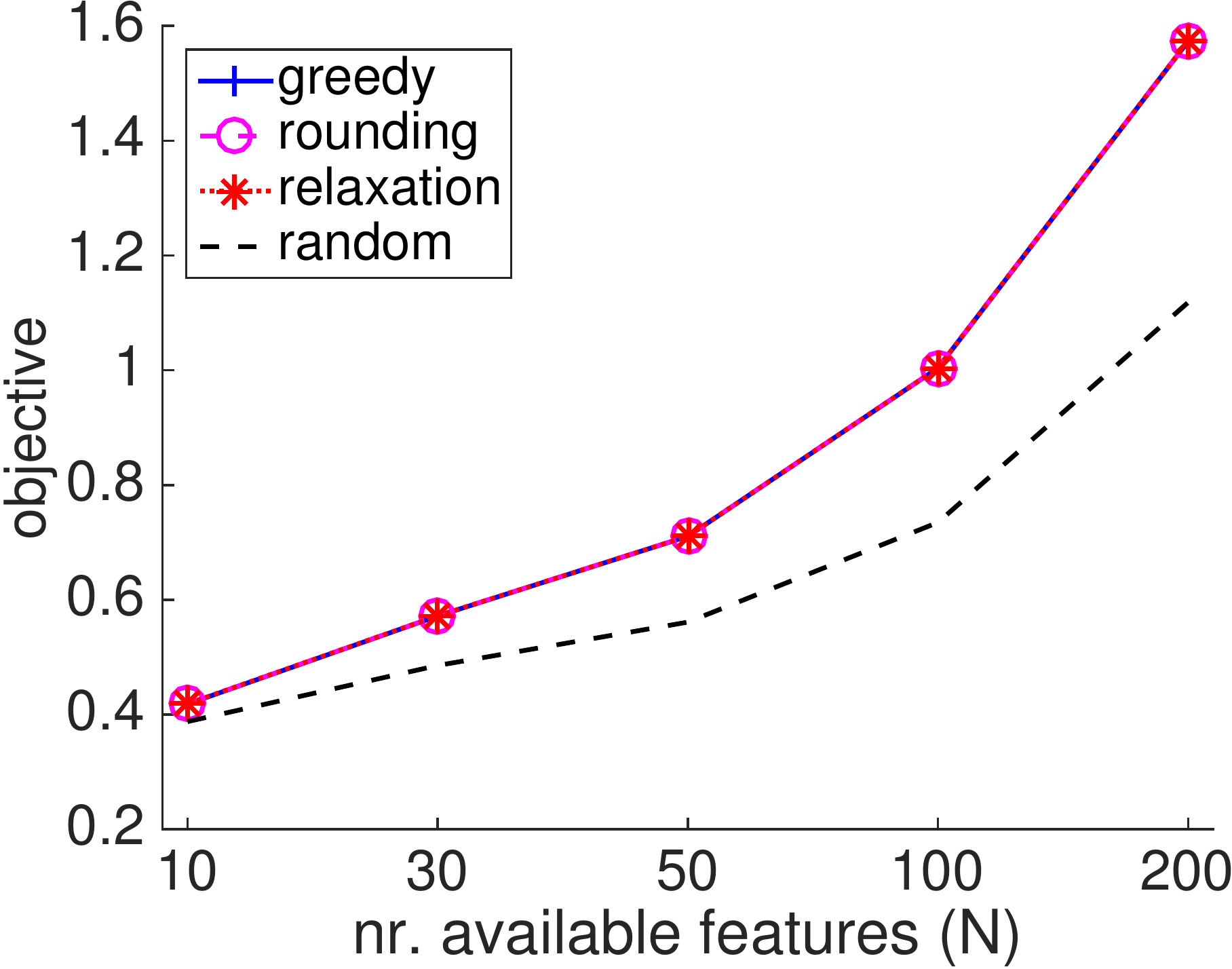} \\
\hspace{1.5cm}(a) $\feigMin$
\end{minipage}
& \hspace{0.05cm}
\begin{minipage}{4cm}%
\centering%
\includegraphics[width=1.14\columnwidth]{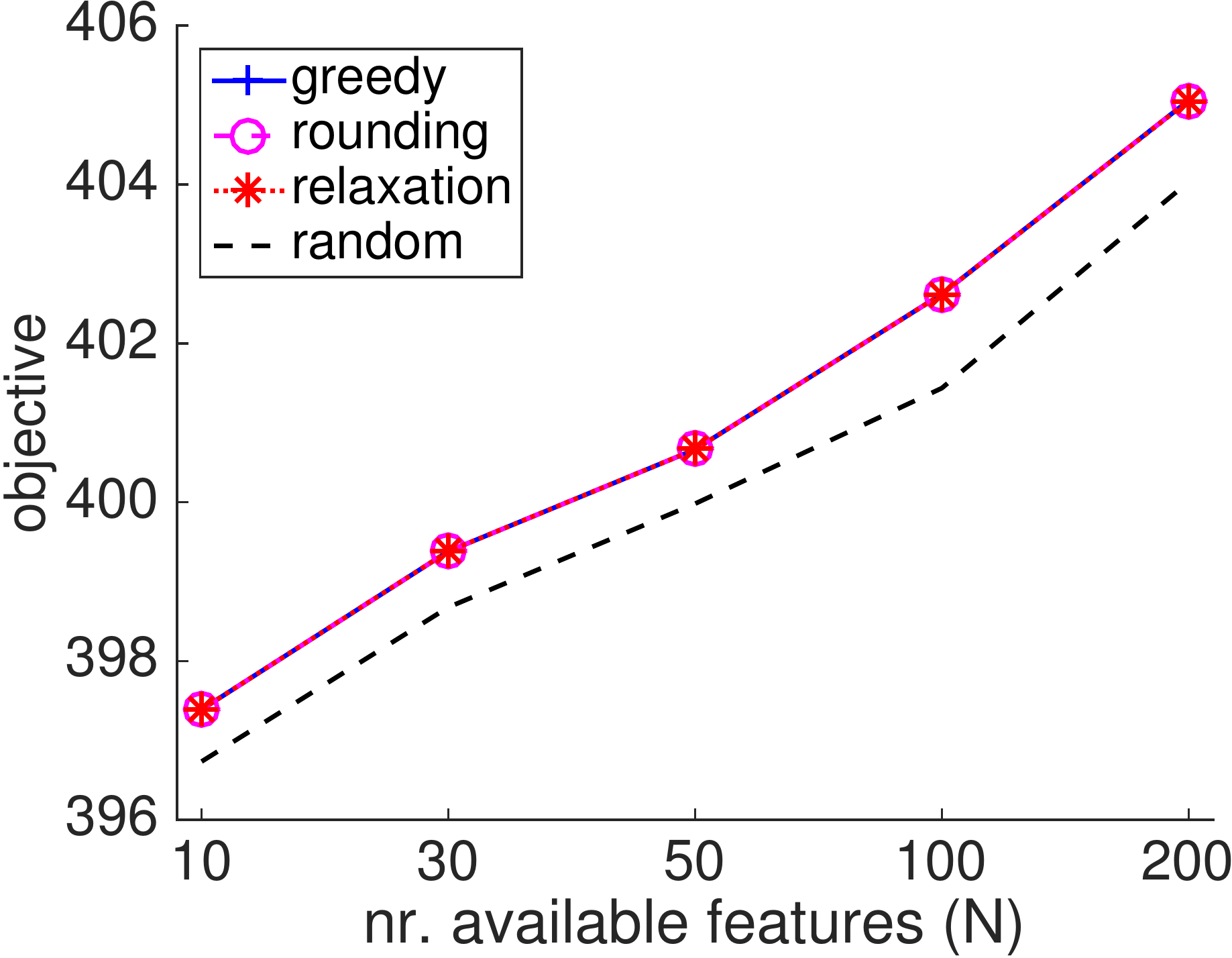} \\
\hspace{1.5cm}(b) $\flogDet$
\end{minipage}
\end{tabular}
\end{minipage}%
\caption{\label{fig:matlabSim}
Techniques to approximately solve problem~\eqref{eq:featureSelection} for 
(a) the smallest eigenvalue objective $\feigMin$, and (b) 
the log-determinant objective $\flogDet$.
The figure reports the objective attained by the greedy algorithm (\greedy), the 
rounded solution (\rounded), and a random selection (\random).
The  upper bound \relaxed, attained 
by the convex problem~\eqref{eq:convexRelaxation} (before rounding), 
% and the objective value when no features are selected $f(\barMOmega_{\k:\k+\hor})$ 
is shown for comparison.
}
\end{figure}

%% file: figMonteCarloResults.tex
%!TEX root = main.tex

\newcommand{\mpw}{4.2cm}
\begin{figure}[h]
\begin{minipage}{\textwidth}
\begin{tabular}{cc}%
\begin{minipage}{\mpw}%
\centering
\includegraphics[width=1.2\columnwidth, trim = 20mm 0mm 3mm 0mm, clip]{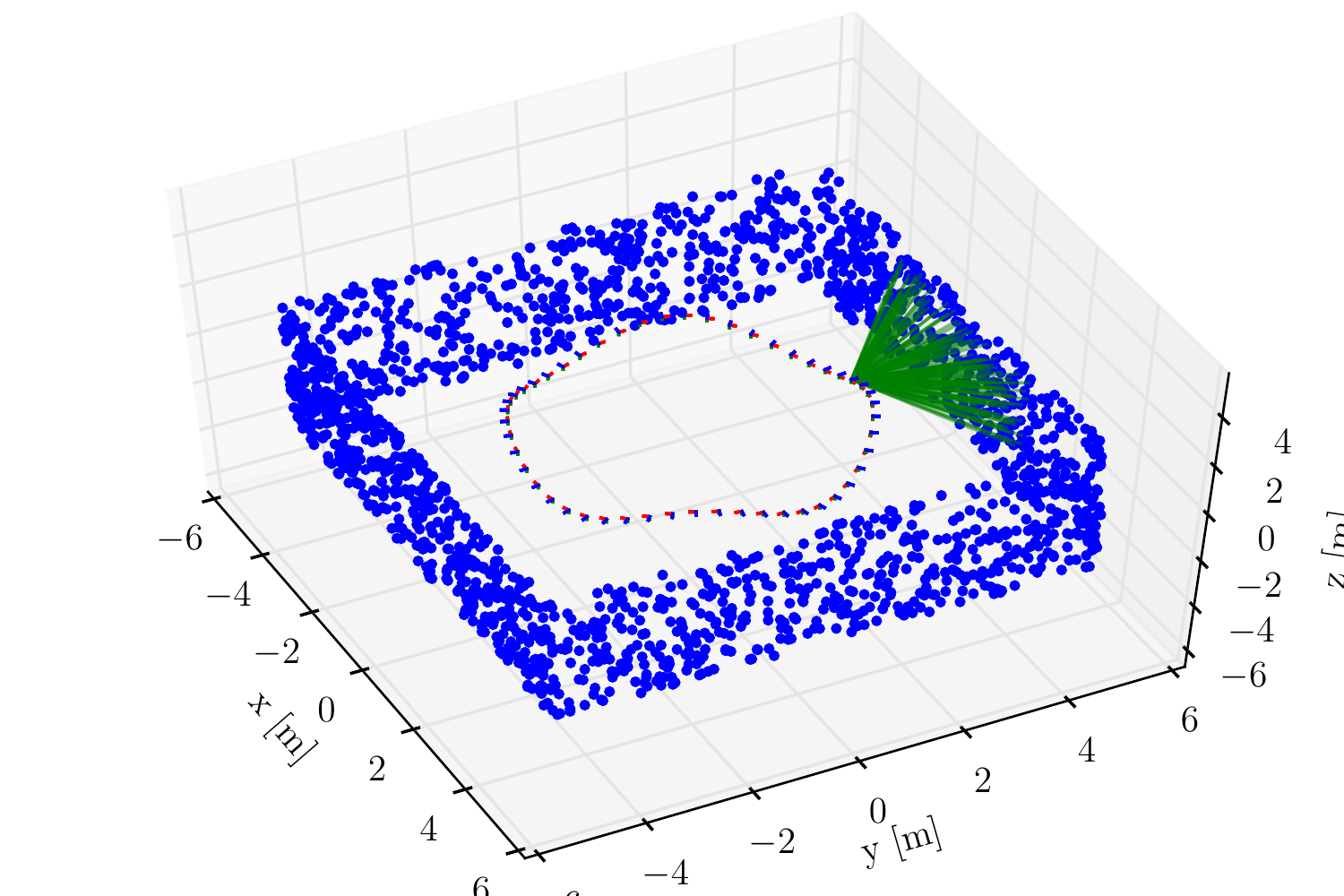} \\
(a)
\end{minipage}
& 
% \begin{minipage}{\mpw}%
% \centering%
% \includegraphics[width=\columnwidth]{figures/results_monteCarlo_timing} \\
% (b)
% \end{minipage}
\hspace{0.2cm}
\begin{minipage}{\mpw}%
\centering
\begin{tabular}{ | c | c |}%
\hline
 &  Time [\sec]\\ 
 \hline
\minEig naive & 0.238\\
\minEig lazy  & 0.195 \\
\logDet naive & 0.070 \\
\logDet lazy  & 0.069\\
\hline
\end{tabular}  \vspace{0.5cm}\\ (b) 
\end{minipage}
\\
\begin{minipage}{\mpw}%
\centering%
\hspace{-1.7cm}
\includegraphics[width=1.15\columnwidth]{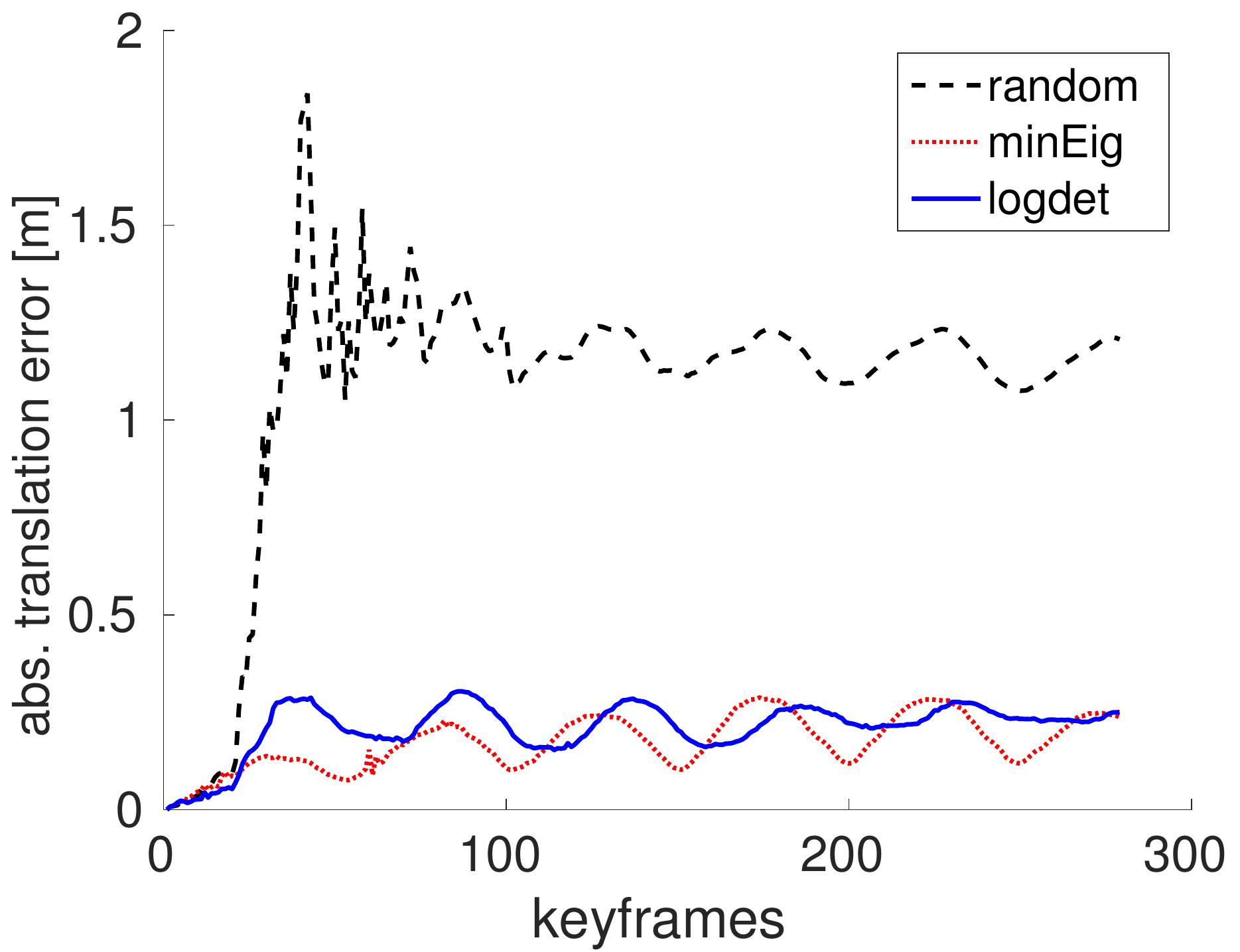} \\
\hspace{-1.5cm}(c)
\end{minipage}
\hspace{-1.2cm} 
&
%\hspace{-1.2cm} 
\begin{minipage}{\mpw}%
\centering%
\hspace{-1.0cm}
\includegraphics[width=1.15\columnwidth]{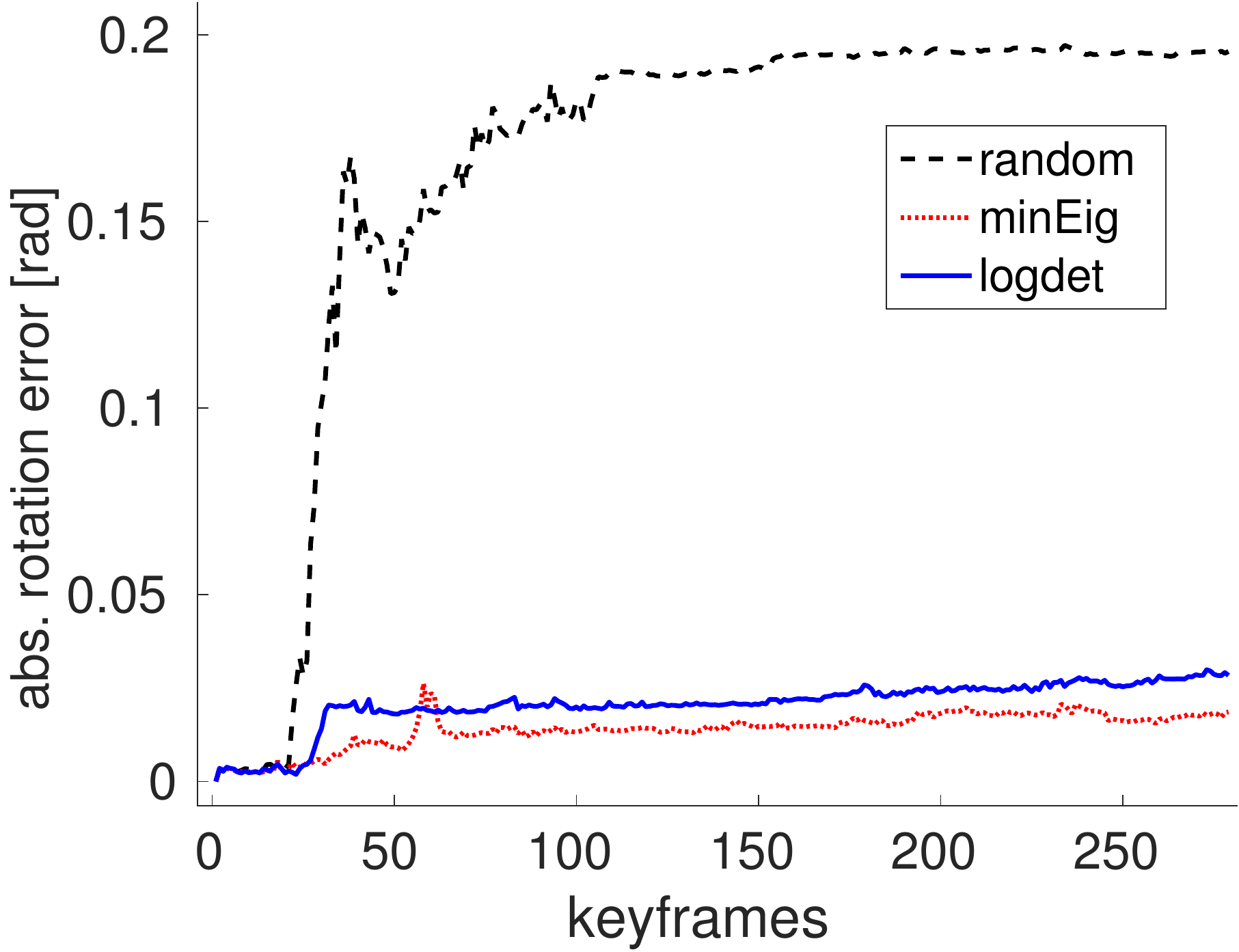} \\
\hspace{0.2cm}(d)
\end{minipage}
\end{tabular}
\end{minipage}%
\caption{\label{fig:monteCarlo}
Simulation results: (a) simulated environment,
(b) table with CPU times for different implementations of the greedy algorithms,
(c) absolute translation errors, (d) absolute rotation errors.}
\end{figure}

%% file: tableRelativeErrorsSim.tex
%!TEX root = main.tex

\begin{table}[h!]
\centering
\vspace{-2mm}
\begin{tabular}{ c | l | c }
Technique & Rel. Translation Error [\meters] &  Rel. Rotation Error [rad]  \\ 
\hline       
\random & \hspace{1cm} 0.0103 			    & \hspace{1cm}  0.0049 \\
\hline 
\minEig & \hspace{1cm}0.0064 \blue{(-37\%)} & \hspace{1cm}  0.0025 \blue{(-48\%)} \\
\hline 
\logDet & \hspace{1cm}0.0053 \blue{(-48\%)} & \hspace{1cm}  0.0018 \blue{(-63\%)} \\
\hline 
\end{tabular}
\caption{Relative translation and rotation errors for the simulated tests of~\prettyref{sec:monteCarlo} 
(average over 50 Monte Carlo runs) \label{tab:relativeErrorsSim}}% 
\vspace{-3mm}
\end{table}

%% file: tableParameters.tex
%!TEX root = main.tex

\begin{table}[h!]
\centering
\vspace{-2mm}
\begin{tabular}{ | c | c | c | }
\hline  
& Parameter name &  Value  \\ 
\hline       
\multirow{3}{*}{Front-end} & Nr. features to detect ($N$) & 100\\ 
						   & Minimum quality level &  0.001  \\ 
						   & Time between keyframes & 0.2\sec\\
						   \hline 
\multirow{2}{*}{Back-end} &  Smoothing window & 6\sec \\
						  &  iSAM2 iterations &  1  \\ 
						  \hline 
\multirow{2}{*}{Feature selector} 
								  &  Nr. features to select ($\kappa$) &  10  \\ 
								  &  Horizon &  3\sec  \\ 
						  	      %&  Default depth &  5\meters \;- 2\meters 	 \\					  
						  		% //  maxFeatureAge_(20),  
						  		% featureSelectionCosineNeighborhood_(cos( (10*PI)/(180.0) )), // 10 degrees
\hline 
\end{tabular}
\caption{VIN and feature selection parameters \label{tab:parameters}}% 
\vspace{-5mm}
\end{table}

%% file: figAccuracyNoSuccProb.tex
%!TEX root = main.tex

\newcommand{\mys}{5.2cm}

% \begin{figure*}[h]
% \hspace{-0.5cm}
% \begin{minipage}{\textwidth}
% \begin{tabular}{ccc}%
% \begin{minipage}{\mys}%
% \centering%
% \includegraphics[width=1.14\columnwidth, trim={0cm 5.5cm 0 5.5cm},clip]{T1-avePercTranErrOverTrajLen_vio_bars-eps-notes3} \\
% \hspace{1.5cm}(a)
% \end{minipage}
% & \hspace{0.5cm}
% \begin{minipage}{\mys}%
% \centering%
% \includegraphics[width=1.14\columnwidth]{T1-aveRelTranErrors_vio_precentWrtQuality_bars} \\
% \hspace{1.5cm}(b) 
% \end{minipage}
% & \hspace{0.5cm}
% \begin{minipage}{\mys}%
% \centering%
% \includegraphics[width=1.14\columnwidth]{T1-vioAndSelectorTime} \\
% \hspace{1.5cm}(c) 
% \end{minipage}
% \end{tabular}
% \end{minipage}%
% \caption{\label{fig:accuracyNoSuccProb}
% Accuracy and CPU time for the compared techniques on the 11 EuRoC MAV datasets.
% (a) Translation error as percentage of the overall trajectory length;
% (b) Relative improvement (translation error reduction) of the proposed techniques with respect to the \quality baseline;
% (c) CPU time for the back-end (including feature selection) for the proposed techniques.
% }
% \end{figure*}

\begin{figure*}[h]
\hspace{-0.5cm}
\begin{minipage}{\textwidth}
\begin{tabular}{ccc}%
\begin{minipage}{\mys}%
\centering%
\includegraphics[width=1.14\columnwidth, trim={0cm 0.5cm 4cm 13cm},clip]{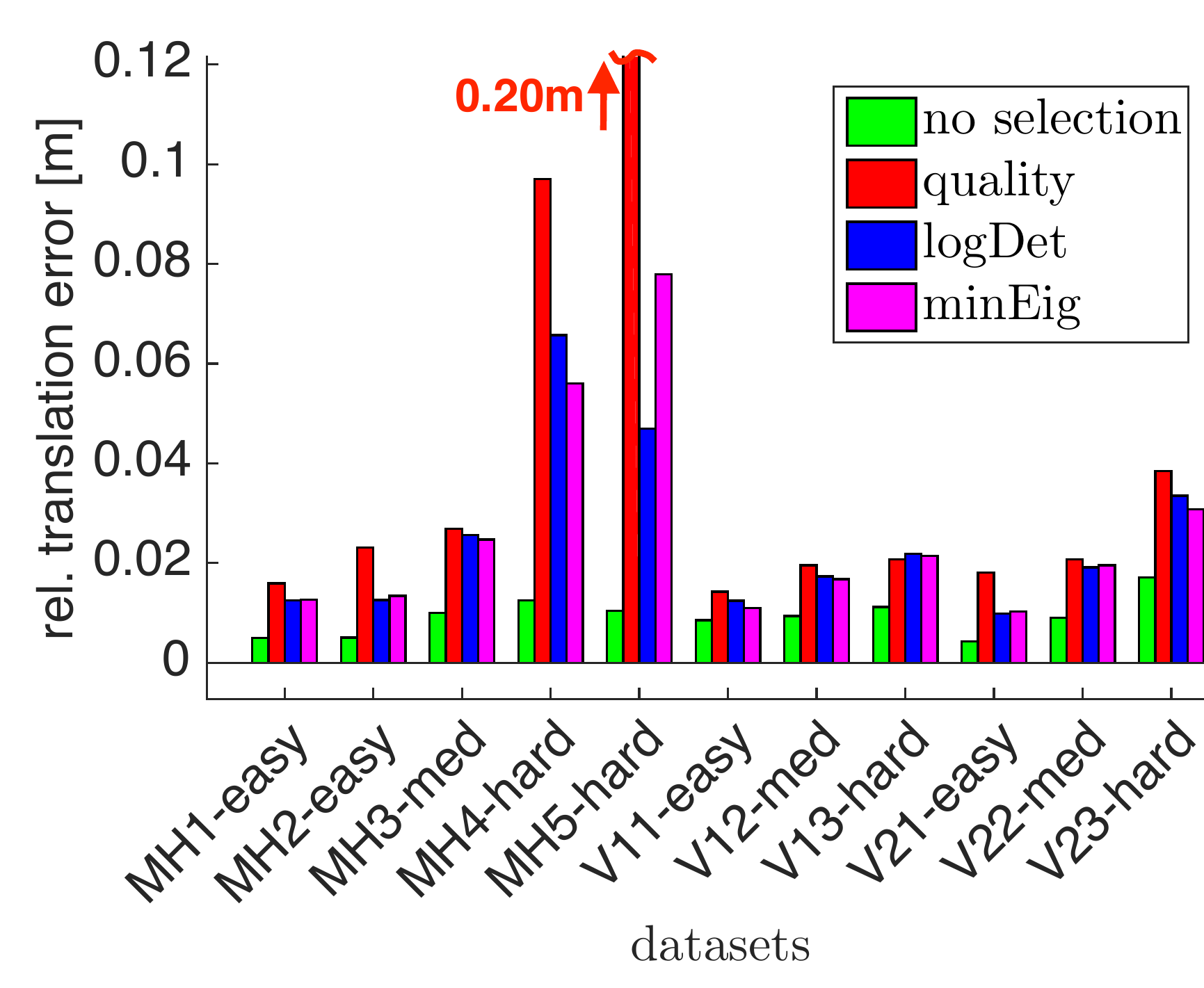} \vspace{-0.7cm}\\
\hspace{-2cm}(a)
\end{minipage}
& \hspace{0.5cm}
\begin{minipage}{\mys}%
\centering%
\includegraphics[width=1.14\columnwidth]{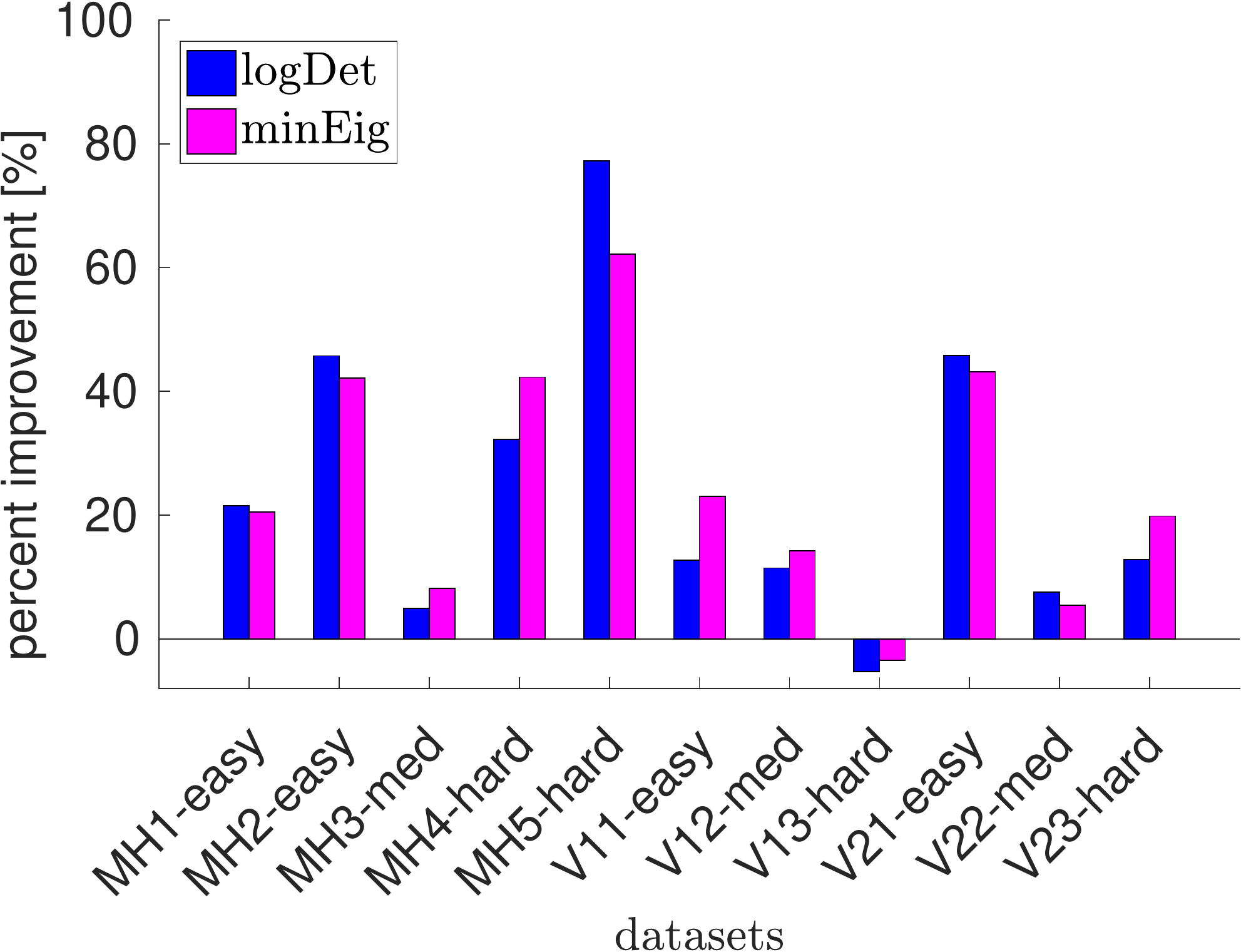} \vspace{-0.7cm} \\ 
\hspace{-2cm}(b) 
\end{minipage}
& \hspace{0.5cm}
\begin{minipage}{\mys}%
\centering%
\includegraphics[width=1.14\columnwidth, trim={0cm 5.5cm 0 5.5cm},clip]{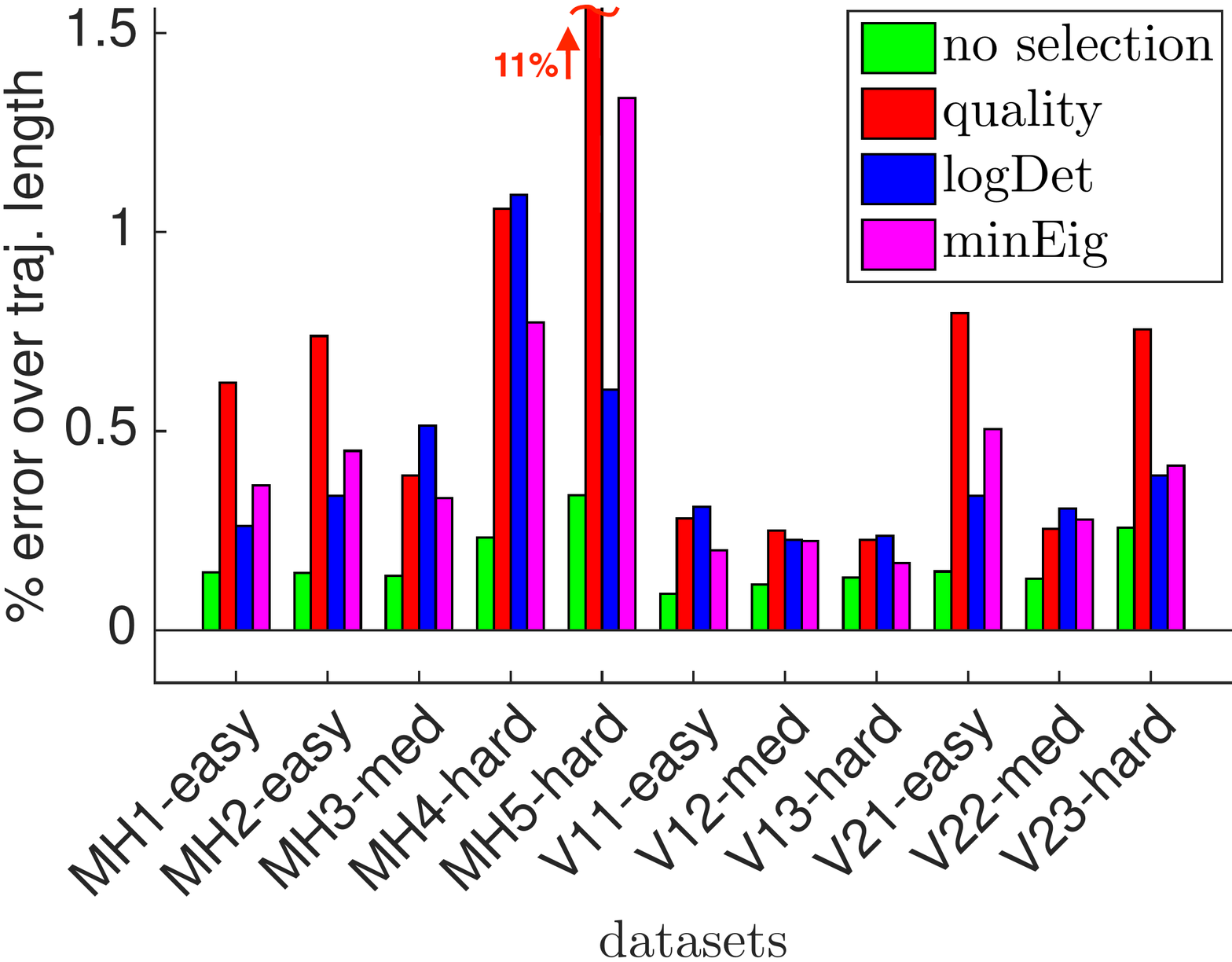} \vspace{-0.7cm} \\
\hspace{-2cm}(c) 
\end{minipage}
\end{tabular} 
\end{minipage}%
\caption{\label{fig:accuracyNoSuccProb}
Accuracy for the compared techniques on the 11 EuRoC MAV datasets.
(a) Relative translation error; 
(b) Relative improvement (relative translation error reduction) of the proposed techniques with respect to the \quality baseline;
(c) Translation error as percentage of the overall trajectory length.
%(c) CPU time for the back-end (including feature selection) for the proposed techniques.
}
\end{figure*}

%% file: figSnapshots.tex
%!TEX root = main.tex

\renewcommand{\mpw}{5.7cm}
\begin{figure*}[h]
\begin{minipage}{\textwidth}
\begin{tabular}{cccc}%
\begin{minipage}{\mpw}%
\centering
\includegraphics[width=1.06\columnwidth]{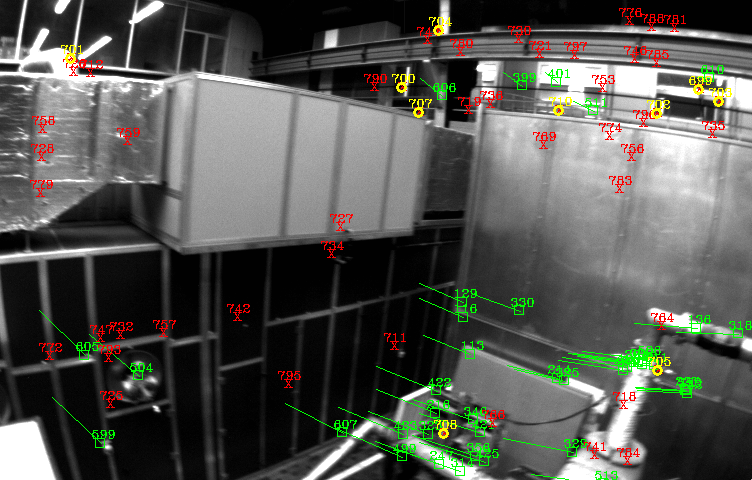} \\
(a) \quality
\end{minipage}
& 
\begin{minipage}{\mpw}%
\centering%
\includegraphics[width=1.06\columnwidth]{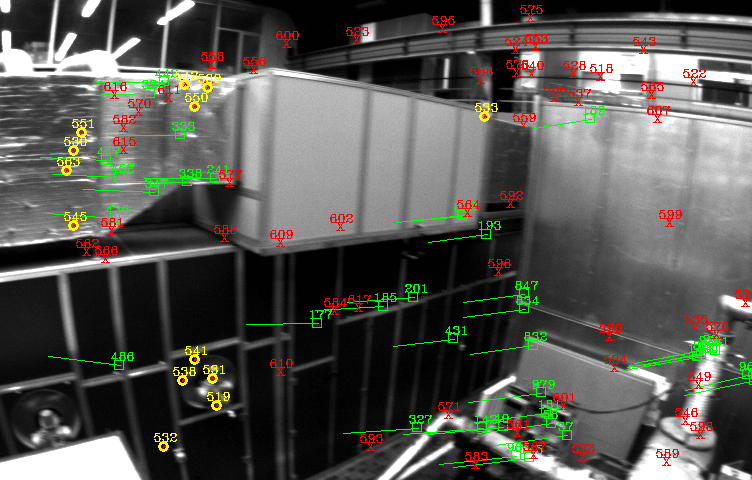} \\
(b) \minEig
\end{minipage}
&
\begin{minipage}{\mpw}%
\centering%
\includegraphics[width=1.06\columnwidth]{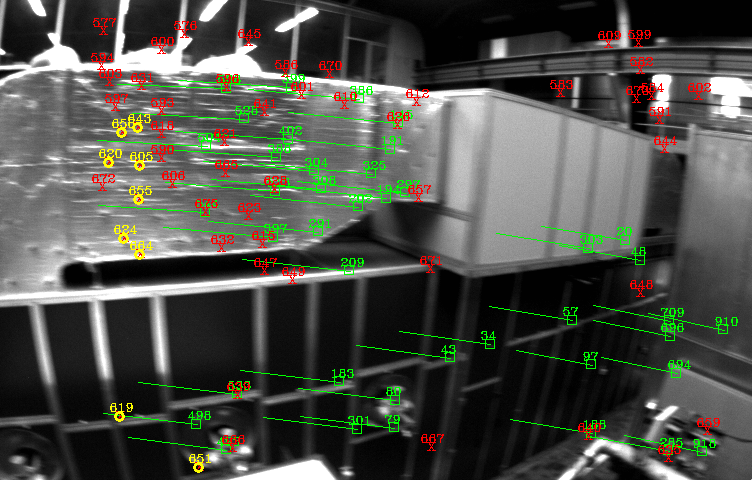} \\
(c) \logDet
\end{minipage}
% & 
% \begin{minipage}{\mpw}%
% \centering%
% \includegraphics[width=\columnwidth]{figures/img_382-logdet} \\
% (d)
%\end{minipage}
\end{tabular}
\end{minipage}%
\caption{\label{fig:snapshots}
Snapshots of the feature selection performed by the techniques \quality, \minEig, and \logDet during a sharp
left turn. Features tracked from previous frames are shown as green squares (with the 
corresponding optical flow 
vectors), the newly detected features are shown as red crosses, and the selected features are shown as yellow circles.
We note that \quality only selects the features from their appearance, and chooses many features on the 
right-hand side of the frames: these features will soon fall out of the field of view due to the sharp turn.
}
\end{figure*}

%% file: figCPUTtimeNoSuccProb.tex
%!TEX root = main.tex

\begin{figure}[h]
\centering%
\begin{minipage}{\columnwidth}%
\centering%
\includegraphics[width=0.65\columnwidth]{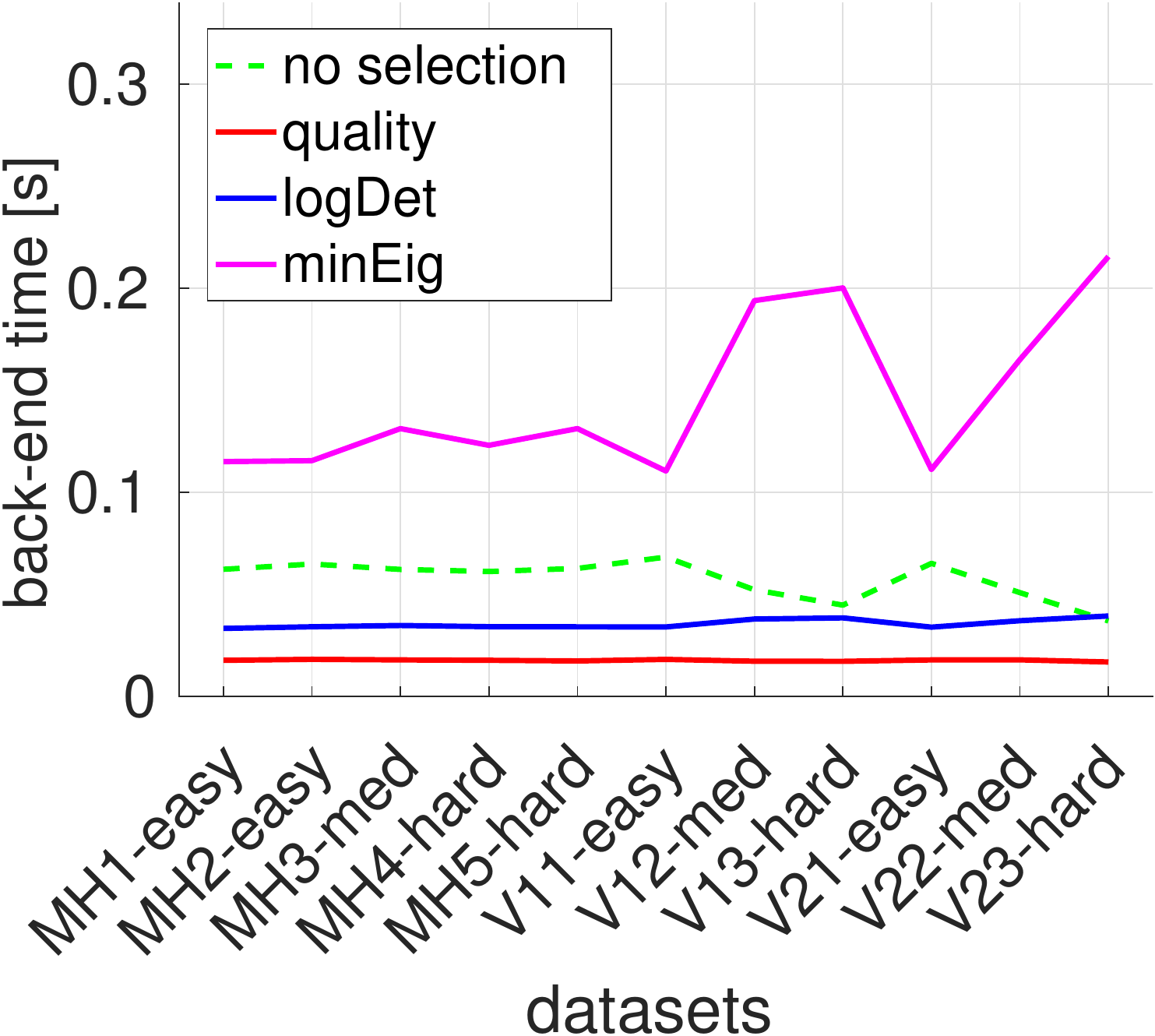} \\
\end{minipage}%
\caption{\label{fig:CPUtime}
CPU time for the back-end (including feature selection) for the compared techniques on the EuRoC datasets.
}
\end{figure}

%% file: figLogDetDetails-timing.tex
%!TEX root = main.tex

\renewcommand{\mpw}{2.6cm}
\newcommand{\mww}{1.1}
\begin{figure*}[h]
\begin{minipage}{\textwidth}
\begin{tabular}{cccccc}%
\begin{minipage}{\mpw}%
\centering
\includegraphics[width=0.9\columnwidth]{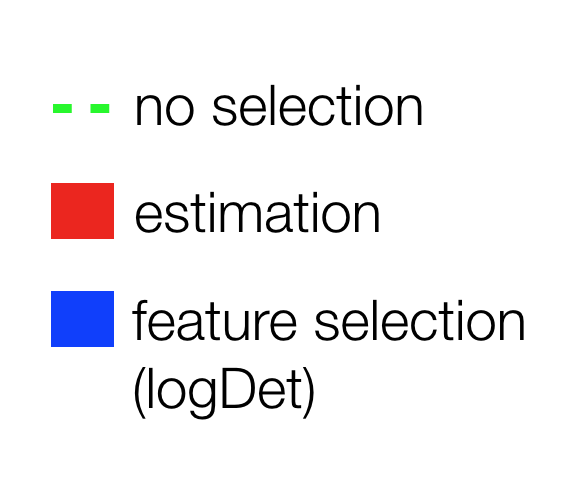} \\
\end{minipage}
&
\begin{minipage}{\mpw}%
\centering
\includegraphics[width=\mww\columnwidth]{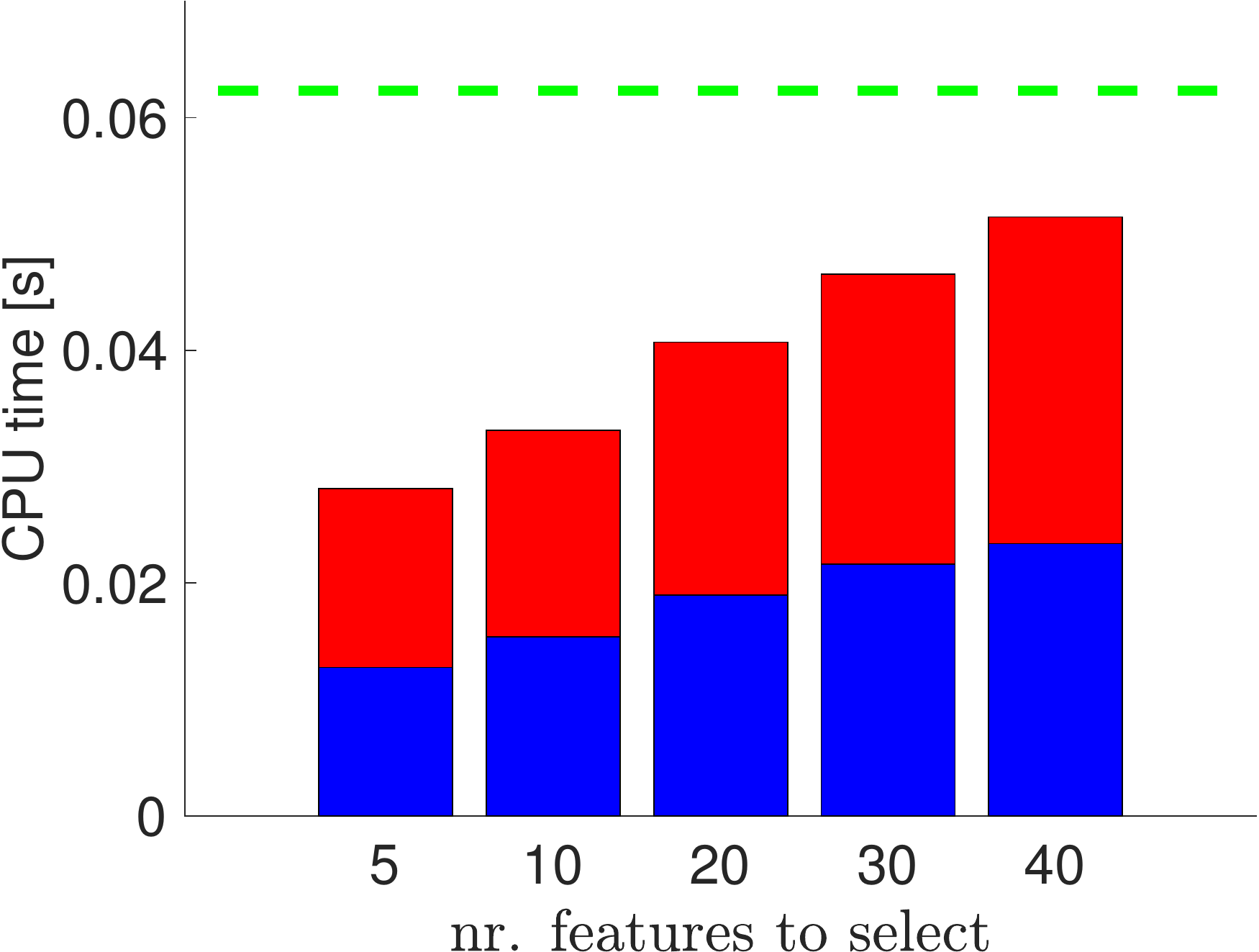} \\
\hspace{2mm}{\smaller MH_01_easy}
\end{minipage}
& 
\begin{minipage}{\mpw}%
\centering%
\includegraphics[width=\mww\columnwidth]{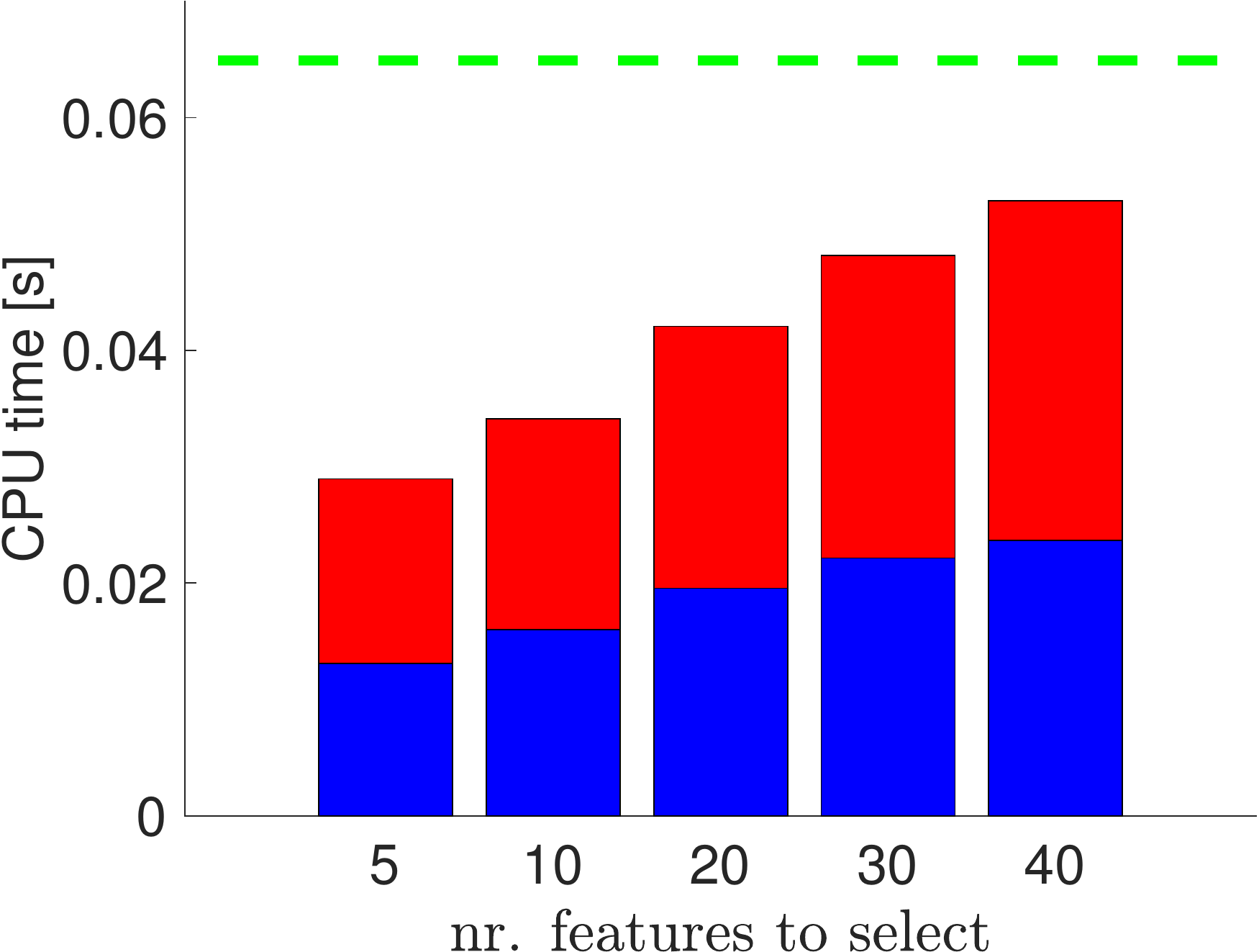} \\
\hspace{2mm}{\smaller MH_02_easy}
\end{minipage}
&
\begin{minipage}{\mpw}%
\centering%
\includegraphics[width=\mww\columnwidth]{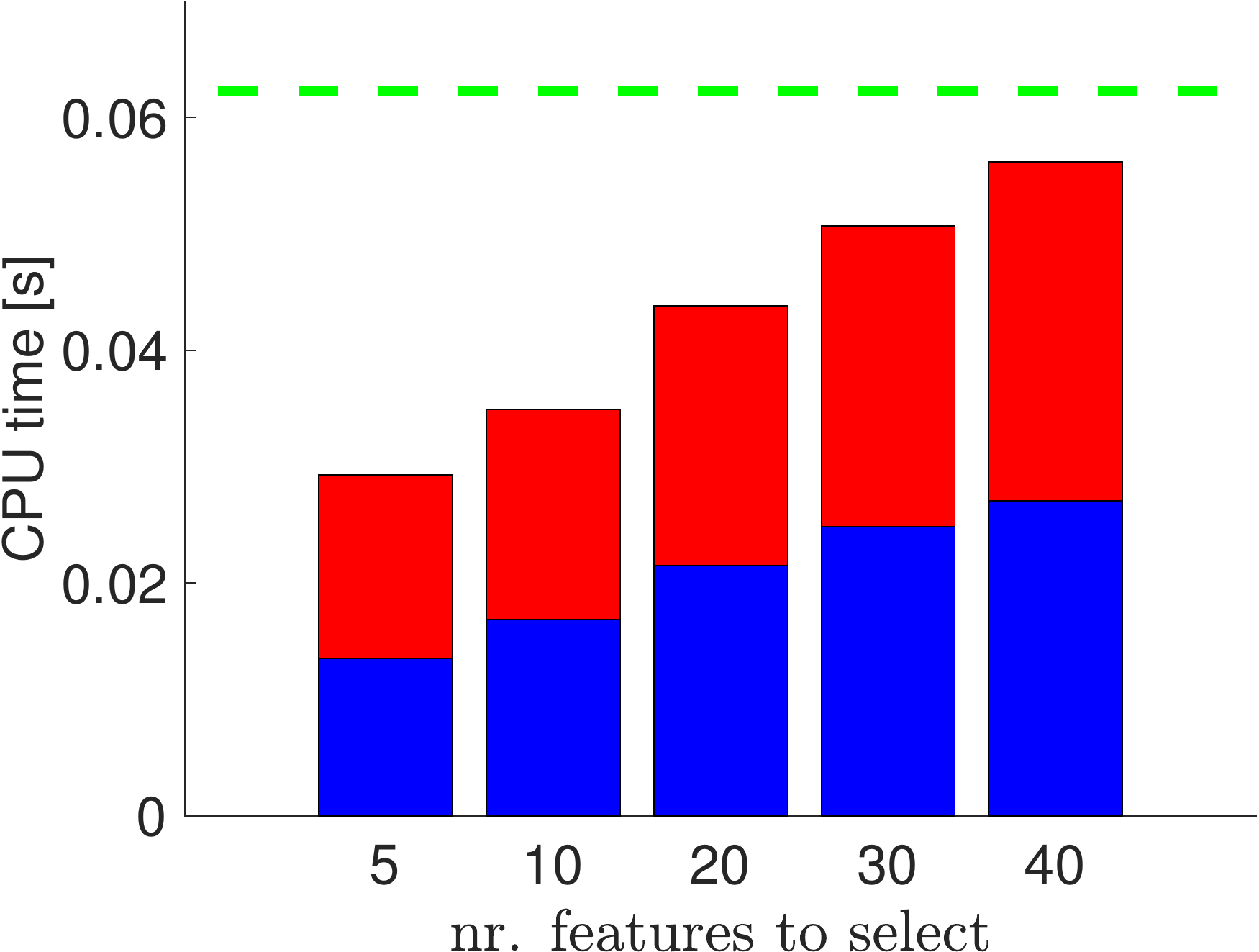} \\
\hspace{2mm}{\smaller MH_03_med}
\end{minipage}
&
\begin{minipage}{\mpw}%
\centering%
\includegraphics[width=\mww\columnwidth]{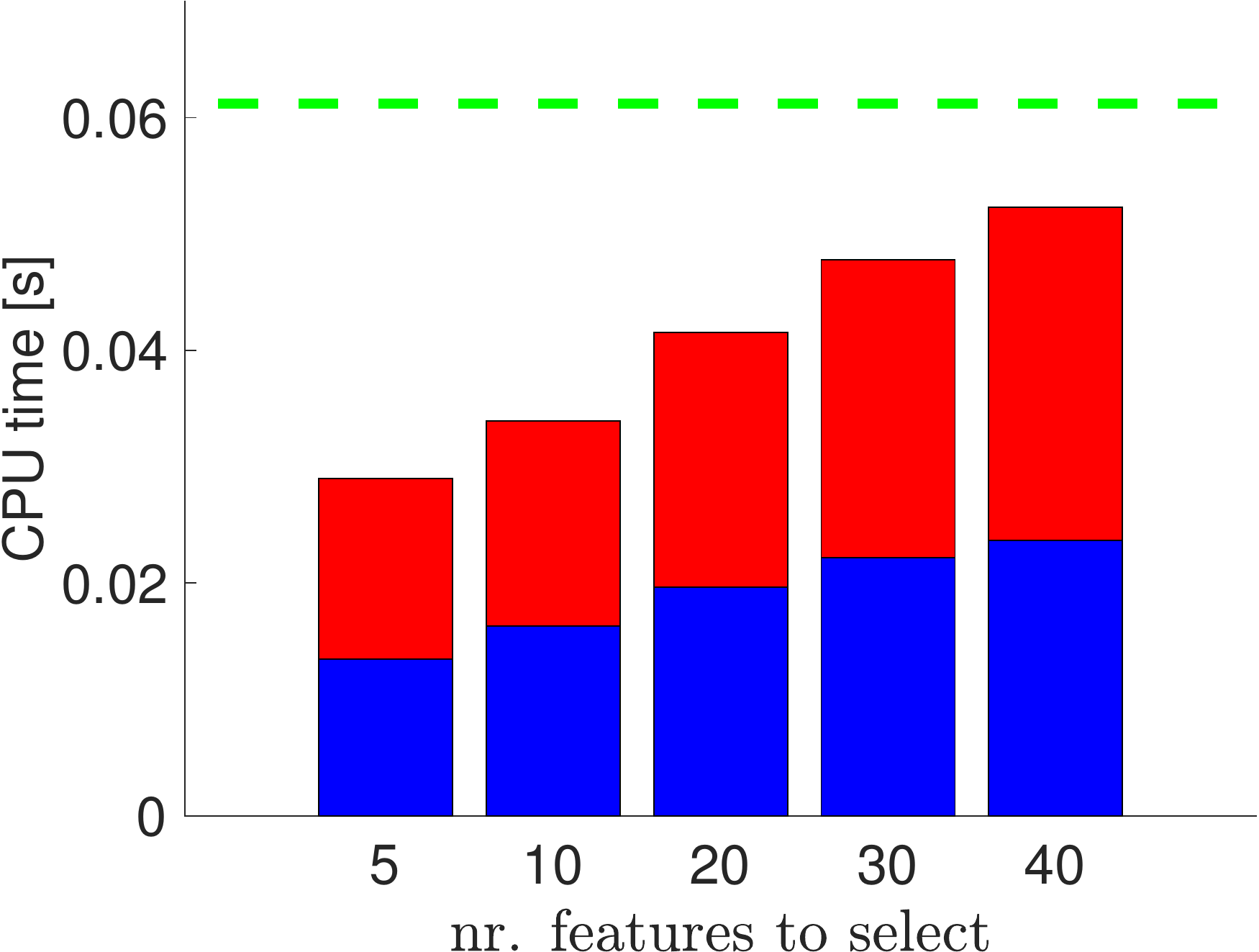} \\
\hspace{2mm}{\smaller MH_04_hard}
\end{minipage}
&
\begin{minipage}{\mpw}%
\centering%
\includegraphics[width=\mww\columnwidth]{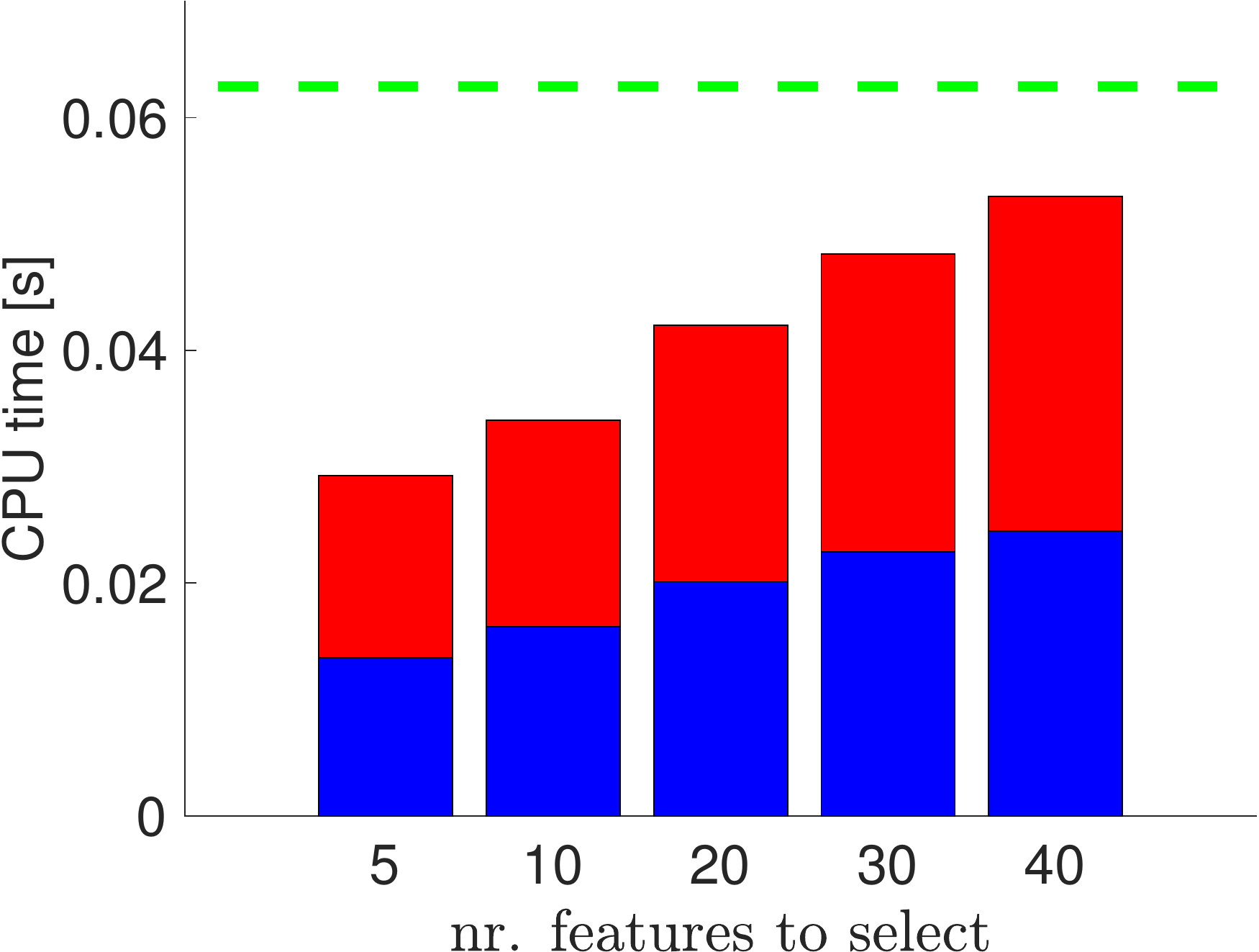} \\
\hspace{2mm}{\smaller MH_05_hard}
\end{minipage}
\\
%%%%%%%%%%%%%%%%%%%%%%%%%%%%%%%%%%%%%%%%%%%%%%%%%%%%%%%%%%%%%%%%%%%%%%%%%%%%%%%%%%%%%%%%%%
%%%%%%%%%%%%%%%%%%%%%%%%%%%%%%%%%%%%%%%%%%%%%%%%%%%%%%%%%%%%%%%%%%%%%%%%%%%%%%%%%%%%%%%%%%
\begin{minipage}{\mpw}%
\centering
\includegraphics[width=\mww\columnwidth]{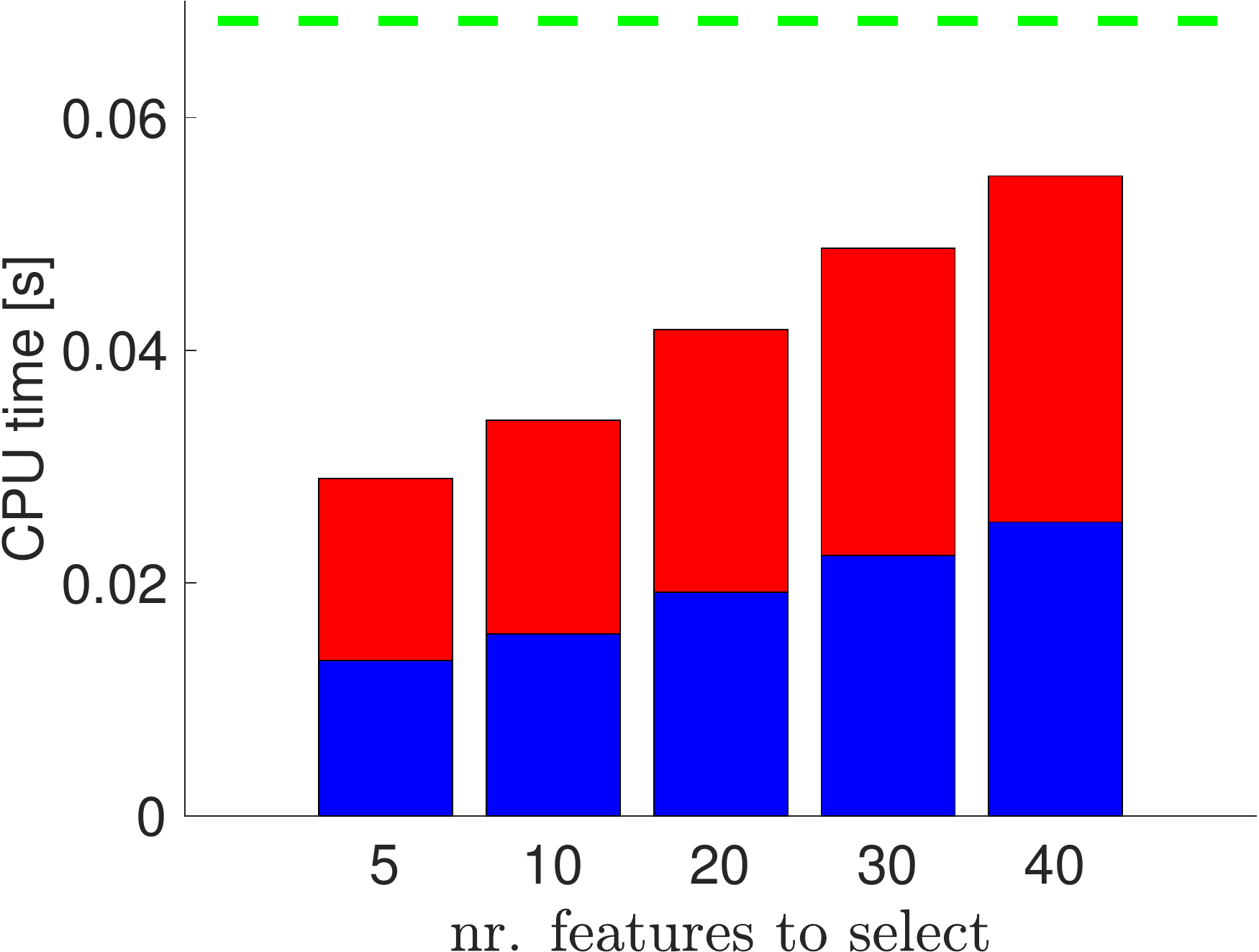} \\
\hspace{2mm}{\smaller V11_easy}
\end{minipage}
&
\begin{minipage}{\mpw}%
\centering
\includegraphics[width=\mww\columnwidth]{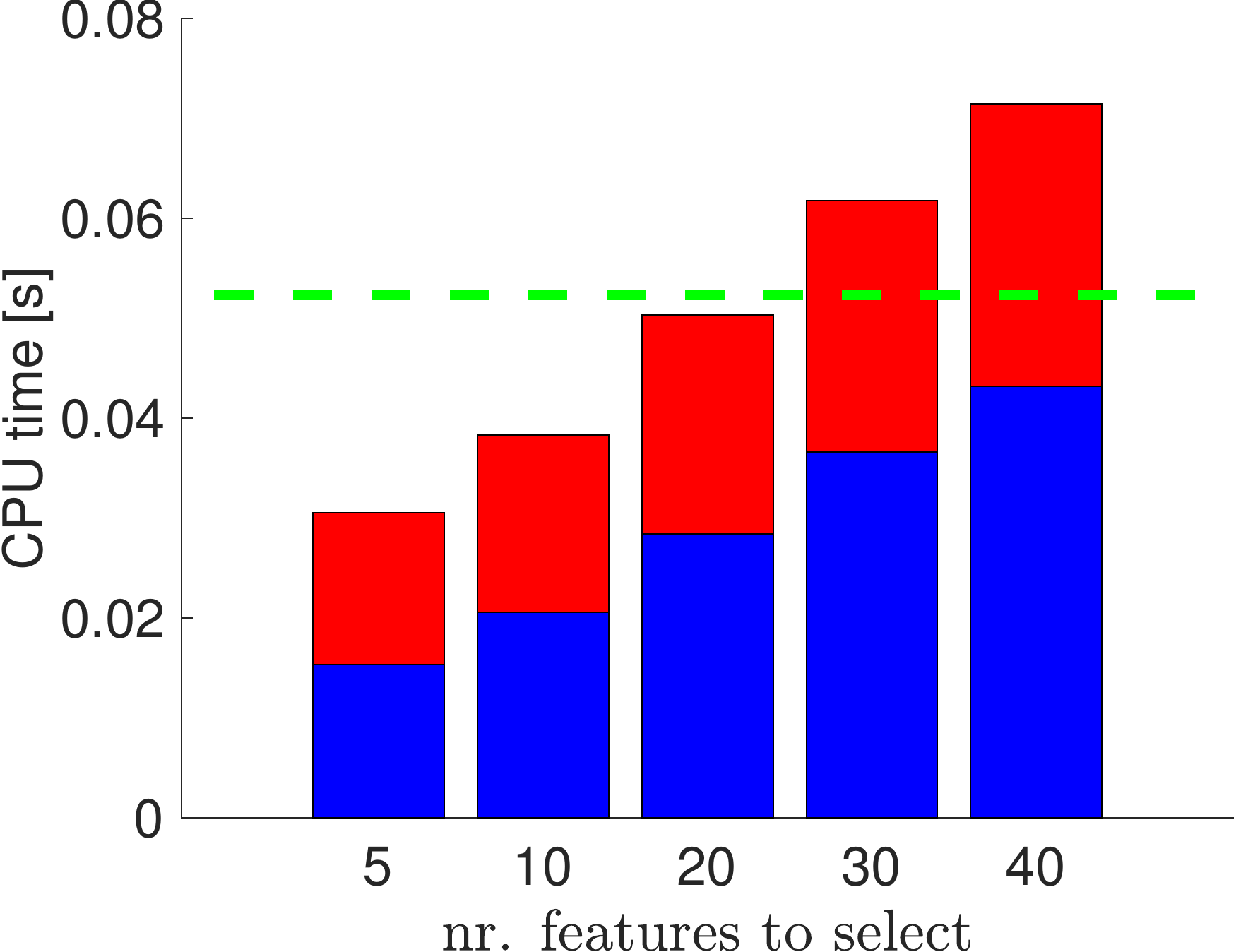} \\
\hspace{2mm}{\smaller V12_med}
\end{minipage}
& 
\begin{minipage}{\mpw}%
\centering%
\includegraphics[width=\mww\columnwidth]{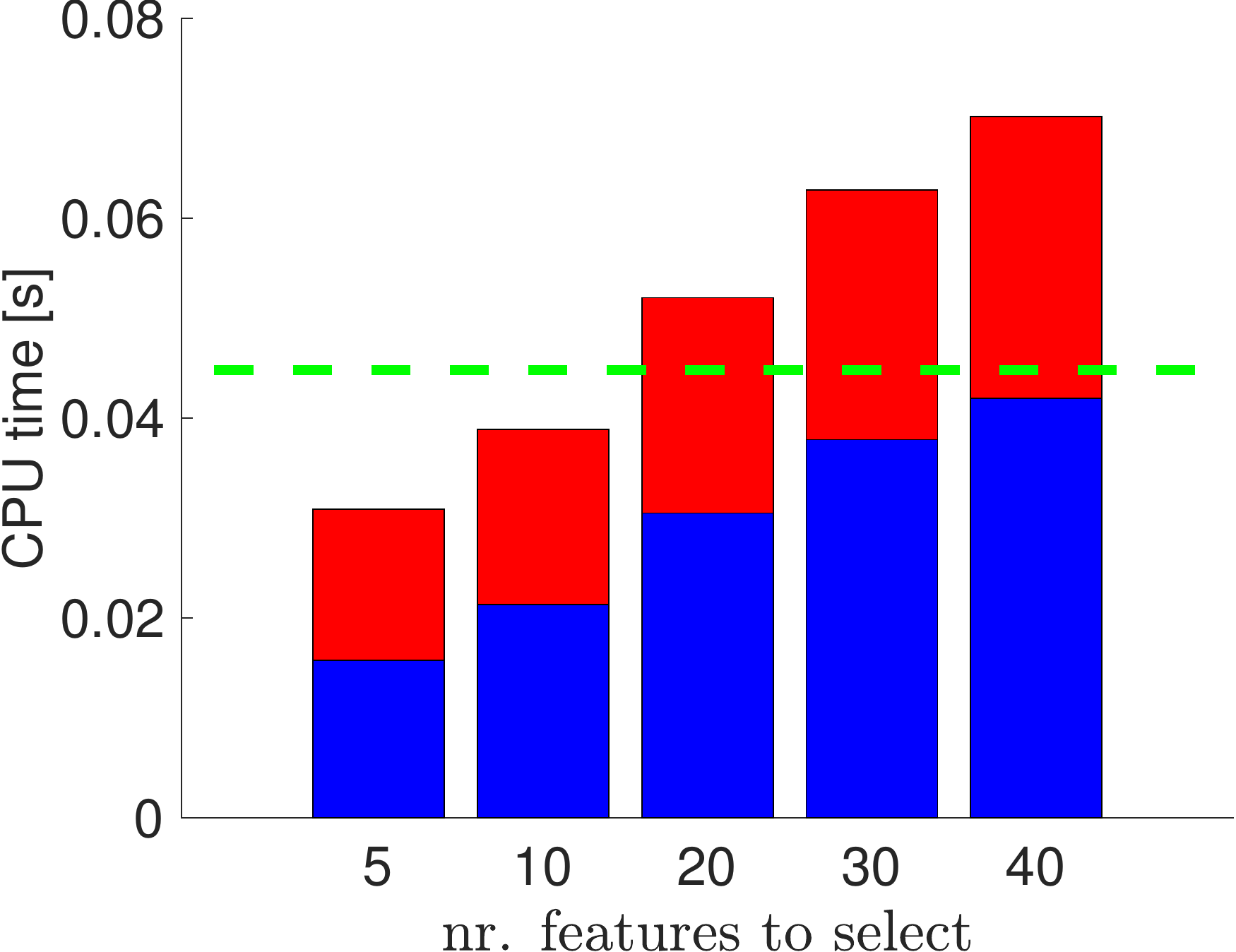} \\
\hspace{2mm}{\smaller V13_hard}
\end{minipage}
&
\begin{minipage}{\mpw}%
\centering%
\includegraphics[width=\mww\columnwidth]{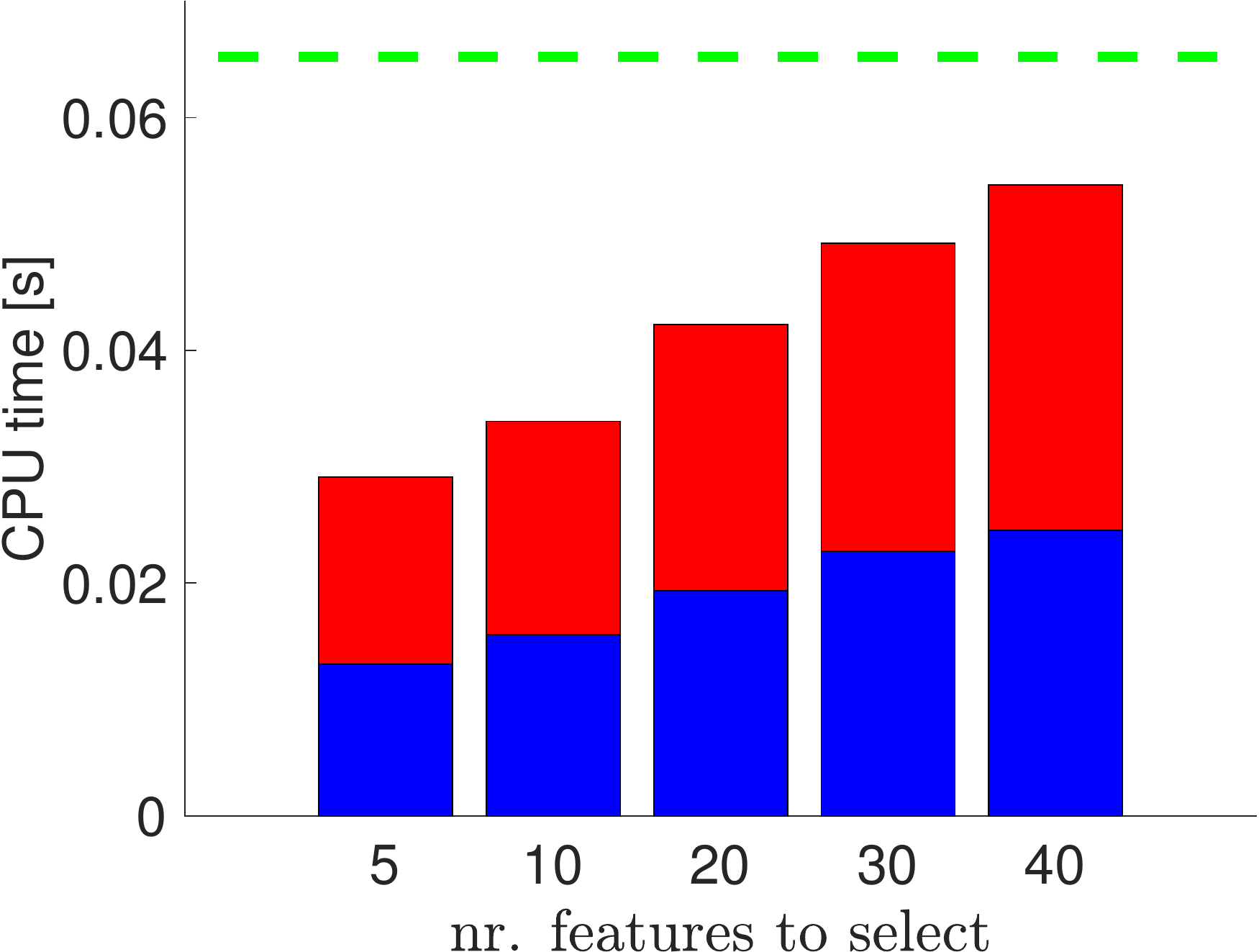} \\
\hspace{2mm}{\smaller V21_easy}
\end{minipage}
&
\begin{minipage}{\mpw}%
\centering%
\includegraphics[width=\mww\columnwidth]{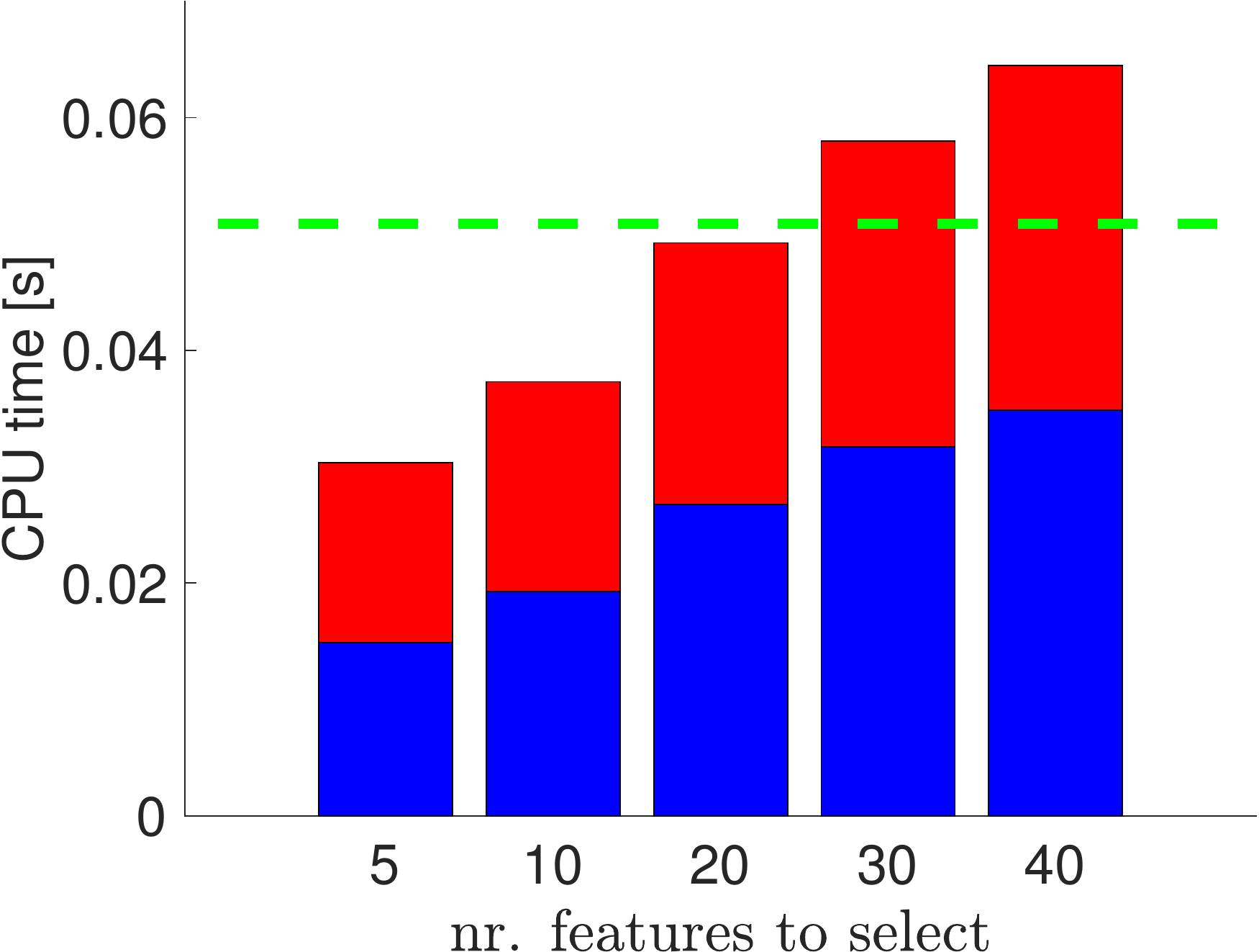} \\
\hspace{2mm}{\smaller V22_med}
\end{minipage}
&
\begin{minipage}{\mpw}%
\centering%
\includegraphics[width=\mww\columnwidth]{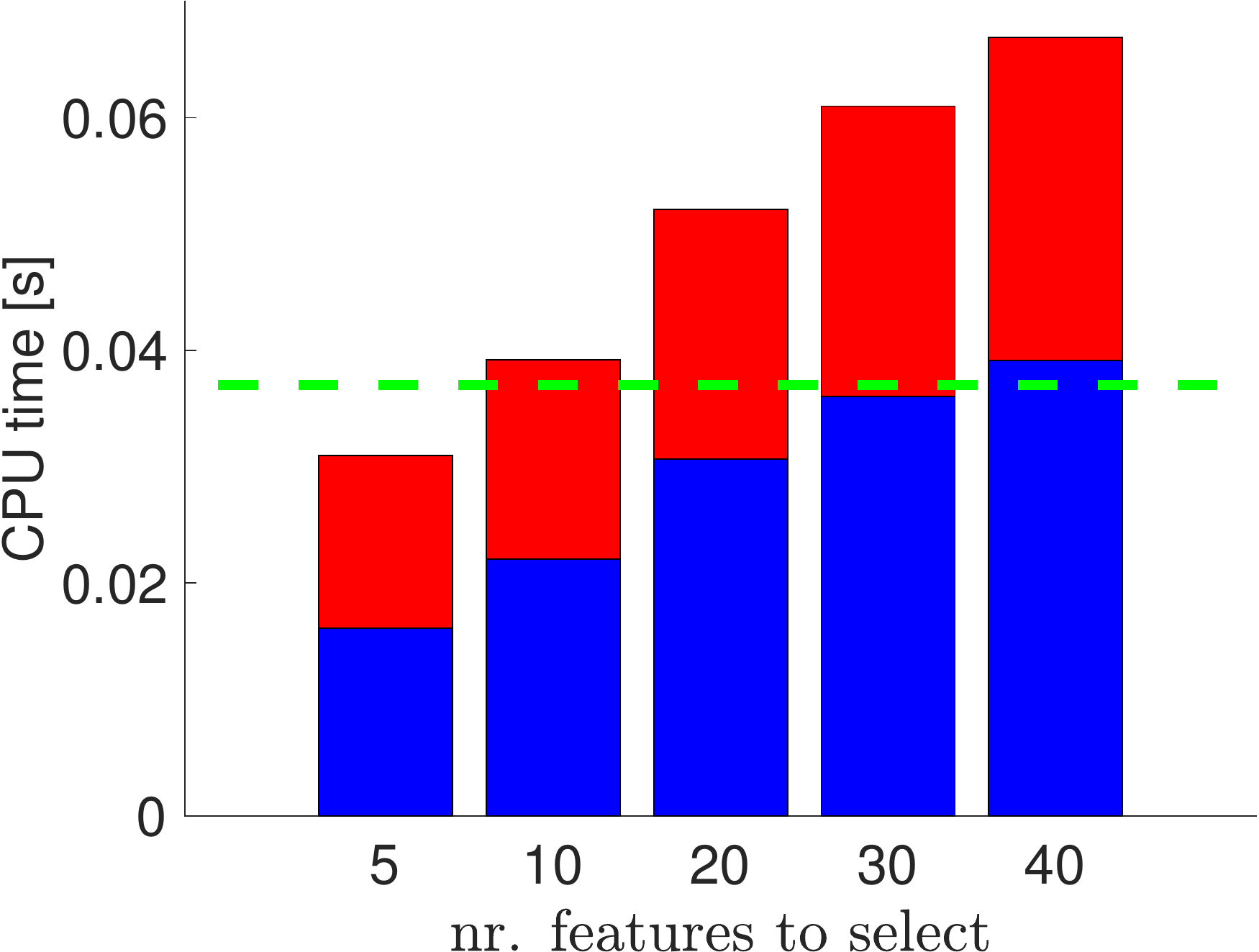} \\
\hspace{2mm}{\smaller V23_hard}
\end{minipage}
\end{tabular}
\end{minipage}%
\caption{\label{fig:logDetDetails-timing}
\edited{CPU time breakdown for the \logDet selector for increasing number of selected features and for each of the EuRoC datasets.
The blue portion of each bar reports the time spent on feature selection (\prettyref{alg:greedy}), while the red portion corresponds 
to the time spent on factor graph optimization in iSAM2; the sum of these two times corresponds to the overall back-end 
time. The back-end time is compared against the CPU time required by \noselection, shown as a dashed green line.}
}
\end{figure*}

%% file: figLogDetDetails-tradeoff.tex
%!TEX root = main.tex

\renewcommand{\mpw}{2.7cm}
\renewcommand{\mww}{1.12}
% trim=0.000001mm 0.0mm 0.000003mm 0.04mm

\begin{figure*}[h]
\begin{minipage}{\textwidth}
\hspace{-0.5cm}
\begin{tabular}{cccccc}%
\begin{minipage}{\mpw}%
\centering
\includegraphics[width=0.9\columnwidth]{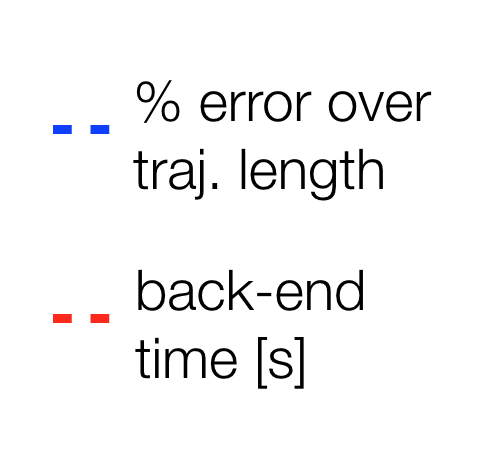} \\
\end{minipage}
&
\begin{minipage}{\mpw}%
\centering
% \psfrag{A}{Luca}
% \psfrag{B}{LucaLuca}
% \psfrag{C}{LucaLucaLuca}
\includegraphics[width=\mww\columnwidth, trim=0.000001mm 0.0mm 0.000003mm 0.04mm,clip]{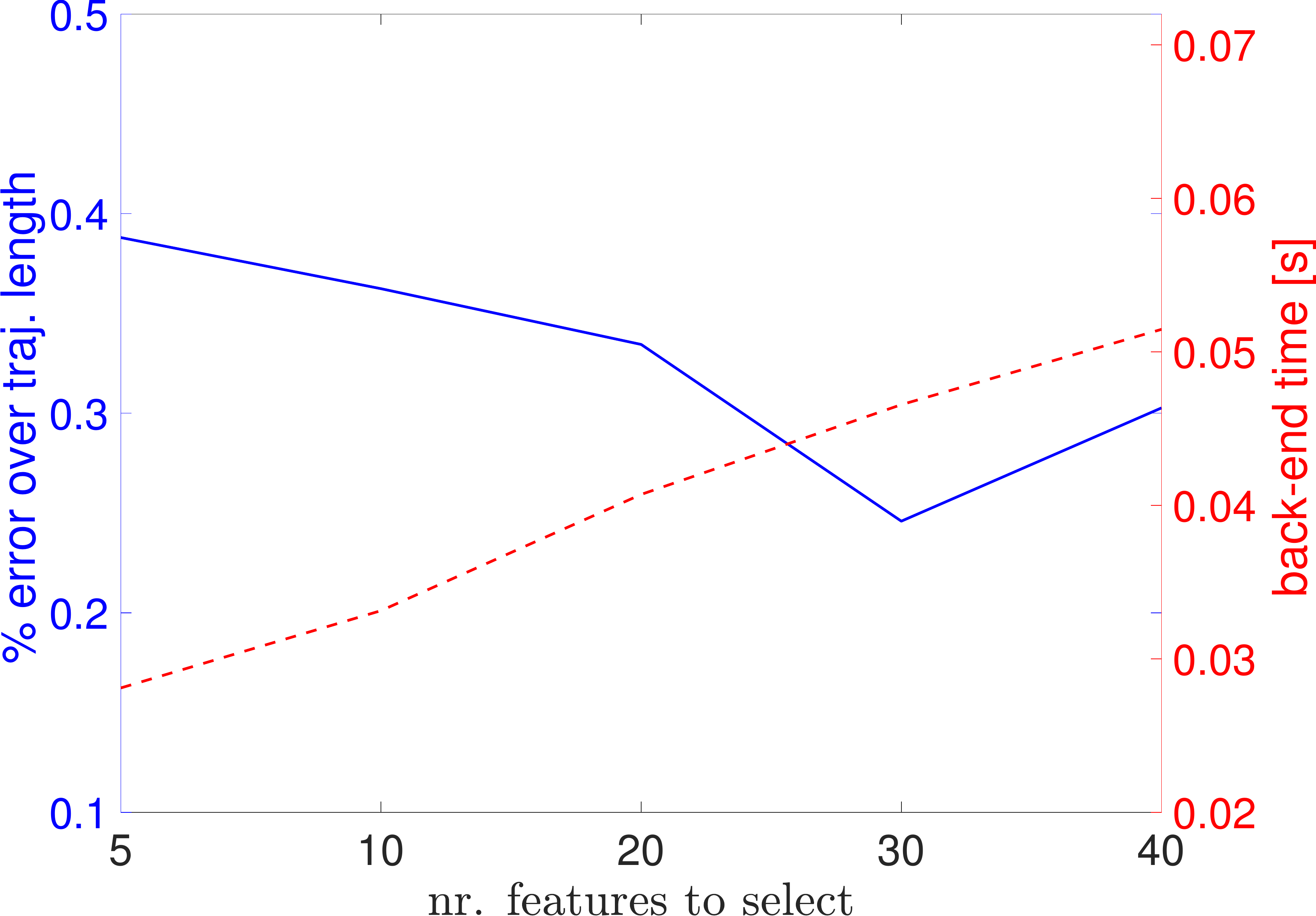} \\
\hspace{2mm}{\smaller MH_01_easy} 
\end{minipage}
& 
\begin{minipage}{\mpw}%
\centering%
\includegraphics[width=\mww\columnwidth, trim=0.000001mm 0.0mm 0.000003mm 0.04mm,clip]{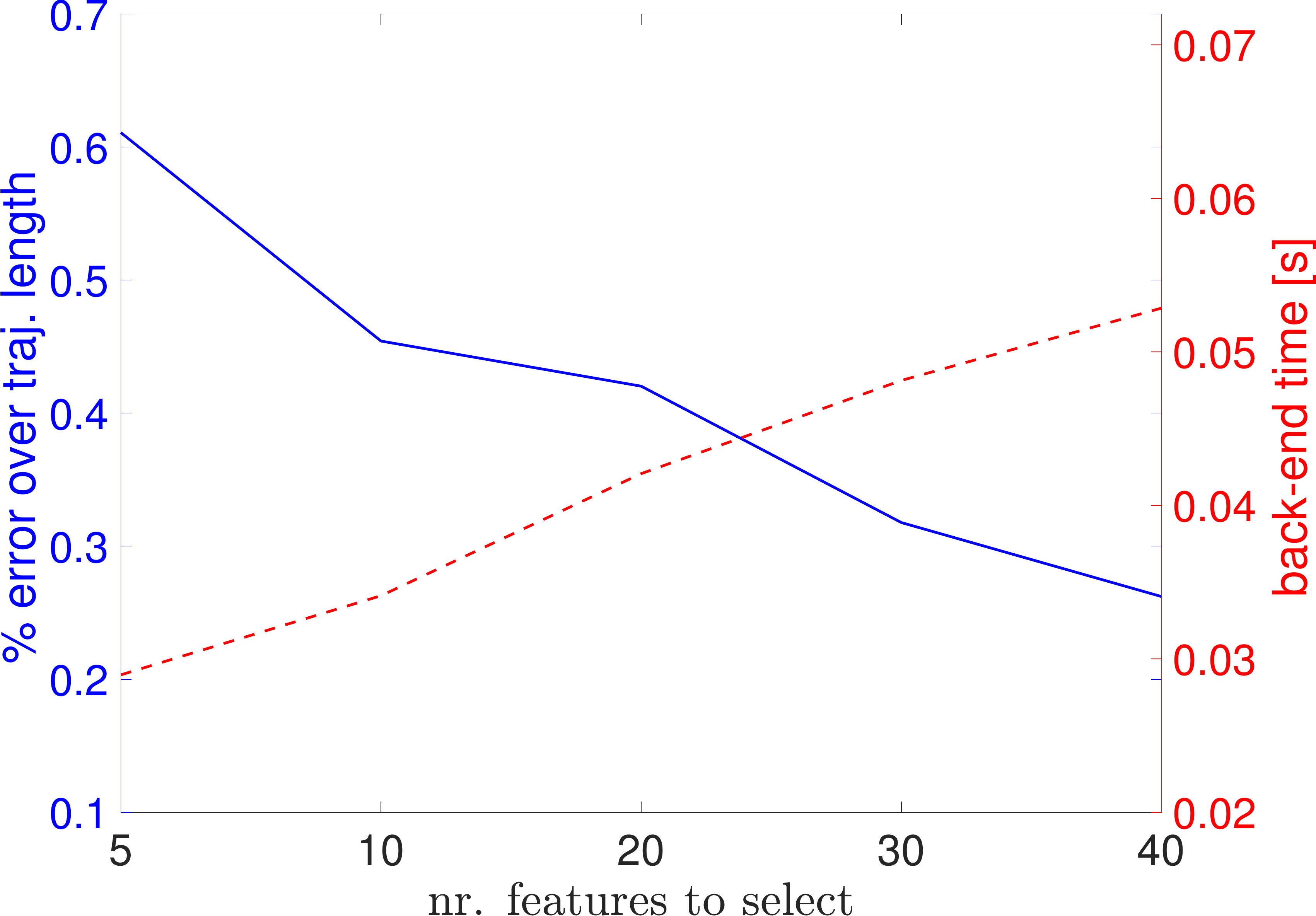} \\
\hspace{2mm}{\smaller MH_02_easy}
\end{minipage}
&
\begin{minipage}{\mpw}%
\centering%
\includegraphics[width=\mww\columnwidth, trim=0.000001mm 0.0mm 0.000003mm 0.04mm,clip]{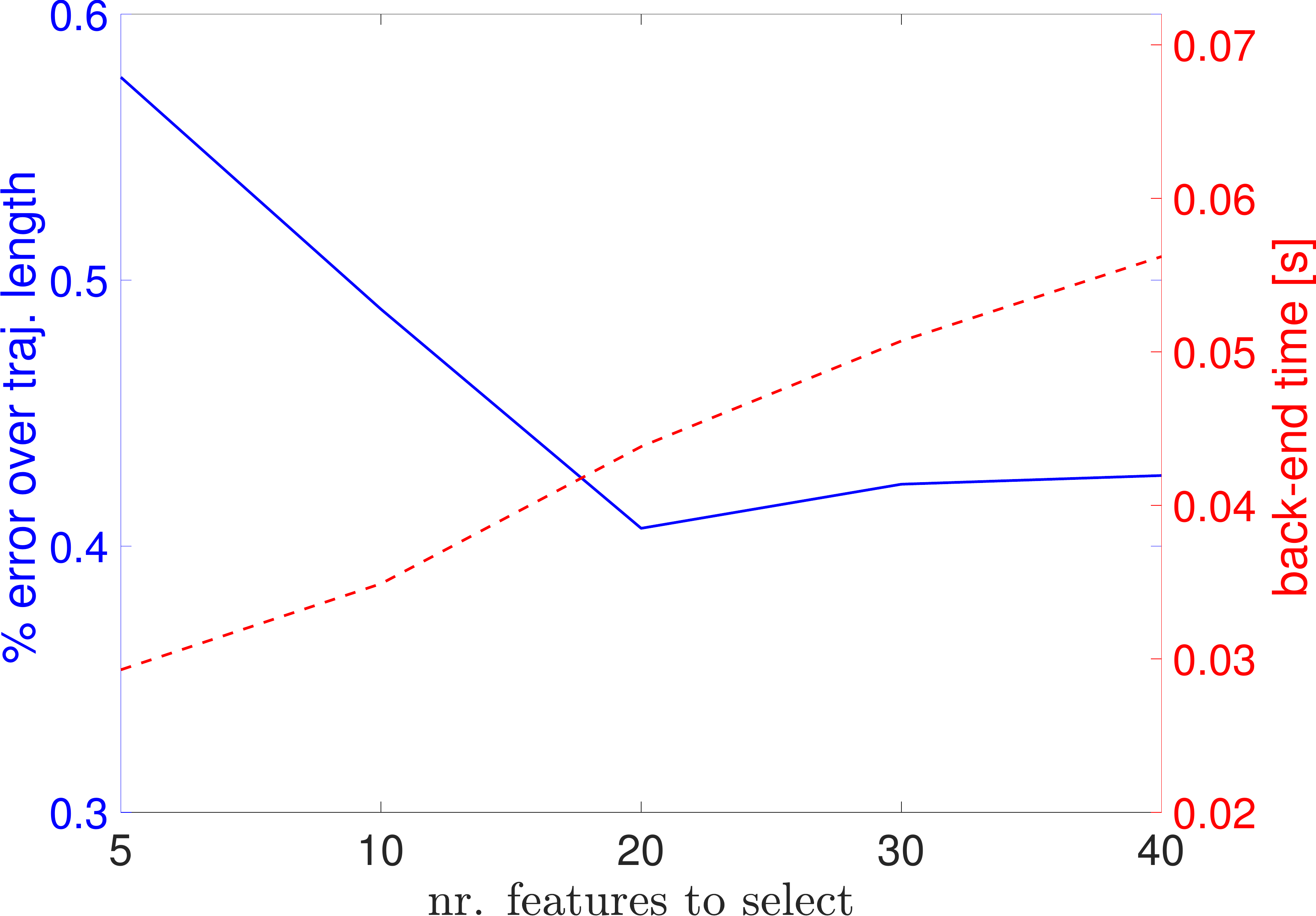} \\
\hspace{2mm}{\smaller MH_03_med}
\end{minipage}
&
\begin{minipage}{\mpw}%
\centering%
\includegraphics[width=\mww\columnwidth, trim=0.000001mm 0.0mm 0.000003mm 0.04mm,clip]{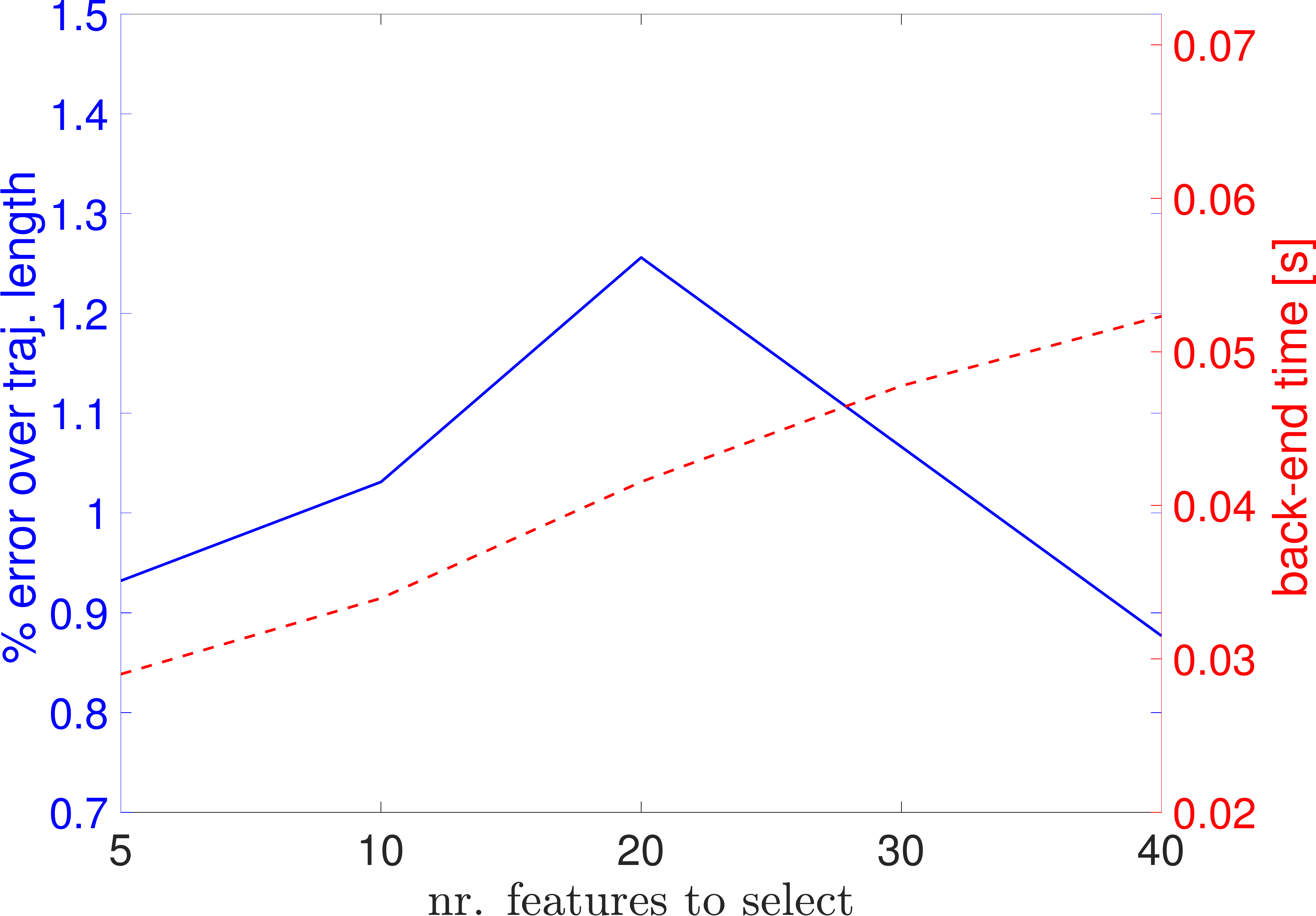} \\
\hspace{2mm}{\smaller MH_04_hard}
\end{minipage}
&
\begin{minipage}{\mpw}%
\centering%
\includegraphics[width=\mww\columnwidth, trim=0.000001mm 0.0mm 0.000003mm 0.04mm,clip]{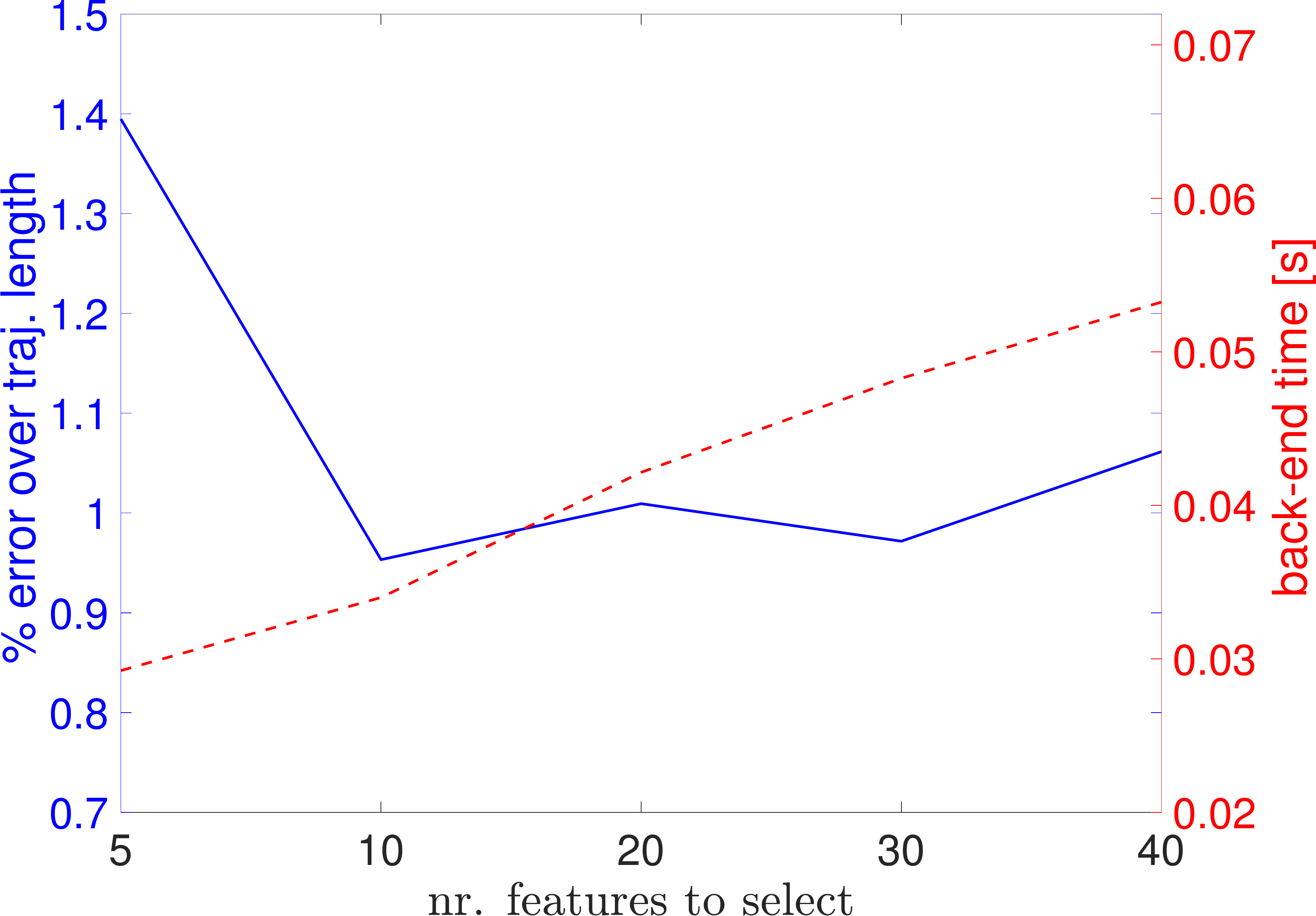} \\
\hspace{2mm}{\smaller MH_05_hard}
\end{minipage}
\\
%%%%%%%%%%%%%%%%%%%%%%%%%%%%%%%%%%%%%%%%%%%%%%%%%%%%%%%%%%%%%%%%%%%%%%%%%%%%%%%%%%%%%%%%%%
%%%%%%%%%%%%%%%%%%%%%%%%%%%%%%%%%%%%%%%%%%%%%%%%%%%%%%%%%%%%%%%%%%%%%%%%%%%%%%%%%%%%%%%%%%
\begin{minipage}{\mpw}%
\centering
\includegraphics[width=\mww\columnwidth, trim=0.000001mm 0.0mm 0.000003mm 0.04mm,clip]{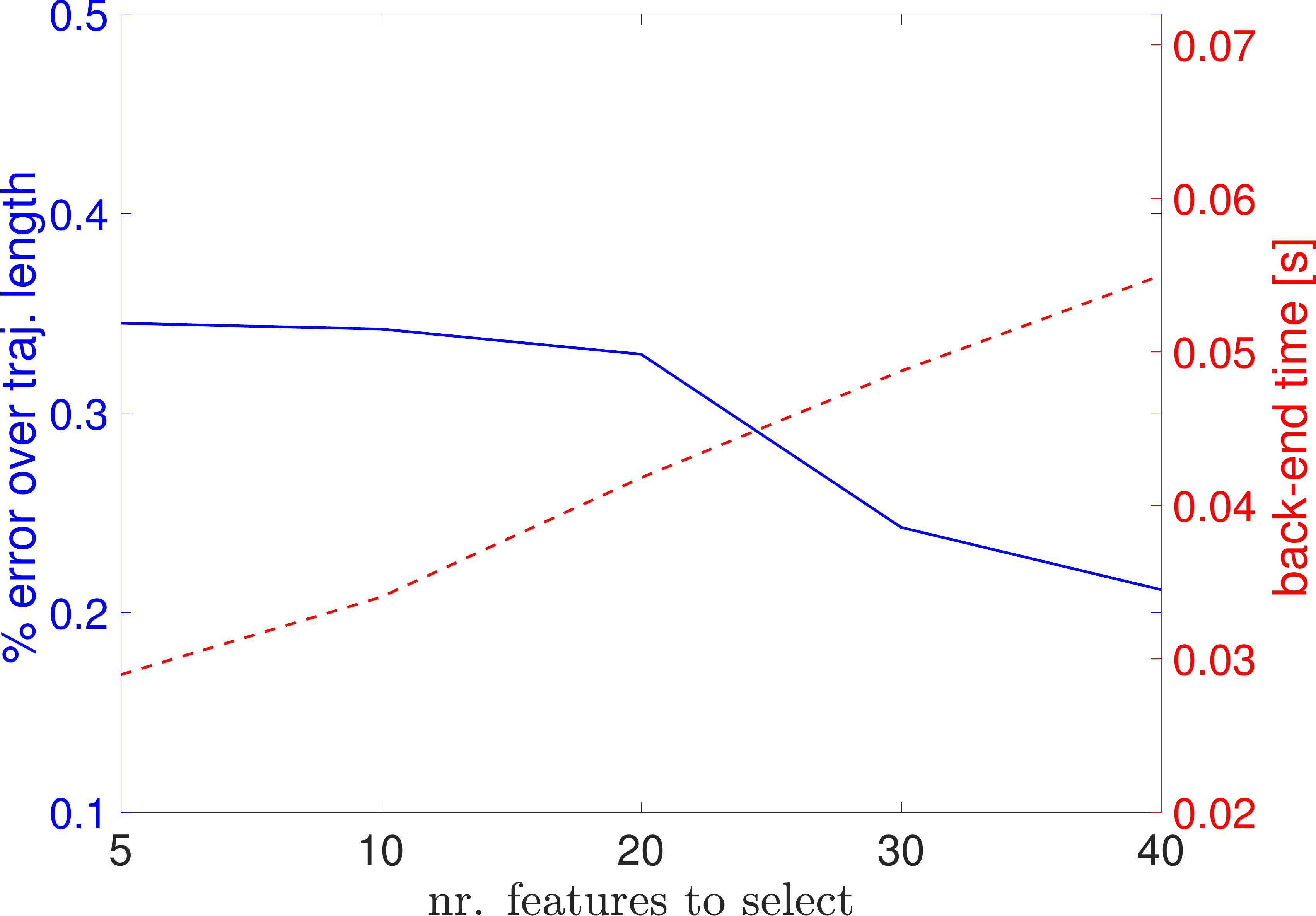} \\
\hspace{2mm}{\smaller V11_easy}
\end{minipage}
&
\begin{minipage}{\mpw}%
\centering
\includegraphics[width=\mww\columnwidth, trim=0.000001mm 0.0mm 0.000003mm 0.04mm,clip]{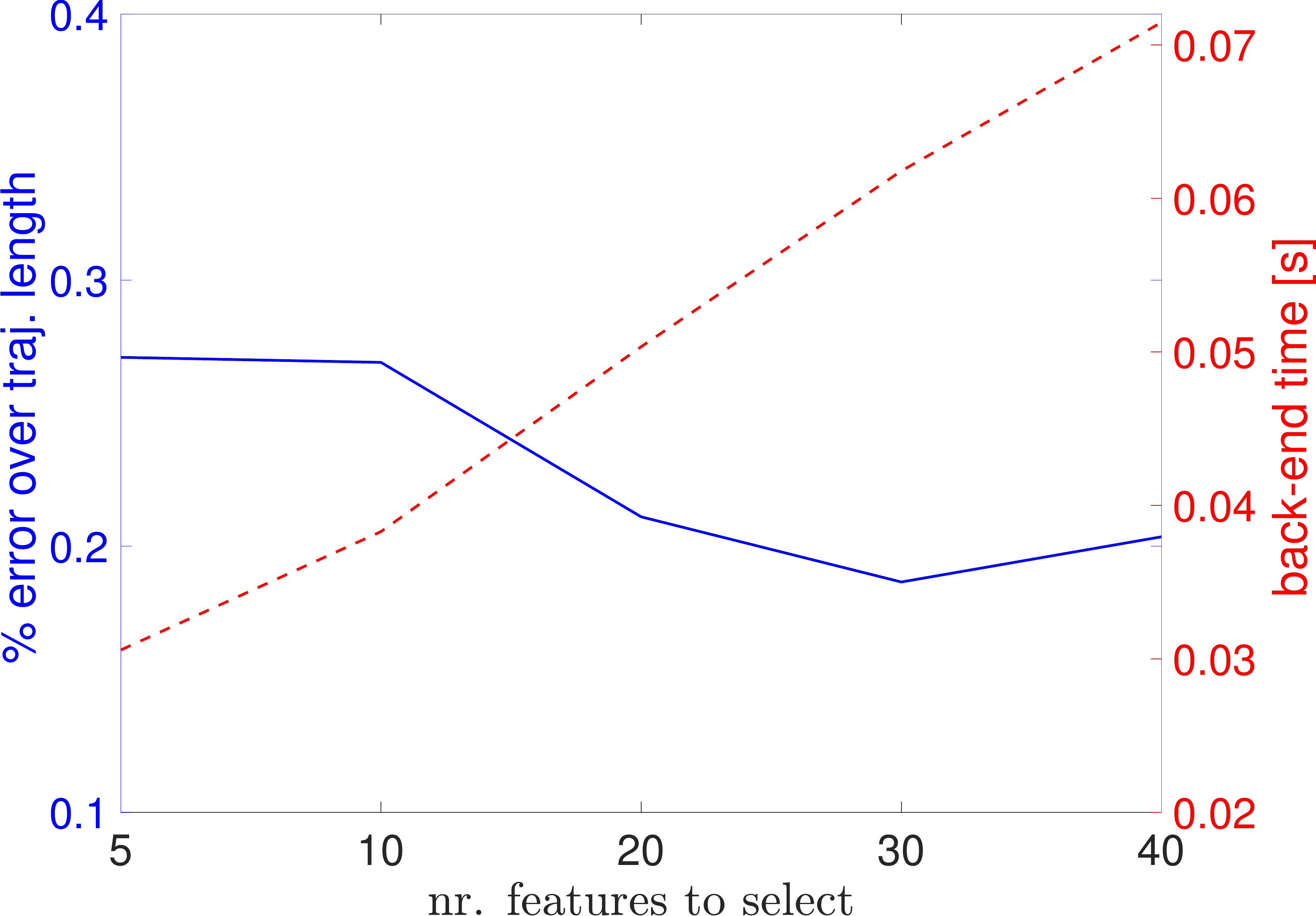} \\
\hspace{2mm}{\smaller V12_med}
\end{minipage}
& 
\begin{minipage}{\mpw}%
\centering%
\includegraphics[width=\mww\columnwidth, trim=0.000001mm 0.0mm 0.000003mm 0.04mm,clip]{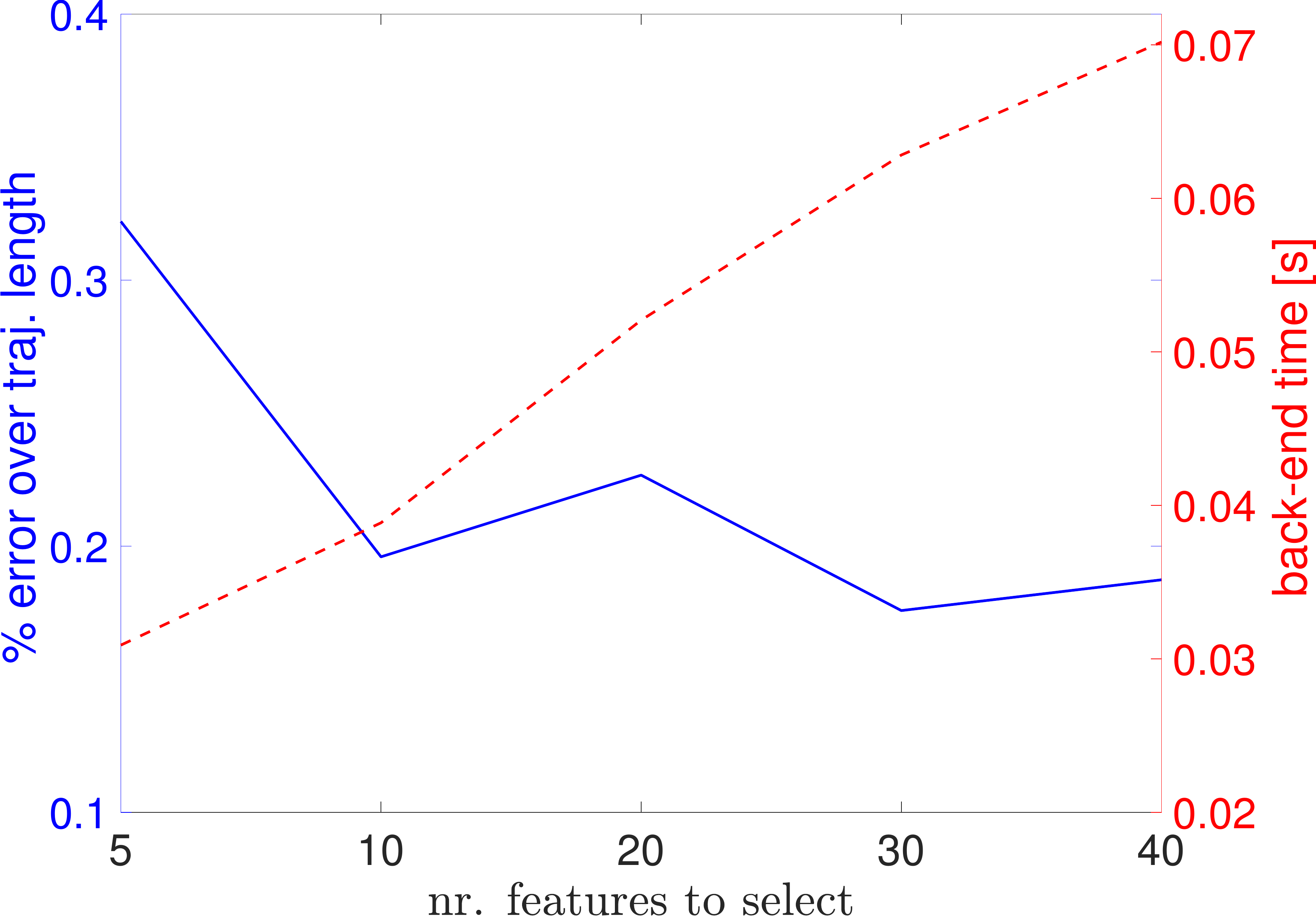} \\
\hspace{2mm}{\smaller V13_hard}
\end{minipage}
&
\begin{minipage}{\mpw}%
\centering%
\includegraphics[width=\mww\columnwidth, trim=0.000001mm 0.0mm 0.000003mm 0.04mm,clip]{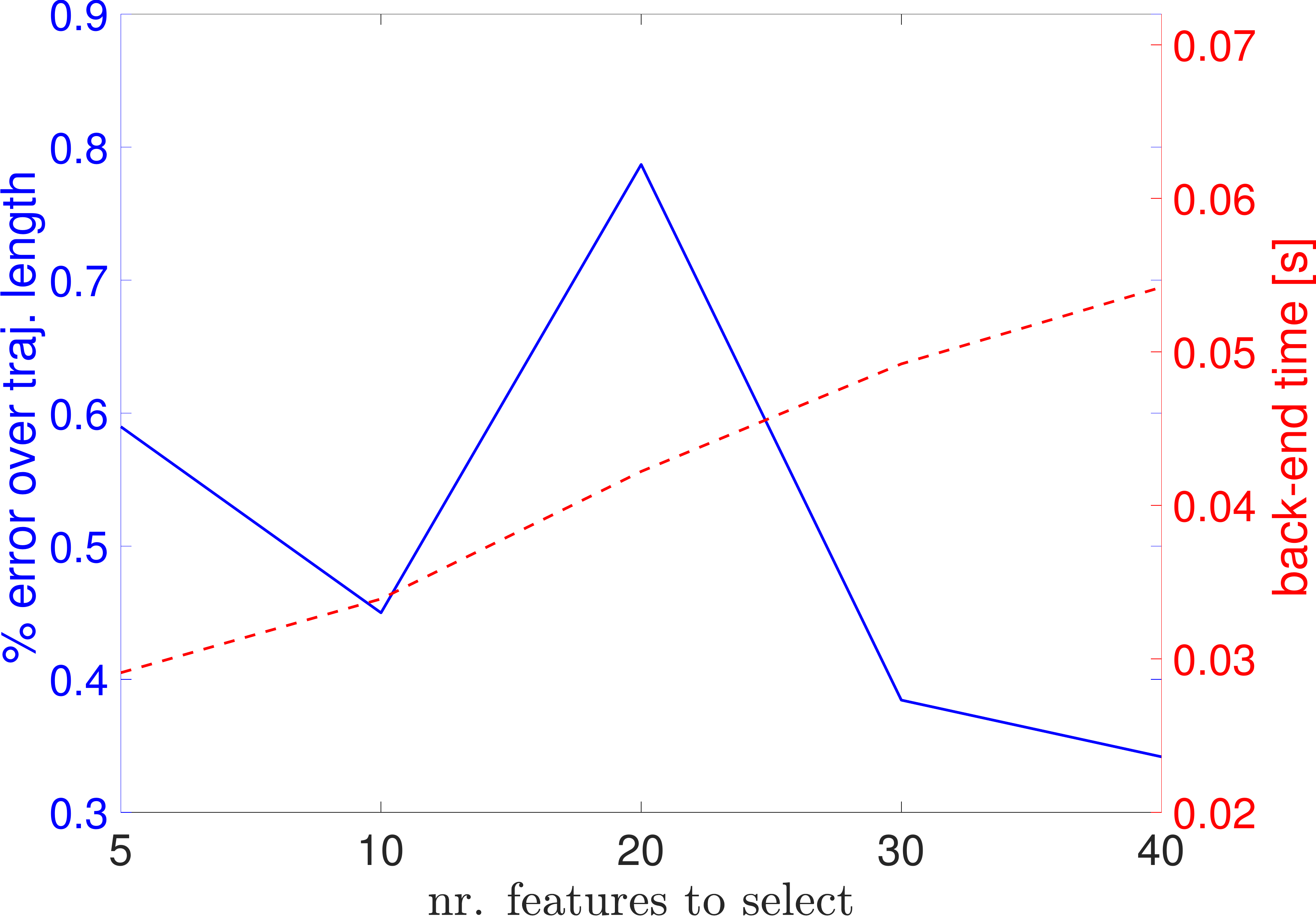} \\
\hspace{2mm}{\smaller V21_easy}
\end{minipage}
&
\begin{minipage}{\mpw}%
\centering%
\includegraphics[width=\mww\columnwidth, trim=0.000001mm 0.0mm 0.000003mm 0.04mm,clip]{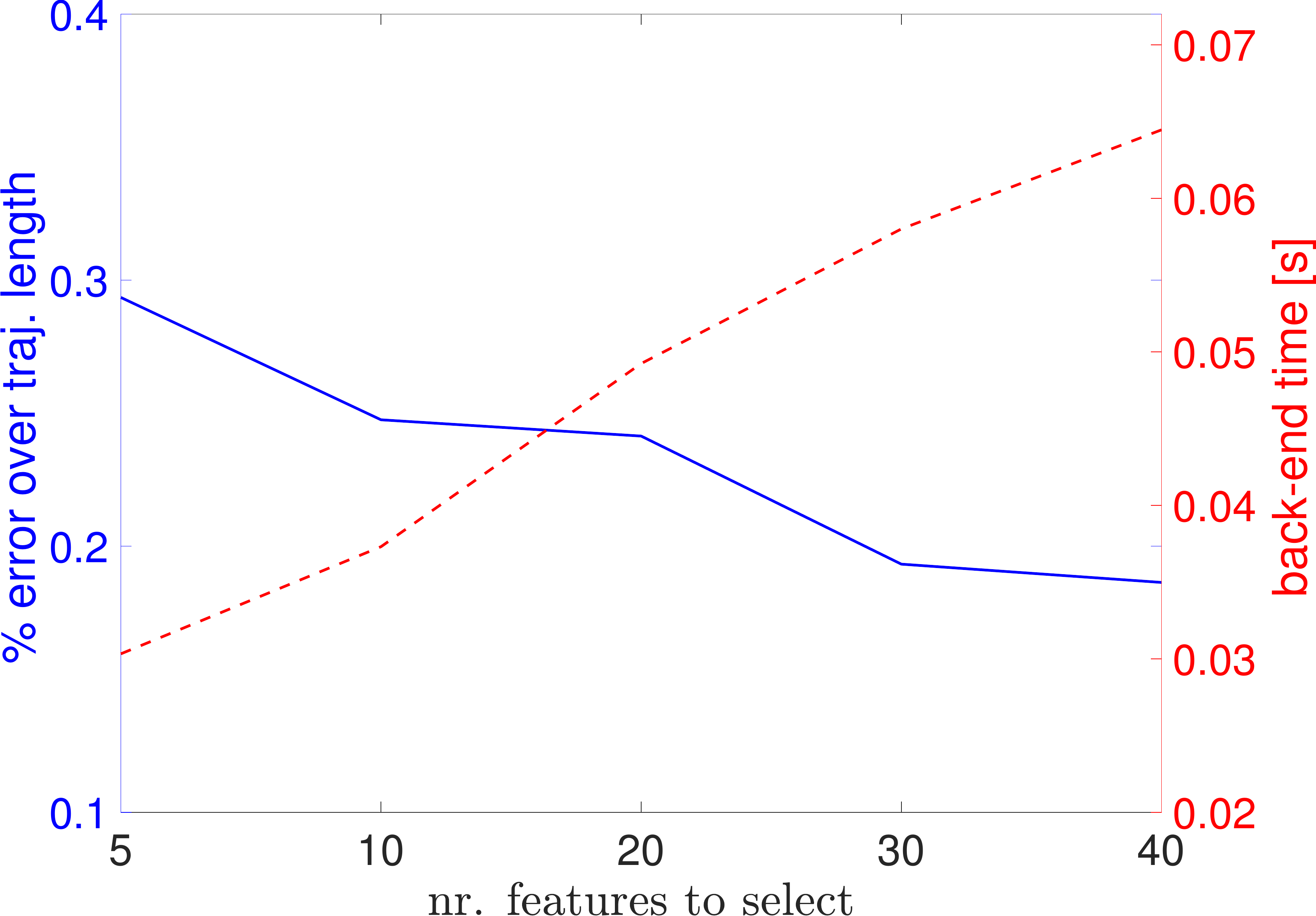} \\
\hspace{2mm}{\smaller V22_med}
\end{minipage}
&
\begin{minipage}{\mpw}%
\centering%
\includegraphics[width=\mww\columnwidth, trim=0.000001mm 0.0mm 0.000003mm 0.04mm,clip]{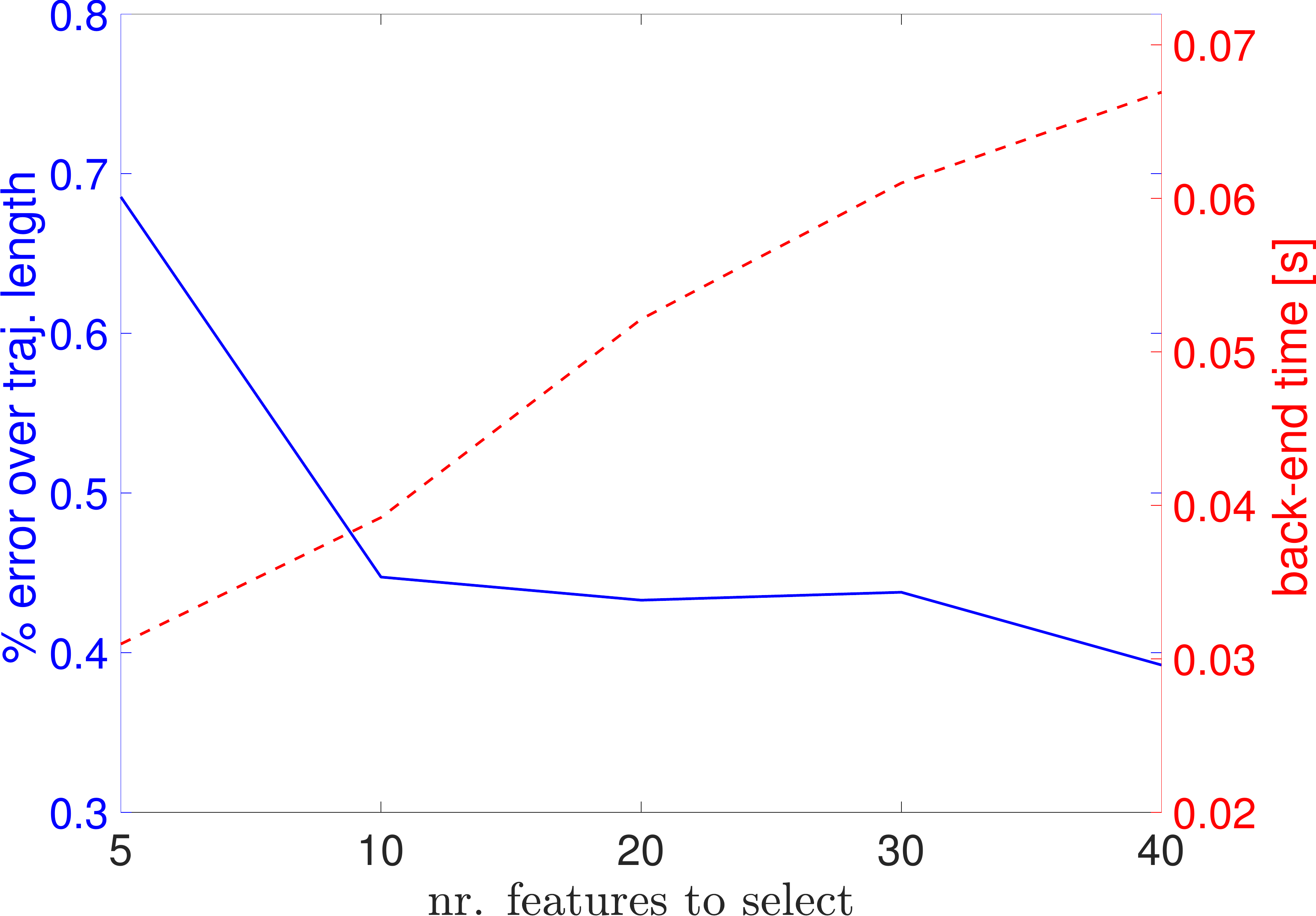} \\
\hspace{2mm}{\smaller V23_hard}
\end{minipage}
\end{tabular}
\end{minipage}%
\caption{\label{fig:logDetDetails-tradeoff}
\edited{Performance/computation trade-off for the \logDet selector for increasing number of selected features and for each of the EuRoC datasets.
The back-end time, shown as a dashed red line, corresponds to the sum of the feature selection time and the iSAM2 estimation time. 
The figure also reports the absolute translation error as a percentage of the 
trajectory traveled, shown as a solid blue line.  
Results are averaged over 5 runs performed on each dataset.}
}
\end{figure*}

%% file: conclusion.tex
%!TEX root = main.tex

\section{Conclusion and Future Work}
\label{sec:conclusion}

This work provides an attention mechanism for visual-inertial navigation.
This mechanism takes the form of a feature selector, which retains the most informative 
visual features detected by the VIN front-end (pre-attentive process) and feeds them to 
the estimation back-end. We proposed two algorithms for feature selection. 
Both algorithms enjoy four desirable qualities: they are predictive in nature, 
in that they are aware of the motion of the robot in the immediate future;
they are task-driven, since they select a set of features that minimize the VIN estimation error;
they are greedy, hence efficient and easy to implement; they come with performance guarantees 
that bound their sub-optimality. We demonstrated our feature selection extensively on both 
realistic Monte Carlo simulations and real-world data collected by a micro aerial vehicle. 
The experiments suggest that the feature selection seriously impacts VIN performance; 
the use of the proposed techniques reduces the estimation error in easy datasets, 
and enables accurate estimation in difficult datasets in which standard approaches 
would fail on a limited budget of visual features. 
This work opens many avenues for future investigation.

\myparagraph{Computational improvements}
The first avenue for future work consists in reducing the computational time of feature selection. 
Two main ideas can make the feature selection time  negligible. 
The first stems from the observation that the greedy algorithm is trivially parallelizable: 
the marginal gain of each feature can be computed independently; leveraging this fact alone 
would result in large computational savings. 
The second idea is to use sparse matrix manipulation to compute the determinant and the 
smallest eigenvalues; our current implementation uses dense matrices. % in \eigen.
% ; moreover, 
% for the computation of the smallest eigenvalue, we use \eigen svd, which performs a full 
% singular value decomposition. In the future we plan to explore iterative algorithms that 
%  compute only the smallest eigenvalue, rather than the full spectrum.
% Finally, during our tests, we observed that the time to create the matrices 
% to be used in the greedy algorithm is non-negligible; more mundane work consists in 
% improving memory allocation and management to build these 
% (sparse) matrices.
\editedTwo{Another interesting avenue  consists in including a learning mechanism to improve 
feature selection. 
% that may improve computational aspects 
A learning-based method may allow capturing more complex sensory-motor dynamics 
 and may improve the feature selection over time, potentially adjusting to changing environments 
 and time-varying robot dynamics.
% The approach proposed in this paper is not able to leverage
% previous ``experience'', but one could
%  design a learning-based method that learns how to select the
% best features depending on the robot ``intentions''. 
}
% Such a learning approach would take a set of candidate
% features and some information about the ìintentionsî of the robot as
% input, and output a suitable choice of features. Indeed (in relation to
% the following comment), we thought about using learning-based
% approaches as an alternative to the greedy algorithm (Algorithm 1), but
% decided to postpone that investigation to future work. In our opinion,
% while a learning-based approach may be viable in practice, it would
% make much harder to establish performance guarantees on the resulting
% feature selection.

\myparagraph{Task-driven perception}
A second avenue for future work consists in extending our attention framework. 
We plan to explore two paths. First, while~\eqref{eq:featureSelection} 
minimizes the localization uncertainty subject to a feature budget, one may 
also consider a ``dual'' problem in which one minimizes the number of features 
to be used, while satisfying a desired localization performance. 
From the technical standpoint, this alternative formulation can be tackled 
in a similar manner and in both cases greedy algorithms have sub-optimality guarantees.
This alternative formulation would provide a grounded answer to the 
question: \emph{how much visual information is needed to navigate at a desired accuracy?}
The second avenue consists in extending our attention framework to other tasks: for instance, 
\emph{how many visual features does the robot need to sense in order to avoid crashing into nearby obstacles?}
% An initial idea to do so, is to include a \emph{requirement matrix} in the cost~\eqref{eq:featureSelection}, 
% similarly to~\cite{Lerner07tro} 
%\emph{how many visual feature the robot needs to sense to avoid crashing into nearby obstacles?}
We believe  these are necessary steps towards a task-driven perception 
theory, that can enable autonomy on robots with tight resource constraints. 
% $$ed robots with tight budget on sensing and computation.

% - requirement matrix that multiplies the covariance/ info matrix (diagonal)
% - select minimum number of features such that metric is below a bound
% - select k features to minimize uncertainty
% - select min features to bound uncertainty
%
% - setup: linear least squares problems with know rotations, IMU, vision (algebraic errors)
% - submodularity lambda max
% - linear relaxation
% - SDP relaxation?
% - MILP

%% file: appendix-linearImuModel.tex
%!TEX root = main.tex

\newcommand{\sumIndTwo}{{\color{black}\sum_{\ind = \k}^{\j-2}}} % {{\color{red}\sum_{\ind = \k+1}^{\j-1}}}
\newcommand{\sumInd}{{\color{black}\sum_{\ind = \k}^{\j-1}}} % {{\color{red}\sum_{\ind = \k+1}^{\j-1}}}

\subsection{Linear Imu Model}
\label{sec:linearImuModelAppendix}

In this appendix we provide explicit expressions for the matrices and vectors appearing in 
the IMU model~\eqref{eq:imuModel}.

Given the velocity $\vel{\k}$ and the position $\pos{\k}$ of the robot at time $k$, 
we can get $\vel{\k+1}$ and $\vel{\k+1}$ by simple Euler integration 
using the acceleration $\acc{\k}$:
\bea
 \vel{\k+1} &=& \textstyle  \vel{\k} + \acc{\k} \Deltak \\
  \edited{\pos{\k+1}} &=& \textstyle  \pos{\k} + \vel{\k} \Deltak + \half \acc{\k} \Deltak^2
\eea
%
% Iterating:
% %
% \bea
%  \vel{\k+2} &=& \textstyle  \vel{\k+1} + \acc{\k+1} \Deltak \\
%  &=& \textstyle  \vel{\k} + \acc{\k} \Deltak + \acc{\k+1} \Deltak \\
%  %% POSITION
%   \pos{\k+2} &=& \textstyle  \pos{\k+1} + \vel{\k+1} \Deltak + \half \acc{\k+1} \Deltak^2 \\
%   &=& \textstyle \pos{\k} + \vel{\k} \Deltak + \vel{\k+1} \Deltak + \half \acc{\k} \Deltak^2  + \half \acc{\k+1} \Deltak^2 \nonumber 
% \eea
%
By induction, the velocity and position at time $\j > \k$ are:
\bea
 \vel{\j} &=& \textstyle  \vel{\k} + \sum_{\ind = \k}^{\j-1} \acc{\ind} \Deltak \nonumber \\
  \pos{\j} &=& \textstyle \pos{\k} + \sum_{\ind = \k}^{\j-1} \vel{\ind} \Deltak +\half\sum_{\ind = \k}^{\j-1} \acc{\ind} \Deltak^2 \nonumber \\
  &&  \grayText{(substituting $\vel{\ind}$)}\nonumber \\
  &=& \textstyle \pos{\k} + \sum_{\ind = \k}^{\j-1} ( \vel{\k} + \sum_{\indd = \k}^{\ind-1} \acc{\indd} \Deltak ) \Deltak +\half\sum_{\ind = \k}^{\j-1} \acc{\ind} \Deltak^2 
  \nonumber \\
  &&  \grayText{(moving $\vel{\k}$ outside the sum)}\nonumber \\
  &=& \textstyle \pos{\k} + (\editedTwo{\j\!-\!k\!}) \vel{\k} \Deltak +  \sum_{\ind = \k}^{\j-1} \sum_{\indd = \k}^{\ind-1} \acc{\indd} \Deltak^2 \nonumber \\
  && \textstyle  +\half\sum_{\ind = \k}^{\j-1} \acc{\ind} \Deltak^2 
  \nonumber \\
  &&  \grayText{(\edited{developing the double sum for $\ind = \k,\ldots,\j-1$}}\nonumber \\
  &&  \grayText{\edited{and noting that the summand $i=k$ vanishes})}\nonumber \\
  &=& \textstyle \pos{\k} + (\editedTwo{\j\!-\!k\!}) \vel{\k} \Deltak  \nonumber \\
  &&+  \underbrace{\acc{\k} \Deltak^2}_{\edited{\ind=\k+1}}   \label{eq:muy1}
  +  \underbrace{(\acc{\k}+\acc{\k+1}) \Deltak^2}_{\edited{\ind=\k+2}}  \label{eq:muy2} \\
  && +  \underbrace{(\acc{\k}+\acc{\k+1}+\acc{\k+2}) \Deltak^2}_{\edited{\ind=\k+3}} + \ldots \label{eq:muy3} \\
   &&+ \underbrace{(\acc{\k}+\acc{\k+1}+\ldots+\acc{\j-2}) \Deltak^2)}_{\edited{\ind=\j-1}} \label{eq:muy4} \\
  && \textstyle +\half\sum_{\ind = \k}^{\j-1} \acc{\ind} \Deltak^2 
  %\nonumber \\
  %&&  \grayText{(since the first term with $\k$ appears $\j-\k-2$ times)}\nonumber \\
  %  &&  \grayText{(\edited{in the expanded sum the term $\acc{\j-2}$ appear once,}}\nonumber \\
  % &&  \grayText{\edited{and the term $\acc{\k}$ appears $\j-\k-1$ times, we write is as a single sum})}\nonumber \\
  \eea
  \edited{Now we note that in the expanded sum in lines~\eqref{eq:muy1}-\eqref{eq:muy4} the term $\acc{\k}$ appears $\j-\k-1$ times, 
  the term $\acc{\k+1}$ appears $\j-\k-2$ times, and so on, till the term  $\acc{\j-2}$ which appears just once.  
  % while the term $\acc{\k}$ appears $\j-\k-1$ times. 
  Therefore, we can write lines~\eqref{eq:muy1}-\eqref{eq:muy4} 
  succinctly as a single sum, leading to:}
  \bea
 \pos{\j} &\!\!=\!\!& \textstyle  \pos{\k} + (\editedTwo{\j\!-\!k\!}) \vel{\k} \Deltak +  \sumIndTwo (\j\!-\!\ind\!-\!1) \acc{\ind} \Deltak^2  
 % \nonumber \\ && 
 \textstyle +\half\sum_{\ind = \k}^{\j-1} \acc{\ind} \Deltak^2 
  \nonumber \\
    &&  \grayText{(putting last two terms together)}\nonumber \\
&\!\!=\!\!& \textstyle \pos{\k} + (\editedTwo{\j\!-\!k\!}) \vel{\k} \Deltak + \half \acc{\j-1} \Deltak^2  
  % \nonumber \\ && 
  \textstyle  +  \sumIndTwo (\j\!-\!\ind\!-\!\half) \acc{\ind} \Deltak^2 
  \nonumber \\
    &&  \grayText{(simplifying)}\nonumber \\
&\!\!=\!\!&\textstyle \pos{\k} + (\editedTwo{\j\!-\!k\!}) \vel{\k} \Deltak + \sumInd (\j\!-\!\ind\!-\!\half) \acc{\ind} \Deltak^2 
  \nonumber
\eea
Defining $\Deltaij \doteq (\editedTwo{\j\!-\!k\!})  \Deltak$ and
substituting $\acc{\k}$ from~\eqref{eq:accModel}:
\smaller
\bea
\label{eq:app_integration}
 \vel{\j} &=& \textstyle  \vel{\k} + \sum_{\ind = \k}^{\j-1} \left(  \Rot{\ind} ( \measAcc{\ind}\!-\! \biasAcc{\k} \!-\! \noiseAcc{\ind}) + \gravity \right)  \Deltak
  \nonumber \\
   &=& \textstyle  \vel{\k} + \gravity \Deltaij - ( \sum_{\ind = \k}^{\j-1} \Rot{\ind} \Deltak) \biasAcc{\k}  + \sum_{\ind = \k}^{\j-1} \edited{\Rot{\ind}} \measAcc{\ind}  \Deltak
    \nonumber \\
   && \textstyle- \sum_{\ind = \k}^{\j-1} \edited{\Rot{\ind}} \noiseAcc{\ind} \Deltak
  \nonumber \\
  %%% POSITION
  \pos{\j} &=& \textstyle \pos{\k} + \vel{\k} \Deltaij +  \sumInd (\j\!-\!\ind\!-\!\half) \left(  \Rot{\ind} ( \measAcc{\ind}\!-\! \biasAcc{\k} \!-\! \noiseAcc{\ind}) + \gravity \right)  \Deltak^2 
  \nonumber \\
  &=& \textstyle \pos{\k} + \vel{\k} \Deltaij +  \sumInd (\j\!-\!\ind\!-\!\half)  \gravity   \Deltak^2 
  \nonumber \\
  && \textstyle
   -  (\sumInd (\j\!-\!\ind\!-\!\half)  \Rot{\ind} \Deltak^2)  \biasAcc{\k} 
  +  \sumInd (\j\!-\!\ind\!-\!\half)  \Rot{\ind} ( \measAcc{\ind}\!-\! \noiseAcc{\ind})  \Deltak^2 
  \nonumber \\
\eea

\normalsize
Let us now define the following vectors:
\bea
\meas^\V_{\k\j} &\doteq& \textstyle\gravity \Deltaij  + \sum_{\ind = \k}^{\j-1} \edited{\Rot{\ind}} \measAcc{\ind}  \Deltak 
\nonumber \\
\noise^\V_{\k\j} &\doteq& \textstyle \sum_{\ind = \k}^{\j-1} \edited{\Rot{\ind}} \noiseAcc{\ind} \Deltak 
\nonumber \\
%% POSITION
\meas^\P_{\k\j} &\doteq& \textstyle  \sumInd (\j\!-\!\ind\!-\!\half)  \gravity   \Deltak^2    
\nonumber \\
&& \textstyle  +  \sumInd (\j\!-\!\ind\!-\!\half)  \Rot{\ind} \measAcc{\ind} \Deltak^2 
\nonumber \\
\noise^\P_{\k\j} &\doteq& \textstyle  + \sumInd (\j\!-\!\ind\!-\!\half)  \Rot{\ind} \noiseAcc{\ind}  \Deltak^2 
\eea
Using this notation we rewrite eq.~\eqref{eq:app_integration} (putting position first) and adding the 
random walk random model on the bias: 
\bea
\label{eq:app_integration2}
%%% POSITION
  \pos{\j} \!\!&=&\!\! \textstyle \pos{\k} + \vel{\k} \Deltaij - \left( \sumInd (\j\!-\!\ind\!-\!\half)  \Rot{\ind} \Deltak^2) \right) \biasAcc{\k}
  %\nonumber \\
  %&& 
  \textstyle  + \meas^\P_{\k\j} - \noise^\P_{\k\j}
  \nonumber \\
  %%% VELOCITY
 \vel{\j} \!\!&=&\!\! \textstyle  \vel{\k} - ( \sum_{\ind = \k}^{\j-1} \Rot{\ind} \Deltak) \biasAcc{\k}  + \meas^\V_{\k\j} - \noise^\V_{\k\j}
  \nonumber \\
  %%% BIAS
 \biasAcc{\j} \!\!&=&\!\! \textstyle \biasAcc{\k} - \noise^\B_{\k\j}
\eea
In order to write~\eqref{eq:app_integration2} in compact matrix form, we define:
\bea
\MN_{\k\j} &\doteq& \textstyle\sumInd (\j\!-\!\ind\!-\!\half)  \Rot{\ind} \Deltak^2 \\ 
\MM_{\k\j} &\doteq& \textstyle \sum_{\ind = \k}^{\j-1} \Rot{\ind} \Deltak 
\eea
which allows rewriting~\eqref{eq:app_integration2} succinctly as:
\bea
\label{eq:measModel1_app}
%%% POSITION
\meas^\P_{\k\j}  &=&   \textstyle \pos{\j} - \pos{\k} - \vel{\k} \Deltaij + \MN_{\k\j}  \biasAcc{\k} + \noise^\P_{\k\j} 
  \nonumber \\
  %%% VELOCITY
  \meas^\V_{\k\j} &=&   \textstyle \vel{\j} -  \vel{\k} +\MM_{\k\j} \biasAcc{\k} + \noise^\V_{\k\j} 
  \nonumber \\
  %%% BIAS
  \meas^\B_{\k\j} &=& \textstyle \biasAcc{\j}  - \biasAcc{\k} + \noise^\B_{\k\j}
\eea
where \edited{$\meas^\B_{\k\j} = \zero_3$} is the expected change in the bias.

Let us now define the following matrices and vectors:

\small
\bea
\label{eq:matsDefs}
\MA_{{\k\j}} \!\!&=&\!\!
\left[
\begin{array}{c | c | ccc | c | c | c }
                   &          &   -\eye_3  & -\eye_3 \Deltaij & \MN_{\k\j} &          &                      &          \\
\zero_{9 \times 9} &  \ldots  &   \zero    & -\eye_3          & \MM_{\k\j} &  \eye_9  &  \zero_{9 \times 9}  &   \ldots  \\
                   &          &   \zero    & \zero            &   -\eye_3  &          &                      & 
\end{array}
\right]
\nonumber \\
\measImu{\k\j} \!\!&=&\!\!
\left[
\begin{array}{c}
\meas^\P_{\k\j} \\ \meas^\V_{\k\j} \\ \meas^\B_{\k\j}
\end{array}
\right]
\qquad \qquad
%% NOISE
\noiseImu{\k\j} =
\left[
\begin{array}{c}
\noise^\P_{\k\j} \\ \noise^\V_{\k\j} \\ \noise^\B_{\k\j}
\end{array}
\right]
\eea

\normalsize
\noindent
Using~\eqref{eq:matsDefs}, we finally rewrite 
our model~\eqref{eq:measModel1_app}  as:
\bea
\label{eq:measModelLin_app}
\measImu{\k\j} = \MA_{{\k\j}} \vxx_{\k:\k+\hor} + \noiseImu{\k\j}
\eea
To fully characterize the linear measurement model~\eqref{eq:measModelLin_app} 
we only have to compute the covariance of the noise
$\noiseImu{\k\j}$, which is given by:
\bea
\cov(\noiseImu{\k\j}) = 
\left[
\begin{array}{cc}
\sigmaImu^2 \MC \MC\tran    & \zero_{6 \times 3} \\ 
 \zero_{3 \times 6} &  \cov(\noise^\B_{\k\j})
\end{array}
\right]
\eea
where $\MC$ includes the coefficient matrices of the noise in ~\eqref{eq:app_integration2}:
\smaller
\bea
\MC  = 
\left[
\begin{array}{cccccccc}
(\j\!-\!\k\!-\!\frac{1}{2})  
\Rot{k} \Deltak^2 &    (\j\!-\!\k\!-\!\frac{3}{2}) \Rot{k+1} \Deltak^2     & \ldots &  \half \Rot{j-1} \Deltak^2
\\
\Rot{k} \Deltak         & \Rot{k+1} \Deltak & \ldots &  \Rot{j-1} \Deltak
\end{array}
\right] \nonumber
\eea

\normalsize
Using the fact that any rotation matrix satisfies $\Rot{k}\tran \Rot{k} = \eye_3$, the matrix $\MC \MC\tran$ 
can be computed simply as:
\smaller
\bea
\MC \MC\tran = 
\left[
\begin{array}{cc}
 \left( \sumInd (\j\!-\!\ind\!-\!\half)^2 \right) \Deltak^4 \eye_3 &            
 \left( \sumInd (\j\!-\!\ind\!-\!\half) \right) \Deltak^3 \eye_3
         \\
  \left( \sumInd (\j\!-\!\ind\!-\!\half) \right) \Deltak^3 \eye_3
                      & (\j\!-\!\k\!-\!1) \Deltak^2 \eye_3 
\end{array}
\right]. \nonumber
\eea

\normalsize

%% file: appendix-proofLongerFeatureTracks.tex
%!TEX root = main.tex

\subsection{Proof of~\prettyref{prop:longerFeatureTracks}}
\label{app:proof:prop:longerFeatureTracks}

The information matrix of the joint state~\eqref{eq:infoMatrixPoint} is additive in the measurements, 
hence the information matrices of the joint states $[\vxx_{\k:\k+\hor} \; \p_{\l_1}]$ 
and $[\vxx_{\k:\k+\hor} \; \p_{\l_2}]$ given the predicted visual measurements 
to landmarks $\l_1$ and $\l_2$ can be respectively written as:
\bea
\MOmega^{(\l_1)}_{\k:\hor} = \sum_{\tau = k}^{k_1} \MOmega^{(\l_1)}_{\tau}
\qquad 
\MOmega^{(\l_2)}_{\k:\hor} = \sum_{\tau = k}^{k_2} \MOmega^{(\l_2)}_{\tau}
\eea
where $\MOmega^{(\l_1)}_{\tau}$ (resp. $\MOmega^{(\l_2)}_{\tau}$) is the contribution to the information matrix corresponding to the measurement of landmark 1 (resp. 2) at time $\tau$. Since the proposition assumes that the future measurements are identical, it follows $\sum_{\tau = k}^{k_1} \MOmega^{(\l_1)}_{\tau} = \sum_{\tau = k}^{k_1} \MOmega^{(\l_2)}_{\tau}$; this, combined with the fact that $k_2 > k_1$ implies:
\bea
\MOmega^{(\l_2)}_{\k:\hor} = \MOmega^{(\l_1)}_{\k:\hor} + \sum_{\tau = k_1 + 1}^{k_2} \MOmega^{(\l_2)}_{\tau} \succeq  \MOmega^{(\l_1)}_{\k:\hor} 
\eea
Now we only have to prove that the Schur complement preserves the ordering $\MOmega^{(\l_2)}_{\k:\hor} \succeq  \MOmega^{(\l_1)}_{\k:\hor}$, 
since $\Delta_{\l_1}$ and $\Delta_{\l_2}$ are simply the Schur complements of $\MOmega^{(\l_1)}_{\k:\hor}$ and $\MOmega^{(\l_2)}_{\k:\hor}$, respectively.
For this purpose, we first  observe that
\bea
\label{eq:invertInequality}
\MOmega^{(\l_2)}_{\k:\hor} \succeq  \MOmega^{(\l_1)}_{\k:\hor} \Longrightarrow 
\left( \MOmega^{(\l_2)}_{\k:\hor} \right)\inv \preceq  \left( \MOmega^{(\l_1)}_{\k:\hor} \right)\inv
\eea
Moreover, we make explicit the block structure of the two matrices as follows: 
\bea
\label{eq:blocks}
\MOmega^{(\l_1)}_{\k:\hor} \doteq \matTwo{\MA_1 & \MB_1 \\ \MB_1\tran & \MC_1}  \qquad
\MOmega^{(\l_2)}_{\k:\hor} \doteq \matTwo{\MA_2 & \MB_2 \\ \MB_2\tran & \MC_2} 
\eea
where the upper-left blocks ($\MA_1$ and $\MA_2$) correspond to entries of the information matrix associated to the states $\vxx_{\k:\k+\hor}$ 
and the bottom-right 
blocks ($\MC_1$ and $\MC_2$) correspond to the landmark states we want to marginalize.

Now we note that using standard block inversion for the block matrix $\MOmega^{(\l_1)}_{\k:\hor}$ we obtain:
\bea
\label{eq:blockInverse}
\left( \MOmega^{(\l_1)}_{\k:\hor} \right)\inv = \matTwo{ (\MA_1  - \MB_1 \editedTwo{\MC_1\inv} \MB_1\tran)\inv  & \star \\ \star & \star}
\eea
where we denoted with ``$\star$'' blocks which are irrelevant for the following derivation.
Combining the inequality~\eqref{eq:invertInequality} with the block inverse~\eqref{eq:blockInverse} we get:
 
 \vspace{-0.3cm}
 \smaller
 \bea \label{eq:block11}
 \editedTwo{
 \matTwo{ (\MA_2  - \MB_2 \editedTwo{\MC_2\inv} \MB_2\tran)\inv  & \star \\ \star & \star}
\preceq 
\matTwo{ (\MA_1  - \MB_1 \editedTwo{\MC_1\inv} \MB_1\tran)\inv  & \star \\ \star & \star} 
}
\eea

\normalsize
\noindent
\editedTwo{Since diagonal blocks of positive semidefinite matrices are also semidefinite, eq.~\eqref{eq:block11}
 implies $ (\MA_2  - \MB_2 \MC_2\inv \MB_2\tran)\inv \preceq (\MA_1  - \MB_1 \MC_1\inv \MB_1\tran)\inv$} 
hence:
 \bea
 \label{eq:SchurPreservesOrder}
 \editedTwo{
 (\MA_2  - \MB_2 \MC_2\inv \MB_2\tran) \succeq
 (\MA_1  - \MB_1 \MC_1\inv \MB_1\tran).
 }
 \eea 
Comparing the block structure in~\eqref{eq:blocks} with the description in eqs.~\eqref{eq:infoMatrixPoint}-\eqref{eq:infoMatrixVision}, 
we realize that $(\MA_1  - \MB_1 \editedTwo{\MC_1\inv} \MB_1\tran)  = \Delta_{\l_1}$ and $(\MA_2  - \MB_2 \editedTwo{\MC_2\inv} \MB_2\tran)  = \Delta_{\l_2}$ 
hence~\eqref{eq:SchurPreservesOrder} implies $\Delta_{\l_1} \preceq \Delta_{\l_2}$ concluding the proof.

%% file: appendix-proofUppenBoundLambdaMin.tex
%!TEX root = main.tex

\subsection{Proof of Corollary~\ref{prop:boundEigMin}}
\label{app:proof:prop:boundEigMin}

The proof relies on the inequality~\eqref{eq:ipsen} for $i$ chosen to be the smallest eigenvalue.
%Moreover, 
From the Weyl inequality~\cite{Ipsen09siam-eigenvaluePerturbation}, it follows  $\lambda_j(\MM+\Delta) \geq \lambdaMin(\MM)$, 
for all $j$. Using this fact, it follows that the minimum in~\eqref{eq:ipsen} is attained by $\lambdaMin(\MM+\Delta)$.
Therefore, the inequality~\eqref{eq:ipsen}  becomes:
 \bea
 \label{eq:eigBound1}
|\lambdaMin(\MM) - \lambdaMin(\MM+\Delta)| \leq \|\Delta \vv_\min \| 
\eea
From the positive definiteness of $\MM$ and $\Delta$ (which implies $\lambdaMin(\MM) \geq 0$ and $\lambdaMin(\MM+\Delta) \geq 0$),
and from the Weyl inequality, it follows  $|\lambdaMin(\MM) - \lambdaMin(\MM+\Delta)| = \lambdaMin(\MM+\Delta) - \lambdaMin(\MM)$, 
which substituted in~\eqref{eq:eigBound1} leads to~\eqref{eq:boundEigMin}. 

%% file: appendix-submodularityRatio.tex
%!TEX root = main.tex

\subsection{Proof of~\prettyref{prop:nonVanishingRatio}}
\label{app:proof:prop:nonVanishingRatio}

In order to show that the submodularity ratio~\eqref{eq:subRatio} does not vanish, we show that 
its numerator is bounded away from zero. 
To do so, we consider a single summand in~\eqref{eq:subRatio}:
\bea
\label{eq:num1}
(a) \doteq f(\setL \cup \{e\}) - f(\setL)  = 
\lambdaMin(\MOmega_\setL + \Delta_e) - \lambdaMin(\MOmega_\setL)
\eea
where $\MOmega_\setL \doteq \barMOmega_{\k:\k+\hor} + \sum_{\l \in \setL} \Delta_\l$.
Our task is to prove that~\eqref{eq:num1} is different from zero.
To do so, we substitute the eigenvalue with its definition through the Rayleigh quotient:
\beal
(a) = \min_{\|\vmu\|=1} \vmu\tran (\MOmega_\setL + \Delta_e) \vmu - 
\min_{\|\vnu\|=1}\vnu\tran \edited{(\MOmega_\setL)} \vnu \nonumber \\
\grayText{(calling $\vmubar$ the minimizer of the first summand)} \nonumber \\
= \vmubar\tran (\MOmega_\setL + \Delta_e) \vmubar - 
\min_{\|\vnu\|=1}\vnu\tran (\MOmega_\setL) \vnu \nonumber \\
\grayText{(\edited{since $\vnu = \vmubar$ is suboptimal for the second summand})} \nonumber \\
\geq \vmubar\tran (\MOmega_\setL + \Delta_e) \vmubar - 
\vmubar\tran  (\MOmega_\setL) \vmubar \nonumber \\
\grayText{(simplifying and substituting the expression of $\Delta_e$ from~\eqref{eq:infoMatrixVision})} \nonumber \\
= \vmubar\tran \Delta_e \vmubar = 
\vmubar\tran  \MF_e\tran (\eye - \ME_e (\ME_e\tran \ME_e)\inv  \ME_e\tran) \MF_e \vmubar
\nonumber \\
\grayText{(defining the idempotent matrix $\MQ_e \doteq (\eye - \ME_e (\ME_e\tran \ME_e)\inv  \ME_e\tran)$)} \nonumber \\
= \vmubar\tran  \MF_e\tran \MQ_e\MF_e \vmubar  = 
\vmubar\tran  \MF_e\tran \MQ_e \MQ_e \MF_e \vmubar  =
 \| \MQ_e \MF_e \vmubar \|^2  \nonumber 
 \label{eq:nonVanishingProof}
\eeal
Now we write $\ME_e$ in terms of its $3\times 3$ blocks: 
\bea
\label{eq:Femu}
\ME_e = 
\vect{
 \ME_{e0}
\\ 
\ME_{e1} 
\\
\vdots 
\\
\ME_{e\hor}
}
\eea
Moreover,  we recall that $\MF_e$ has the following block structure:
\bea
\label{eq:Femu2}
\MF_e = 
\left[
\begin{array}{ccc|ccc|ccc|c}
 - \ME_{e0}  & \myZ & \myZ       & \myZ & \myZ & \myZ     &\myZ & \myZ & \myZ       & \cdots \\
\myZ & \myZ & \myZ               & \editedTwo{- \ME_{e1}}  & \myZ  & \myZ      & \myZ & \myZ & \myZ   & \cdots  \\ 
\myZ & \myZ & \myZ               & \myZ & \myZ & \myZ      & \editedTwo{- \ME_{e2}} & \myZ &   \myZ   & \cdots \\ 
\vdots & \vdots & \vdots         & \vdots & \vdots & \vdots   & \vdots & \vdots & \vdots & \ddots \
\end{array}
\right] \nonumber
\eea
where we noted that the nonzero blocks in $\MF_e$ are the same (up-to-sign) as the ones in $\ME_e$ 
(\emph{c.f.} the coefficient matrices in~\eqref{eq:visionModel2}).
It follows that $\MF_e \vmubar$ can be written explicitly as:
\bea
\label{eq:Femu3}
\MF_e \vmubar = 
\vect{
-\ME_{e0} \vmubar_0 
\\ 
- \ME_{e1}  \vmubar_1
\\
\vdots 
\\
-\ME_{e\hor} \vmubar_\hor
}
\eea
Now we observe that $\MQ_e$ is an orthogonal projector onto the null space of 
$\ME_e$, and the null space of $\MQ_e$ is spanned by the columns of $\ME_e$. 
Therefore, any vector $\vv$ that falls in the null space of $\MQ_e$ can be 
written as a linear combination of the columns of $\ME_e$:
\bea
\MQ_e \vv = \zero \Leftrightarrow  \vv =  \ME_e \vw
\eea
with $\vw \in \Real{3}$. By comparison with~\eqref{eq:Femu}, we note that  
$\MF_e \vmubar$ can be written as $\ME_e \vw$ if and only if $\vmubar_1 = \vmubar_2 = \ldots = \vmubar_\hor$.
Therefore, if $\vmubar_i \neq \vmubar_j$ for some $i, j$, then 
the vector $\MF_e \vmubar$ cannot be in the null space of $\MQ_e$, and the lower bound~\eqref{eq:nonVanishingProof} 
must be greater than zero, concluding the proof.

%% file: main.bbl
\begin{thebibliography}{10}
\providecommand{\url}[1]{#1}
\csname url@samestyle\endcsname
\providecommand{\newblock}{\relax}
\providecommand{\bibinfo}[2]{#2}
\providecommand{\BIBentrySTDinterwordspacing}{\spaceskip=0pt\relax}
\providecommand{\BIBentryALTinterwordstretchfactor}{4}
\providecommand{\BIBentryALTinterwordspacing}{\spaceskip=\fontdimen2\font plus
\BIBentryALTinterwordstretchfactor\fontdimen3\font minus
  \fontdimen4\font\relax}
\providecommand{\BIBforeignlanguage}[2]{{%
\expandafter\ifx\csname l@#1\endcsname\relax
\typeout{** WARNING: IEEEtran.bst: No hyphenation pattern has been}%
\typeout{** loaded for the language `#1'. Using the pattern for}%
\typeout{** the default language instead.}%
\else
\language=\csname l@#1\endcsname
\fi
#2}}
\providecommand{\BIBdecl}{\relax}
\BIBdecl

\bibitem{Potter14app-attention}
M.~Potter, B.~Wyble, C.~Hagmann, and E.~McCourt, ``Detecting meaning in {RSVP}
  at 13 ms per picture,'' \emph{Attention, Perception, \& Psychophysics},
  vol.~76, no.~2, pp. 270--279, 2014.

\bibitem{Carrasco11vr-attentionSurvey}
M.~Carrasco, ``Visual attention: The past 25 years,'' \emph{Vision Research},
  vol.~51, pp. 1484--1525, 2011.

\bibitem{Redmon16cvpr}
J.~Redmon, S.~Divvala, R.~Girshick, and A.~Farhadi, ``You only look once:
  Unified, real-time object detection,'' in \emph{IEEE Conf. on Computer Vision
  and Pattern Recognition (CVPR)}, 2016.

\bibitem{Pillai16icra}
S.~Pillai, S.~Ramalingam, and J.~Leonard, ``High-performance and tunable stereo
  reconstruction,'' in \emph{IEEE Intl. Conf. on Robotics and Automation
  (ICRA)}, 2016.

\bibitem{Mur-Artal15tro}
R.~Mur-Artal, J.~Montiel, and J.~Tard{\'o}s, ``{ORB-SLAM}: A versatile and
  accurate monocular {SLAM} system,'' \emph{{IEEE} Trans. Robotics}, vol.~31,
  no.~5, pp. 1147--1163, 2015.

\bibitem{TitanXwebsite}
\BIBentryALTinterwordspacing
{NVIDIA GeForce Website}, ``Geforce {GTX TITAN X} specifications.'' [Online].
  Available:
  \url{http://www.geforce.com/hardware/desktop-gpus/geforce-gtx-titan-x/specifications}
\BIBentrySTDinterwordspacing

\bibitem{processorsWebsite}
\BIBentryALTinterwordspacing
Wikipedia, ``List of cpu power dissipation figures.'' [Online]. Available:
  \url{https://en.wikipedia.org/wiki/List_of_CPU_power_dissipation_figures}
\BIBentrySTDinterwordspacing

\bibitem{Zhang17rss-vioChip}
Z.~Zhang, A.~Suleiman, L.~Carlone, V.~Sze, and S.~Karaman, ``Visual-inertial
  odometry on chip: An algorithm-and-hardware co-design approach,'' in
  \emph{Robotics: Science and Systems (RSS)}, 2017,
  \linkToPdf{http://rss2017.lids.mit.edu/static/papers/74.pdf}
  \linkToWeb{http://navion.mit.edu/}, highlighted in the MIT News:
  \linkToWeb{http://news.mit.edu/2017/miniaturizing-brain-smart-drones-0712}.

\bibitem{Cavanagh11vr-visualCognition}
\BIBentryALTinterwordspacing
P.~Cavanagh, ``Visual cognition,'' \emph{Vision Research}, vol.~51, no.~13, pp.
  1538 -- 1551, 2011, vision Research 50th Anniversary Issue: Part 2. [Online].
  Available:
  \url{http://www.sciencedirect.com/science/article/pii/S0042698911000381}
\BIBentrySTDinterwordspacing

\bibitem{Carlone17icra-vioAttention}
L.~Carlone and S.~Karaman, ``Attention and anticipation in fast visual-inertial
  navigation,'' in \emph{IEEE Intl. Conf. on Robotics and Automation (ICRA)},
  2017, pp. 3886--3893, extended arxiv preprint: 1610.03344
  \linkToPdf{https://arxiv.org/pdf/1610.03344.pdf}.

\bibitem{Burri16ijrr-eurocDataset}
M.~Burri, J.~Nikolic, P.~Gohl, T.~Schneider, J.~Rehder, S.~Omari, M.~Achtelik,
  and R.~Siegwart, ``The {EuRoC} micro aerial vehicle datasets,'' \emph{Intl.
  J. of Robotics Research}, 2016.

\bibitem{Borji13pami-surveyAttention}
A.~Borji and L.~Itti, ``State-of-the-art in visual attention modeling,''
  \emph{{IEEE} Trans. Pattern Anal. Machine Intell.}, vol.~35, no.~1, 2013.

\bibitem{Scholl01cog-attentionSurvey}
B.~Scholl, ``Objects and attention: the state of the art,'' \emph{Cognition},
  vol.~80, no.~1, pp. 1--46, 2001.

\bibitem{Caduff08cp-attention}
D.~Caduff and S.~Timpf, ``On the assessment of landmark salience for human
  navigation,'' \emph{Cogn. Process}, vol.~9, pp. 249--267, 2008.

\bibitem{Moran85science-attention}
J.~Moran and J.~Desimone, ``Selective attention gates visual processing in the
  extrastriate cortex,'' \emph{Science}, vol. 229, no. 4715, pp. 782--784,
  1985.

\bibitem{Wolfe94pbr-attention}
J.~Wolfe, ``Guided search 2.0 - a revised model of visual search,''
  \emph{Psychon Bull. Rev.}, vol.~1, no.~2, pp. 202--238, 1994.

\bibitem{Spekreijse00vr-attention}
H.~Spekreijse, ``Pre-attentive and attentive mechanisms in vision. perceptual
  organization and dysfunction,'' \emph{Vision Research}, vol.~40, no. 10-12,
  pp. 1179--1182, 2000.

\bibitem{Sim99iccv-featureSelection}
R.~Sim and G.~Dudek, ``Learning and evaluating visual features for pose
  estimation,'' in \emph{Intl. Conf. on Computer Vision (ICCV)}, 1999, pp.
  1217--1222.

\bibitem{Peretroukhin15iros-predictiveVIO}
V.~Peretroukhin, L.~Clement, M.~Giamou, and J.~Kelly, ``{PROBE}: Predictive
  robust estimation for visual-inertial navigation,'' in \emph{IEEE/RSJ Intl.
  Conf. on Intelligent Robots and Systems (IROS)}, 2015.

\bibitem{Ouerhani05ecmr-attention}
N.~Ouerhani, A.~Bur, and H.~H\"ugli, ``Visual attention-based robot
  self-localization,'' in \emph{ECMR}, 2005, pp. 8--13.

\bibitem{Newman05icra-salientFeatures}
P.~Newman and K.~Ho, ``Slam-loop closing with visually salient features,'' in
  \emph{IEEE Intl. Conf. on Robotics and Automation (ICRA)}, 2005, pp.
  635--642.

\bibitem{Sala06tro-landmarkSelection}
P.~Sala, R.~Sim, A.~Shokoufandeh, and S.~Dickinson, ``Landmark selection for
  vision-based navigation,'' \emph{{IEEE} Trans. Robotics}, vol.~22, no.~2, pp.
  334--349, 2006.

\bibitem{Siagian07iros-attention}
C.~Siagian and L.~Itti, ``Biologically-inspired robotics vision monte-carlo
  localization in the outdoor environment,'' in \emph{IEEE/RSJ Intl. Conf. on
  Intelligent Robots and Systems (IROS)}, 2007, pp. 1723--1730.

\bibitem{Frintrop08tro-attentionLandmarks}
S.~Frintrop and P.~Jensfelt, ``Attentional landmarks and active gaze control
  for visual {SLAM},'' \emph{{IEEE} Trans. Robotics}, vol.~24, no.~5, pp.
  1054--1065, 2008.

\bibitem{Hochdorfer09icra-landmarksSelection}
S.~Hochdorfer and C.~Schlegel, ``Landmark rating and selection according to
  localization coverage: Addressing the challenge of lifelong operation of slam
  in service robots,'' in \emph{IEEE/RSJ Intl. Conf. on Intelligent Robots and
  Systems (IROS)}, 2009, pp. 382--387.

\bibitem{Strasdat09icra-landmarksSelection}
H.~Strasdat, C.~Stachniss, and W.~Burgard, ``Which landmark is useful? learning
  selection policies for navigation in unknown environments,'' in \emph{IEEE
  Intl. Conf. on Robotics and Automation (ICRA)}, 2009, pp. 1410--1415.

\bibitem{Chli09ras}
M.~Chli and A.~Davison, ``Active matching for visual tracking,'' \emph{Robotics
  and Autonomous Systems}, vol.~57, no.~12, pp. 1173--1187, Dec. 2009.

\bibitem{Handa10cvpr}
A.~Handa, M.~Chli, H.~Strasdat, and A.~J. Davison, ``Scalable active
  matching,'' in \emph{{IEEE} Conference on Computer Vision and Pattern
  Recognition}, 2010.

\bibitem{Jang13icufn-featureSelection}
J.~Jang, K.~Won, and S.~Jung, ``Geometric feature selection for vehicle pose
  estimation on dynamic road scenes,'' in \emph{5th Intl. Conf. on Ubiquitous
  and Future Networks}, 2013.

\bibitem{Shi13mva-featureSelection}
Z.~Shi, Z.~Liu, X.~Wu, and W.~Xu, ``Feature selection for reliable data
  association in visual {SLAM},'' \emph{Machine Vision and Applications},
  vol.~24, pp. 667--682, 2013.

\bibitem{Oliva01ijcv}
A.~Oliva and A.~Torralba, ``Modeling the shape of the scene: a holistic
  representation of the spatial envelope,'' \emph{Intl. J. of Computer Vision},
  vol.~42, pp. 145--175, 2001.

\bibitem{Torralba06pr}
A.~Torralba, A.~Oliva, M.~Castelhano, and J.~Henderson, ``Contextual guidance
  of attention in natural scenes: the role of global features on object
  search,'' \emph{Psychological Review}, vol. 113, no.~4, pp. 766--786, 2006.

\bibitem{Mnih14arxiv-attention}
V.~Mnih, N.~Heess, A.~Graves, and K.~Kavukcuoglu, ``Recurrent models of visual
  attention,'' in \emph{Advances in Neural Information Processing Systems
  (NIPS)}, 2014, pp. 2204--2212.

\bibitem{Xu16arxiv-attention}
K.~Xu, J.~L. Ba, R.~Kiros, K.~Cho, A.~Courville, R.~Salakhutdinov, R.~Zemel,
  and Y.~Bengio, ``Show, attend and tell: Neural image caption generation with
  visual attention,'' in \emph{ArXiv preprint: 1502.03044}, 2016.

\bibitem{Cvisic15ecmr-featureSelection}
I.~Cvi\v{s}i\'c and I.~Petrovi\'c, ``Stereo odometry based on careful feature
  selection and tracking,'' in \emph{Proc.~of the European Conference on Mobile
  Robots (ECMR)}, 2015.

\bibitem{Cadena16tro-SLAMsurvey}
C.~Cadena, L.~Carlone, H.~Carrillo, Y.~Latif, D.~Scaramuzza, J.~Neira, I.~Reid,
  and J.~J. Leonard, ``Past, present, and future of simultaneous localization
  and mapping: Toward the robust-perception age,'' \emph{{IEEE} Trans.
  Robotics}, vol.~32, no.~6, pp. 1309--1332, 2016, arxiv preprint: 1606.05830,
  \linkToPdf{https://arxiv.org/abs/1606.05830}.

\bibitem{Davison05iccv}
A.~Davison, ``Active search for real-time vision,'' in \emph{Intl. Conf. on
  Computer Vision (ICCV)}, Oct 2005.

\bibitem{Lerner07tro}
R.~Lerner, E.~Rivlin, and I.~Shimshoni, ``Landmark selection for task-oriented
  navigation,'' \emph{{IEEE} Trans. Robotics}, vol.~23, no.~3, pp. 494--505,
  2007.

\bibitem{Mu15rss}
B.~Mu, A.~Agha-mohammadi, L.~Paull, M.~Graham, J.~How, and J.~Leonard,
  ``Two-stage focused inference for resource-constrained collision-free
  navigation,'' in \emph{Robotics: Science and Systems (RSS)}, 2015.

\bibitem{Wu16iser-vins}
K.~Wu, T.~Do, L.~C. Carrillo-Arce, and S.~I. Roumeliotis, ``On the {VINS}
  resource-allocation problem for a dual-camera, small-size quadrotor,'' in
  \emph{Intl. Sym. on Experimental Robotics (ISER)}, 2016, pp. 538--549.

\bibitem{Zhang15cvpr}
G.~Zhang and P.~Vela, ``Good features to track for visual slam,'' in \emph{IEEE
  Conf. on Computer Vision and Pattern Recognition (CVPR)}, 2015.

\bibitem{Kottas14icra-vins}
D.~Kottas, R.~DuToit, A.~Ahmed, C.~Guo, G.~Georgiou, R.~Li, and S.~Roumeliotis,
  ``A resource-aware vision-aided inertial navigation system for wearable and
  portable computers,'' in \emph{IEEE Intl. Conf. on Robotics and Automation
  (ICRA)}, 2014, pp. 6336--6343.

\bibitem{Joshi09tsp-sensorSelection}
S.~Joshi and S.~Boyd, ``Sensor selection via convex optimization,''
  \emph{{IEEE} Trans. Signal Processing}, vol.~57, pp. 451--462, 2009.

\bibitem{Giraud95iros-sensorSelection}
C.~Giraud and B.~Jouvencel, ``Sensor selection: A geometrical approach,'' in
  \emph{IEEE/RSJ Intl. Conf. on Intelligent Robots and Systems (IROS)}, vol.~2,
  1995, pp. 1410--1415.

\bibitem{Krause08jmlr-submodularity}
A.~Krause, A.~Singh, and C.~Guestrin, ``Near-optimal sensor placements in
  gaussian processes: Theory, efficient algorithms and empirical studies,''
  \emph{Journal of Machine Learning Research}, vol.~9, pp. 235--284, 2008.

\bibitem{Bian06ipsn-sensorSelection}
F.~Bian, D.~Kempe, and R.~Govindan, ``Utility based sensor selection,'' in
  \emph{{5th Int. Conf. Information Processing Sensor Networks}}, 2006, pp.
  11--18.

\bibitem{Shamaiah10cdc-sensorScheduling}
M.~Shamaiah, S.~Banerjee, and H.~Vikalo, ``Greedy sensor selection: leveraging
  submodularity,'' in \emph{IEEE Conf. on Decision and Control (CDC)}, 2010.

\bibitem{Vitus08automatica-sensorScheduling}
M.~Vitus, W.~Zhang, A.~Abate, J.~Hu, and C.~Tomlin, ``On efficient sensor
  scheduling for linear dynamical systems,'' \emph{Automatica}, vol.~48, pp.
  2482--2493, 2012.

\bibitem{Zhang15cdc-sensorSelection}
H.~Zhang, R.~Ayoub, and S.~Sundaram, ``Sensor selection for optimal filtering
  of linear dynamical systems: Complexity and approximation,'' in \emph{IEEE
  Conf. on Decision and Control (CDC)}, 2015.

\bibitem{Jawaid15automatica-sensorScheduling}
S.~Jawaid and S.~Smith, ``Submodularity and greedy algorithms in sensor
  scheduling for linear dynamical systems,'' \emph{Automatica}, vol.~61, pp.
  282--288, 2015.

\bibitem{Tzoumas16acc-sensorScheduling}
V.~Tzoumas, A.~Jadbabaie, and G.~Pappas, ``Sensor placement for optimal kalman
  filtering: Fundamental limits, submodularity, and algorithms,'' in
  \emph{American Control Conference}, 2016.

\bibitem{Summers16tcns-sensorScheduling}
T.~Summers, F.~Cortesi, and J.~Lygeros, ``On submodularity and controllability
  in complex dynamical networks,'' \emph{IEEE Transactions on Control of
  Network Systems}, vol.~3, no.~1, pp. 91--101, 2016.

\bibitem{Mourikis07icra}
A.~Mourikis and S.~Roumeliotis, ``A multi-state constraint {K}alman filter for
  vision-aided inertial navigation,'' in \emph{IEEE Intl. Conf. on Robotics and
  Automation (ICRA)}, April 2007, pp. 3565--3572.

\bibitem{Davison07pami}
A.~Davison, I.~Reid, N.~Molton, and O.~Stasse, ``Mono{SLAM}: Real-time single
  camera {SLAM},'' \emph{{IEEE} Trans. Pattern Anal. Machine Intell.}, vol.~29,
  no.~6, pp. 1052--1067, Jun 2007.

\bibitem{Blosch15iros}
M.~Bloesch, S.~Omari, M.~Hutter, and R.~Siegwart, ``Robust visual inertial
  odometry using a direct {EKF}-based approach,'' in \emph{IEEE/RSJ Intl. Conf.
  on Intelligent Robots and Systems (IROS)}.\hskip 1em plus 0.5em minus
  0.4em\relax IEEE, 2015.

\bibitem{Jones11ijrr}
E.~Jones and S.~Soatto, ``Visual-inertial navigation, mapping and localization:
  A scalable real-time causal approach,'' \emph{Intl. J. of Robotics Research},
  vol.~30, no.~4, Apr 2011.

\bibitem{Hesch14ijrr}
J.~Hesch, D.~Kottas, S.~Bowman, and S.~Roumeliotis, ``Camera-imu-based
  localization: Observability analysis and consistency improvement,''
  \emph{Intl. J. of Robotics Research}, vol.~33, no.~1, pp. 182--201, 2014.

\bibitem{Mourikis08wvlmp}
A.~Mourikis and S.~Roumeliotis, ``A dual-layer estimator architecture for
  long-term localization,'' in \emph{Proc. of the Workshop on Visual
  Localization for Mobile Platforms at CVPR}, Anchorage, Alaska, June 2008.

\bibitem{Sibley10jfr}
G.~Sibley, L.~Matthies, and G.~Sukhatme, ``Sliding window filter with
  application to planetary landing,'' \emph{J. of Field Robotics}, vol.~27,
  no.~5, pp. 587--608, 2010.

\bibitem{DongSi11icra}
T.-C. Dong-Si and A.~Mourikis, ``Motion tracking with fixed-lag smoothing:
  Algorithm consistency and analysis,'' in \emph{IEEE Intl. Conf. on Robotics
  and Automation (ICRA)}, 2011.

\bibitem{Leutenegger15ijrr}
S.~Leutenegger, S.~Lynen, M.~Bosse, R.~Siegwart, and P.~Furgale,
  ``Keyframe-based visual-inertial slam using nonlinear optimization,''
  \emph{Intl. J. of Robotics Research}, 2015.

\bibitem{Bryson09icra}
M.~Bryson, M.~Johnson-Roberson, and S.~Sukkarieh, ``Airborne smoothing and
  mapping using vision and inertial sensors,'' in \emph{IEEE Intl. Conf. on
  Robotics and Automation (ICRA)}, 2009, pp. 3143--3148.

\bibitem{Indelman13ras}
V.~Indelman, S.~Wiliams, M.~Kaess, and F.~Dellaert, ``Information fusion in
  navigation systems via factor graph based incremental smoothing,''
  \emph{Robotics and Autonomous Systems}, vol.~61, no.~8, pp. 721--738, August
  2013.

\bibitem{Shen14thesis}
S.~Shen, \emph{Autonomous Navigation in Complex Indoor and Outdoor Environments
  with Micro Aerial Vehicles}.\hskip 1em plus 0.5em minus 0.4em\relax PhD
  Thesis, University of Pennsylvania, 2014.

\bibitem{keivan14iser}
N.~Keivan, A.~Patron-Perez, and G.~Sibley, ``Asynchronous adaptive conditioning
  for visual-inertial {SLAM},'' in \emph{Intl. Sym. on Experimental Robotics
  (ISER)}, 2014.

\bibitem{PatronPerez15}
A.~Patron-Perez, S.~Lovegrove, and G.~Sibley, ``A spline-based trajectory
  representation for sensor fusion and rolling shutter cameras,'' \emph{Intl.
  J. of Computer Vision}, February 2015.

\bibitem{Lupton12tro}
T.~Lupton and S.~Sukkarieh, ``Visual-inertial-aided navigation for high-dynamic
  motion in built environments without initial conditions,'' \emph{{IEEE}
  Trans. Robotics}, vol.~28, no.~1, pp. 61--76, Feb 2012.

\bibitem{Forster16tro}
C.~Forster, L.~Carlone, F.~Dellaert, and D.~Scaramuzza, ``On-manifold
  preintegration theory for fast and accurate visual-inertial navigation,''
  \emph{{IEEE} Trans. Robotics}, vol.~33, no.~1, pp. 1--21, 2016, arxiv
  preprint: 1512.02363, \linkToPdf{http://arxiv.org/abs/1512.02363}, technical
  report GT-IRIM-CP\&R-2015-001\award{, Transactions on Robotics Best paper
  award}.

\bibitem{Das11icml-submodularity}
A.~Das and D.~Kempe, ``Submodular meets spectral: Greedy algorithms for subset
  selection, sparse approximation and dictionary selection,'' in \emph{Intl.
  Conf. on Machine Learning (ICML)}, 2011, pp. 1057--1064.

\bibitem{Farrell08book}
J.~Farrell, \emph{Aided Navigation: {GPS} with High Rate Sensors}.\hskip 1em
  plus 0.5em minus 0.4em\relax McGraw-Hill, 2008.

\bibitem{Tron15acc-rigidity}
R.~Tron, L.~Carlone, F.~Dellaert, and K.~Daniilidis, ``Rigid components
  identification and rigidity enforcement in bearing-only localization using
  the graph cycle basis,'' in \emph{American Control Conference}, 2015,
  \linkToPdf{https://www.dropbox.com/s/h8g7p6sglb1izte/2015c-CDC-bearingOnyRigidity.pdf?dl=0}.

\bibitem{Carlone14icra-smartFactors}
L.~Carlone, Z.~Kira, C.~Beall, V.~Indelman, and F.~Dellaert, ``Eliminating
  conditionally independent sets in factor graphs: A unifying perspective based
  on smart factors,'' in \emph{IEEE Intl. Conf. on Robotics and Automation
  (ICRA)}, 2014.

\bibitem{Carlone14bmvc-miningSfM}
L.~Carlone, P.~F. Alcantarilla, H.~Chiu, K.~Zsolt, and F.~Dellaert, ``Mining
  structure fragments for smart bundle adjustment,'' in \emph{British Machine
  Vision Conf. (BMVC)}, 2014, accepted as oral presentation (acceptance rate
  $7.7\%$),
  \linkToPdf{https://www.dropbox.com/s/ix2fuokq699xy4t/2015c-BMVC-structuralSymmetries.pdf?dl=0}
  (supplemental material:
  \linkToPdf{https://www.dropbox.com/s/n15y54ewcvzpxkw/2015c-BMVC-structuralSymmetries-supplemental.pdf?dl=0}).

\bibitem{Boyd04book}
S.~Boyd and L.~Vandenberghe, \emph{Convex optimization}.\hskip 1em plus 0.5em
  minus 0.4em\relax Cambridge University Press, 2004.

\bibitem{Horn85book}
R.~Horn and C.~Johnson, \emph{Matrix Analysis}.\hskip 1em plus 0.5em minus
  0.4em\relax Cambridge University Press, 1985.

\bibitem{Ipsen09siam-eigenvaluePerturbation}
I.~C.~F. Ipsen and B.~Nadler, ``Refined perturbation bounds for eigenvalues of
  hermitian and non-hermitian matrices,'' \emph{SIAM J. Matrix Analysis},
  vol.~31, no.~1, 2009.

\bibitem{nemhauser78mp-submodularity}
G.~Nemhauser, L.~Wolsey, and M.~Fisher, ``An analysis of approximations for
  maximizing submodular set functions,'' \emph{Mathematical Programming},
  vol.~14, no.~1, pp. 265--294, 1978.

\bibitem{Das12nips-submodularity}
A.~Das, A.~Dasgupta, and R.~Kumar, ``Selecting diverse features via spectral
  regularization,'' in \emph{Advances in Neural Information Processing Systems
  (NIPS)}, vol.~2, 2012, pp. 1583--1591.

\bibitem{Rosen16wafr-sesync}
D.~Rosen, L.~Carlone, A.~Bandeira, and J.~Leonard, ``{SE-Sync}: A certifiably
  correct algorithm for synchronization over the {Special Euclidean} group,''
  in \emph{Intl. Workshop on the Algorithmic Foundations of Robotics (WAFR)},
  2016, extended arxiv preprint: 1611.00128,
  \linkToPdf{http://arxiv.org/abs/1611.00128}
  \linkToPdf{http://wafr2016.berkeley.edu/papers/WAFR_2016_paper_138.pdf}
  \linkToCode{https://github.com/david-m-rosen/SE-Sync}\award{, best paper
  award}.

\bibitem{Dellaert12tr}
F.~Dellaert, ``Factor graphs and {GTSAM}: A hands-on introduction,'' Georgia
  Institute of Technology, Tech. Rep. GT-RIM-CP\&R-2012-002, September 2012.

\bibitem{Sun18ral-UAVsystem}
K.~Sun, K.~Mohta, B.~Pfrommer, M.~Watterson, S.~Liu, Y.~Mulgaonkar, C.~Taylor,
  and V.~Kumar, ``Robust stereo visual inertial odometry for fast autonomous
  flight,'' \emph{{IEEE} Robotics and Automation Letters}, vol.~3, no.~2, pp.
  965--972, 2018.

\end{thebibliography}
